\documentclass[12pt]{ucdavisthesis}

\usepackage[us,nodayofweek,12hr]{datetime}
\usepackage{graphicx}
\usepackage{cite}
\usepackage{amsthm}
\usepackage{amsmath}
\usepackage{amssymb}
\usepackage{amsfonts}
\usepackage{mathtools} 

\usepackage{graphicx}
\usepackage{textcomp}
\usepackage{xcolor}
\usepackage{microtype}
\usepackage{graphicx}
\usepackage{subfigure}
\usepackage{booktabs} 

\usepackage{ulem}
\usepackage{microtype}
\usepackage{graphicx}
\usepackage{subfigure}
\usepackage{booktabs} 

\usepackage{lipsum}
\usepackage{balance} 
\usepackage{multicol}
\usepackage{subfigure}

\usepackage{multirow}
\usepackage{paralist}
\usepackage{arydshln}
\usepackage{graphics}
\usepackage{subfigure}
\usepackage{bbm}
\usepackage{xspace}

\newcommand{\pluseq}{\mathrel{+}=}

\newcommand{\mineq}{\mathrel{-}=}

\usepackage{bbm}
\usepackage{xspace}
\usepackage{booktabs,arydshln}
\usepackage[utf8]{inputenc}

\def\R{\mathbb{R}}
\def\L{\mathbb{L}}
\def\T{\mathcal{T}}
\def\Primal{\sf{Primal-CR}\xspace}
\def\Primalpp{\sf{Primal-CR++}\xspace}

\newcommand{\norm}[1]{\left\lVert#1\right\rVert}

\def\R{\mathbb{R}}

\newtheorem{theorem}{Theorem}

\usepackage{etex}  
\usepackage{amsmath,amssymb,bm,comment,environ,grffile,latexsym,verbatim,xspace,rotating}
\usepackage{amsthm}
\makeatletter
\@ifpackageloaded{float}{}{
  \usepackage{floatrow}
  \floatsetup[table]{capposition=top}
}
\makeatother

\makeatletter
\@ifpackageloaded{algorithmic}{}{\usepackage[noend]{algpseudocode}}
\@ifpackageloaded{algorithm}{}{
\ifx\withchapter\undefined
\usepackage[ruled]{algorithm}
\else
\usepackage[ruled,chapter]{algorithm}
\fi
}

\makeatother

\usepackage{graphicx} 
\usepackage{grffile} 
\usepackage{wrapfig}

\usepackage{multirow}

\usepackage{tikz}
\usepackage{pgfplots}
\pgfplotsset{compat=1.14}
\usepackage{tikz-3dplot}
\usepackage{tikz-qtree}
\usetikzlibrary{shapes,shapes.geometric,shapes.arrows,shapes.gates.logic.US,trees,positioning,arrows,automata,decorations.pathmorphing,chains}
\usetikzlibrary{shadows,calc,backgrounds,patterns,fit}
\usetikzlibrary{matrix,arrows}
\usepgfplotslibrary{external}
\tikzexternaldisable

\usepackage{url}

\makeatletter
\@ifpackageloaded{hyperref}{}{
  \usepackage[colorlinks,linkcolor={red!80!black},citecolor={blue!80!black},urlcolor={blue!80!black}]{hyperref}
}

\@ifclassloaded{beamer}{
  \usepackage{amsfonts}
  }{
  \@ifclassloaded{acmart}{}{
      \usepackage{mathrsfs,dsfont} 
      \usepackage{mathbbol}
      \DeclareSymbolFontAlphabet{\mathbbl}{bbold}
      \usepackage{amsfonts}
      \DeclareSymbolFontAlphabet{\mathbb}{AMSb}
  }

  \usepackage{enumitem}
  \setlist{noitemsep}
  \setlist[1]{noitemsep}
  \setlist[1]{nosep}

}
\makeatother

\usepackage{color}
\usepackage{mathtools}
\usepackage{ifthen}
\usepackage{etoolbox}

\usepackage{xr}
\makeatletter
\newcommand*{\addFileDependency}[1]{
        \typeout{(#1)}
        \@addtofilelist{#1}
        \IfFileExists{#1}{}{\typeout{No file #1.}}
}
\makeatother

\input{definitions.tex}

\usepackage[utf8]{inputenc}
 
\usepackage{listings}
\usepackage{color}
 
\definecolor{codegreen}{rgb}{0,0.6,0}
\definecolor{codegray}{rgb}{0.5,0.5,0.5}
\definecolor{codepurple}{rgb}{0.58,0,0.82}
\definecolor{backcolour}{rgb}{0.95,0.95,0.92}
 
\lstdefinestyle{mystyle}{
    backgroundcolor=\color{backcolour},   
    commentstyle=\color{codegreen},
    keywordstyle=\color{magenta},
    numberstyle=\tiny\color{codegray},
    stringstyle=\color{codepurple},
    basicstyle=\footnotesize,
    breakatwhitespace=false,         
    breaklines=true,                 
    captionpos=b,                    
    keepspaces=true,                 
    numbers=left,                    
    numbersep=5pt,                  
    showspaces=false,                
    showstringspaces=false,
    showtabs=false,                  
    tabsize=2
}

\title          {Advances in Collaborative Filtering and Ranking}

\author         {Liwei Wu}

\authordegrees  {B.S. City University of Hong Kong 2014 \\
                 M.S. University of California, Davis 2016 \& 2018}

\officialmajor  {Statistics}

\graduateprogram{PhD in Statistics}

\degreeyear     {2020}

\degreemonth    {March}

\committee{James Sharpnack}{Cho-Jui Hsieh, Chair}{Thomas Lee}{}{}

\dedication{\textsl{To my loving parents and \\
            adorable girlfriend \ldots}}

\abstract{
In this dissertation, we cover some recent advances in collaborative filtering and ranking. 
In chapter 1, we give a brief introduction of the history and the current landscape of collaborative filtering and ranking; chapter 2 we first talk about pointwise collaborative filtering problem with graph information, and how our proposed new method can encode very deep graph information which helps four existing graph collaborative filtering algorithms; chapter 3 is on the pairwise approach for collaborative ranking and how we speed up the algorithm to near-linear time complexity; chapter 4 is on the new listwise approach for collaborative ranking and how the listwise approach is a better choice of loss for both explicit and implicit feedback over pointwise and pairwise loss; chapter 5 is about the new regularization technique Stochastic Shared Embeddings (SSE) we proposed for embedding layers and how it is both  theoretically sound and empirically effectively for 6 different tasks across recommendation and natural language processing; chapter 6 is how we introduce personalization for the state-of-the-art sequential recommendation model with the help of SSE, which plays an important role in preventing our personalized model from overfitting to the training data; chapter 7, we summarize what we have achieved so far and predict what the future directions can be; chapter 8 is the appendix to all the chapters.
}

\acknowledgments{
Being a PhD student at the Statistics Department of the University of California, Davis is an extremely fun experience. Choosing Davis among many other schools back in 2014 is definitely one of the best decisions I have ever made in my life. Choosing my two advisors is another great decision I made in my life: at Davis, I was fortunate to meet my 2 wonderful advisors Cho (who has now moved to UCLA) and James, who have both become fathers in recent years. It is an interesting fact that they both came to Davis one year later than me, so technically one can argue I am more senior at Davis than my advisors. During my study, I was given a lot of freedom to follow my heart and passion to explore diverse directions that I am most curious about and work on my own projects with lots of encouragements, and appropriate level of guidance whenever I need it. I have enjoyed every single bit of the past five and half years of my study. 

In the process, I was fortunate to obtain 2 master degrees in Statistics and Computer Science in 2016 and 2018 respectively. I got to take or sit-in courses from many wonderful professors from both departments, including Ethan, Alexander, Prabir, Fushing, Jiming, Thomas, Hans, Debashis, James, Jie, Wolfgang, Duncan, and Jane-Ling from Statistics Department, Cho, Ian, Vladimir, Yong Jae, Yu, Matthew, Nitta, Sam from Computer Science Department.

I was blessed to do 4 distinct internships at 4 different companies, namely AT\&T Labs at San Ramon, Facebook at Melno Park and LinkedIn at Mountain View, and Dataminr at midtown Manhattan, during which I have encountered many helpful and inspirational people, especially my colleagues, mentors and supervisors.
Moreover, I got the precious opportunity to present my work at top conferences including KDD'17 at Halifax, ICML'18 at Stockholm and NeurIPS'19 at Vancouver and attend different conferences around the world. I also got to teach a lot at Davis: I would have taught as a teaching assistant for 16 quarters.

I'm truly grateful for the professors that wrote me strong recommendation letters while I was in college. Without the help and encouragements, my journey of PhD study would not have started. Thank you, Betsy, Jun, Rick and Xiang. It is an interesting fact that except Rick has retired, all the rest have now got tenure. Congratulations and how time flies!

I also want to thank my collaborators including Hsiang-Fu and Nikhil from Amazon, Mahdi from Rutgers, Shengli, Joel and Alex from Dataminr, and of course most importantly my girlfriend, collaborator and best friend Shuqing. It has been a great pleasure to work with Shuqing, travel together to different conferences, and present our works. She also helped to proof-read the paper. Her accompany during past few years made my PhD study enjoyable. Actually SQL-Rank algorithm presented in chapter 4 is also named after her.

Last but not the least, I want to mourn for those over 1000 people who have passed away because of the novel coronavirus outbreak in China (including Dr Wenliang Li and other heroes I don't know their names), and pray for families still stuck in Wuhan, including my parents and grandmother. Born and raised up in Wuhan till age of 18, I sincerely wish we can recover from coronavirus soon together as humankind in 2020. I cannot imagine how much pain those who lost their parents, partner and loved ones have been bearing. Maybe this reminds us to spend every day just like the last day on earth. Ask yourself, if today were the last day, what is the most important thing that you want to do? Then just follow your heart and I believe everyone can achieve a meaningful life.
}

\begin{document}


\bibliographystyle{IEEEtran}

\makeintropages 


\def\Primal{\sf{Primal-CR}\xspace}
\def\Primalpp{\sf{Primal-CR++}\xspace}
\def\cpp{{C\nolinebreak[4]\hspace{-.05em}\raisebox{.4ex}{\tiny\bf ++ }}}

\chapter{Introduction}
\section{Overview}
Nowadays in online retail and online content delivery applications, it is commonplace to have embedded recommendation systems algorithms that recommend items to users based on previous user behaviors and ratings. The field of recommender systems has gained more and more popularity ever since the famous Netflix competition \cite{Bennett07thenetflix}, in which competitors utilize user ratings to predict ratings for each user-movie pair and the final winner takes home 1 million dollars. During the competition, 2 distinct approaches stand out: one being the Restricted Boltzmann Machines \cite{salakhutdinov2007restricted} and the other being matrix factorization \cite{mnih2008probabilistic, koren2009matrix}. The combination of both approaches work well during the competition, but due to the ease of training and inference, matrix factorization approaches have dominated the collaborative filtering field before the widespread adoption of deep learning methods \cite{schafer2007collaborative}. Collaborative filtering refers to making automatic predictions (filtering) about the interests of a user by collecting preferences or taste information from many users (collaborating). Usually this does not require approaching the recommendation problem as a ranking problem but rather pretends it is a regression or classification problem. Of course, there may be some loss incurred when using a regression or classification loss for ultimately what is a ranking problem. 
Because at the end of day, the ordering is the most important thing, which is directly associated with the recommender system performance. Collaborative ranking approaches mitigate the concerns by using a ranking loss. The ranking loss can be either pairwise or listwise. The difference among pointwise, pairwise and listwise are rooted in the distinct interpretations of the same data points. Pointwise approaches assume each user-item rating datum is independent; pairwise approaches assume that pairwise comparisons for two items by the same user are independent; the listwise approaches view the list of item preferences as a whole and treats different users' list as independent data points. In a strict definition, the Collaborative Filtering approaches refers to the pointwise approach while the Collaborative Ranking refers to the pairwise and listwise approaches. In a loose definition, Collaborative Ranking sometimes is viewed as a sub-field of Collaborative Filtering. 

In another classification of these approaches, recommender systems can be divided into those designed for explicit feedback, such as ratings \cite{koren2009matrix}, and those for implicit feedback, based on user engagement \cite{hu2008collaborative}. Recently, implicit feedback datasets, such as user clicks of web pages, check-in’s of restaurants, likes of posts, listening behavior of music, watching history and purchase history, are increasingly prevalent. Unlike explicit feedback, implicit feedback datasets can be obtained without users noticing or active participation. All collaborative filtering and collaborative ranking  approaches can be divided into these 6 (3 by 2) categories.

So far, we have not touched the features used in recommender systems. Most open datasets do not have good user-side features due to privacy concerns. Even for item side, most datasets including Netflix, Movielens datasets do not have features prepared. However, feature information is crucial for better recommendation and ranking performances besides the fundamental approaches and data. There are many useful features but among them, probably the most challenging ones are including graphs that encode item or user relationships \cite{rao2015collaborative, berg2017graph} and temporal information \cite{hidasi2015session, kang2018self}. The ways of constructing graphs can vary. One way to define graph over users is to exploit the friendship relationships among friends. But there are many ways to define such graphs. If there is more than one type of relationship, then we call it knowledge graphs instead of graphs. Knowledge graphs widely exist: for example, between 2 movies, they could be starred by the same actors or they could be the same genre of movies. In the field of collaborative filtering with graphs, \cite{rao2015collaborative, berg2017graph} find using graph indeed improves recommendation performances. As to temporal information, it is motivated by the fact that user-item interactions do not happen at one time and then stay static. Instead, usually they occur in a temporal order. This means standard train/validation/test splits may not be realistic. Because during inference, we want to predict future interactions only based on historical interactions. In sequential recommendation \cite{kang2018self}, train/validation/test are split in temporal ordering so they do not overlap. It has been shown in such setting, temporal ordering plays a large role in final ranking results \cite{hidasi2015session, kang2018self}.

\begin{table*}
  \caption{Summary of Different Fundamental Approaches Before This Dissertation.} 
  \label{tab:approaches_before}
  \resizebox{0.5\textwidth}{!}{
\begin{tabular}{ |c | c| c| }
\hline
  & Explicit Feedback & Implicit Feedback \\ 
  \hline
 Point-wise Approach & \cite{koren2009matrix, mnih2008probabilistic}  &   \cite{hu2008collaborative, pan2008one} \\\hline
 Pair-wise Approach & \cite{park2015preference} & \cite{rendle2009bpr}  \\ \hline
 List-wise Approach & \cite{xia2008listwise, huang2015listwise} & \\ \hline
\end{tabular}
}
\end{table*}

\begin{table*}
  \caption{Summary of Different Fundamental Approaches After This Dissertation.} 
  \label{tab:approaches_after}
  \resizebox{0.5\textwidth}{!}{
\begin{tabular}{ |c | c| c| }
\hline
  & Explicit Feedback & Implicit Feedback \\ 
  \hline
 Point-wise Approach & \cite{koren2009matrix, mnih2008probabilistic}  &   \cite{hu2008collaborative, pan2008one} \\\hline
 Pair-wise Approach & \cite{park2015preference} \textbf{\cite{wu2017large}} & \cite{rendle2009bpr}  \\ \hline
 List-wise Approach & \cite{xia2008listwise, huang2015listwise} & \textbf{\cite{wu2018sql}} \\ \hline
\end{tabular}
}
\end{table*}

\section{Contributions and Outline of This Thesis}
This dissertation summarizes several published and under-review works that advance the field of collaborative filtering and ranking.
Our first main contribution is that we fill out the void in Table~\ref{tab:approaches_before}. Before this dissertation (the chapter 4 \cite{wu2018sql}), no one had successfully applied the listwise approach to the implicit feedback setting due to the difficulty of the problem, because in implicit feedback contain only 1's and 0's without different level of ratings. A second contribution is that we speed up the previous best pairwise approaches for explicit feedback significantly. We achieved near-linear time complexity, which allows us to scale up to the full Netflix dataset without subsampling \cite{wu2017large}. We not only contribute to the fundamental approaches but also to extended approaches that utilize extra information, including graph and temporal ordering information. On one hand, we propose a novel way to encode long range graph interactions without require any training using bloom filters as backbone \cite{wu2019graph}. 
On the other hand, with the help of a new embedding-layer regularization called Stochastic Shared Embeddings (SSE) \cite{wu2019stochastic}, we can also introduce personalization for the state-of-the-art sequential recommendation model and achieve much better ranking performance with our personalized model \cite{wu2019temporal}, where personalization is crucial for the success of recommender systems unlike most natural language tasks. This new regularization not only helps existing collaborative filtering and collaborative ranking algorithms but also benefits methods in natural language processing in fields like machine translation and sentiment analysis \cite{wu2019stochastic}. 

The outline of this thesis is as follows: chapter 2 we first talk about pointwise collaborative filtering problem with graph information, and how our proposed new method can encode very deep graph information which helps four existing graph collaborative filtering algorithms; chapter 3 is on the pairwise approach for collaborative ranking and how we speed up the algorithm to near-linear time complexity; chapter 4 is on the new listwise approach for collaborative ranking and how the listwise approach is a better choice of loss for both explicit and implicit feedback over pointwise and pairwise loss; chapter 5 is about the new regularization technique Stochastic Shared Embeddings (SSE) we proposed for embedding layers and how it is both  theoretically sound and empirically effectively for 6 different tasks across recommendation and natural language processing; chapter 6 is how we introduce personalization for the state-of-the-art sequential recommendation model with the help of SSE, which plays an important role in preventing our personalized model from overfitting to the training data.

\subsection*{Summary Chapter 2:}
In this chapter, we consider recommender systems with side information in the form of graphs. Existing collaborative filtering algorithms mainly utilize only immediate neighborhood information and do not efficiently take advantage of deeper neighborhoods beyond 1-2 hops. The main issue with exploiting deeper graph information is the rapidly growing time and space complexity when incorporating information from these neighborhoods. In this chapter, we propose using Graph DNA, a novel Deep Neighborhood Aware graph encoding algorithm, for exploiting multi-hop neighborhood information. DNA encoding computes approximate deep neighborhood information in linear time using Bloom filters, 
and results in a per-node encoding whose dimension is logarithmic in the number of nodes in the graph. It can be used in conjunction with both feature-based and graph-regularization-based collaborative filtering algorithms.  Graph DNA has the advantages of being memory and time efficient and providing additional regularization when compared to directly using higher order graph information. Code is open-sourced at \url{https://github.com/wuliwei9278/Graph-DNA}.
This work is going to be published at the 23rd International Conference on Artificial Intelligence and Statistics (AISTATS 2020).

\subsection*{Summary Chapter 3:}
In this chapter, we consider the Collaborative Ranking (CR) problem for recommendation systems. Given a set of pairwise preferences between items for each user, collaborative ranking can be used to rank un-rated items for each user, and this ranking can be naturally used for recommendation. 
It is observed that collaborative ranking algorithms usually achieve better performance since they directly minimize the ranking loss; however, they are rarely used in practice due to the poor scalability. 
All the existing CR algorithms have time complexity at least $O(|\Omega|r)$ per iteration, where $r$ is the target rank and 
$|\Omega|$ is number of pairs which grows quadratically with number of ratings per user. 
For example, the Netflix data contains totally 20 billion rating pairs, and at this scale all the current algorithms have to work with significant subsampling, resulting in poor prediction on testing data. 

In this chapter, we propose a new collaborative ranking algorithm called Primal-CR that reduces the time complexity to $O(|\Omega|+d_1 \bar{d}_2 r)$, where $d_1$ is number of users
and $\bar{d}_2$ is the averaged number of items rated by a user. Note that $d_1 \bar{d}_2$ 
is strictly smaller and often much smaller than $|\Omega|$. 

Furthermore, by exploiting the fact that most data is in the form of numerical ratings instead of pairwise comparisons, 
we propose Primal-CR++ with $O(d_1\bar{d}_2 (r+ \log \bar{d}_2 ))$ time complexity. 
Both algorithms have better theoretical time complexity than existing approaches and also 
outperform existing approaches in terms of NDCG and pairwise error on real data sets. 
To the best of our knowledge, this is the first collaborative ranking algorithm capable of working on the full Netflix dataset using all the 20 billion rating pairs, and this leads to a model with much better recommendation compared with previous models trained on subsamples. 
Finally, compared with classical matrix factorization algorithm which also requires $O(d_1 \bar{d}_2 r)$ time, our algorithm has almost the same efficiency while making much better recommendations since we consider the ranking loss. Code is open-sourced at \url{https://github.com/wuliwei9278/ml-1m} (Julia version) and \url{https://github.com/wuliwei9278/primalCR} (\cpp version).
This work has been published at the 23rd ACM SIGKDD International Conference on Knowledge Discovery and Data Mining (KDD 2017).

\subsection*{Summary Chapter 4:}
In this chapter, we propose a listwise approach for constructing user-specific rankings in recommendation systems in a collaborative fashion.
We contrast the listwise approach to previous pointwise and pairwise approaches, which are based on treating either each rating or each pairwise comparison as an independent instance respectively.
By extending the work of \cite{cao2007learning}, we cast listwise collaborative ranking as maximum likelihood under a permutation model which applies probability mass to permutations based on a low rank latent score matrix.
We present a novel algorithm called SQL-Rank, which can accommodate ties and missing data and can run in linear time. 
We develop a theoretical framework for analyzing listwise ranking methods based on a novel representation theory for the permutation model.
Applying this framework to collaborative ranking, we derive asymptotic statistical rates as the number of users and items grow together.
We conclude by demonstrating that our SQL-Rank method often outperforms current state-of-the-art algorithms for implicit feedback such as Weighted-MF and BPR and achieve favorable results when compared to explicit feedback algorithms such as matrix factorization and collaborative ranking. Code is open-sourced at \url{https://github.com/wuliwei9278/SQL-Rank}.
This work has been published at the Thirty-fifth International Conference on Machine Learning (ICML 2018).

\subsection*{Summary Chapter 5:}
In deep neural nets, lower level embedding layers account for a large portion of the total number of parameters. Tikhonov regularization, graph-based regularization, and hard parameter sharing are approaches that introduce explicit biases into training in a hope to reduce statistical complexity. Alternatively, we propose stochastic shared embeddings (SSE), a data-driven approach to regularizing embedding layers, which stochastically transitions between embeddings during stochastic gradient descent (SGD). Because SSE integrates seamlessly with existing SGD algorithms, it can be used with only minor modifications when training large scale neural networks. We develop two versions of SSE: SSE-Graph using knowledge graphs of embeddings; SSE-SE using no prior information. We provide theoretical guarantees for our method and show its empirical effectiveness on 6 distinct tasks, from simple neural networks with one hidden layer in recommender systems, to the transformer and BERT in natural languages. We find that when used along with widely-used regularization methods such as weight decay and dropout, our proposed SSE can further reduce overfitting, which often leads to more favorable generalization results. Code is open-sourced at \url{https://github.com/wuliwei9278/SSE}.
This work has been published at the Thirty-third Annual Conference on Neural Information Processing Systems (NeurIPS 2019).

\subsection*{Summary Chapter 6:}
Temporal information is crucial for recommendation problems because user preferences are naturally dynamic in the real world. Recent advances in deep learning, especially the discovery of various attention mechanisms and newer architectures in addition to widely used RNN and CNN in natural language processing, have allowed for better use of the temporal ordering of items that each user has engaged with. In particular, the SASRec model, inspired by the popular Transformer model in natural languages processing, has achieved state-of-the-art results. However, SASRec, just like the original Transformer model, is inherently an un-personalized model and does not include personalized user embeddings. To overcome this limitation, we propose a Personalized Transformer (SSE-PT) model, outperforming SASRec by almost 5\% in terms of NDCG@10 on 5 real-world datasets. Furthermore, after examining some random users' engagement history, we find our model not only more interpretable but also able to focus on recent engagement patterns for each user. Moreover, our SSE-PT model with a slight modification, which we call SSE-PT++, can handle extremely long sequences and outperform SASRec in ranking results with comparable training speed, striking a balance between performance and speed requirements. Our novel application of the Stochastic Shared Embeddings (SSE) regularization is essential to the success of personalization. Code and data are open-sourced at \url{https://github.com/wuliwei9278/SSE-PT}. This work is currently still under review.

\newcommand{\graphbloom}{{\it GraphBloom}\xspace}

\newenvironment{proof}{\paragraph{Proof:}}{\hfill$\square$}

\chapter{Collaborative Filtering with Graph Encoding}

 \section{Introduction}

Recommendation systems are increasingly prevalent due to content delivery platforms, e-commerce websites, and mobile apps \cite{shani2008mining}.
Classical collaborative filtering algorithms use matrix factorization to identify latent features that describe the user preferences and item meta-topics from partially observed ratings \cite{koren2009matrix}.
In addition to rating information, many real-world recommendation datasets also have a wealth of side information in the form of graphs, and incorporating this information often leads to performance gains \cite{rao2015collaborative, zhou2012kernelized, liang2016factorization}.
However, each of these only utilizes the immediate neighborhood information of each node in the side information graph. 
 More recently, \cite{berg2017graph} incorporated graph information when learning features with a Graph Convolution Network (GCN) based recommendation algorithm.
GCNs \cite{kipf2016semi} constitute flexible methods for incorporating graph structure beyond first-order neighborhoods, but their training complexity typically scales rapidly with the depth, even with sub-sampling techniques \cite{chen2018stochastic}.
Intuitively, exploiting higher-order neighborhood information could benefit the generalization performance, especially when the graph is sparse, which is usually the case in practice. The main caveat of exploiting higher-order graph information is the high computational and memory cost when computing higher-order neighbors since the number of $t$-hop neighbors typically grows exponentially with $t$. 

We aim to utilize higher order graph information without introducing much computational and memory overhead. We propose 
a Graph Deep Neighborhood Aware (Graph DNA) encoding, which approximately captures the higher-order neighborhood information of each node via Bloom filters \cite{bloom1970space}.
Bloom filters encode neighborhood sets as $c$ dimensional 0/1 vectors, where $c = O(\log n)$ for a graph with $n$ nodes, which approximately preserves membership information.
This encoding can then be combined with both graph regularized or feature based collaborative filtering algorithms, with little computational and memory overhead. 
In addition to computational speedups, we find that Graph DNA achieves better performance over competitors. We show that our Graph DNA encoding can be used with several collaborative filtering algorithms: graph-regularized matrix factorization with explicit and implicit feedback \cite{zhou2012kernelized, rao2015collaborative}, co-factoring \cite{liang2016factorization}, and GCN-based recommendation systems \cite{monti2017geometric}. 
In some cases, using information from deeper neighborhoods (like $4^{th}$ order) yields a 15x increase in performance, with graph DNA encoding yielding a 6x speedup compared to directly using the $4^{th}$ power of the graph adjacency matrix.  

\section{Related Work}
\label{sec:bloom-filter}
\vspace{-2mm}
Matrix factorization has been used extensively in recommendation systems with both explicit \cite{koren2009matrix} and implicit \cite{hu2008collaborative} feedback. Such methods compute low dimensional user and item representations; their inner product approximates the observed (or to be predicted) entry in the target matrix. To incorporate graph side information in these systems, \cite{rao2015collaborative, zhou2012kernelized} used a graph Laplacian based regularization framework that forces a pair of node representations to be similar if they are connected via an edge in the graph. In \cite{yu2017unified}, this was extended to the implicit feedback setting. \cite{liang2016factorization} proposed a method that incorporates first-order information of the rating bipartite graph into the model by considering item co-occurrences. More recently, GC-MC \cite{berg2017graph} used a GCN approach performing convolutions on the main bipartite graph by treating the first-order side graph information as features, and \cite{monti2017geometric} proposed combining GCNs and RNNs for the same task. 

Methods that use higher order graph information are typically based on taking random walks on the graphs \cite{gori2007itemrank}. 
\cite{jamali2009trustwalker}  extended this method to include graph side information in the model. Finally, the PageRank \cite{page1999pagerank} algorithm can be seen as computing the steady state distribution of a Markov network, and similar methods for recommender systems was proposed in \cite{abbassi2007recommender, xie2015edge}. 

For a complete list of related works of representation learning on graphs, we refer the interested user to \cite{hamilton2017representation}. For the collaborative filtering  setting, \cite{monti2017geometric, berg2017graph} use Graph Convolutional Neural Networks (GCN) \cite{defferrard2016convolutional}, but with some modifications. Standard GCN methods without substantial modifications cannot be directly applied to collaborative filtering rating datasets, including well-known approaches like GCN \cite{kipf2016semi} and GraphSage \cite{hamilton2017inductive}, because they are intended to solve semi-supervised classification problem over graphs with nodes' features. 
PinSage \cite{ying2018graph} is the GraphSage extension to non-personalized graph-based recommendation algorithm but is not meant for collaborative filtering problems. GC-MC \cite{berg2017graph} extends GCN to collaborative filtering, albeit it is less scalable than \cite{ying2018graph}. Our Graph DNA scheme can be used to obtain graph features in these extensions.
In contrast to the above-mentioned methods involving GCNs, we do not use the data driven loss function to train our graph encoder. This property makes our graph DNA suitable for both transductive as well as inductive problems.

Bloom filters have been used in Machine Learning for multi-label classification \cite{cisse2013robust}, and for hashing deep neural network models representations \cite{shi2009hash, han2015deep, courbariaux2015binaryconnect}. However, to the best of our knowledge, until now, they have not been used to encode graphs, nor has this encoding been applied to recommender systems. So it would be interesting to extend our work to other recommender systems settings, such as \cite{wu2019temporal} and \cite{wu2019stochastic}.


\section{Methodology}
\label{sec:dna-algo}
We consider the recommender system problem with a partially observed rating matrix $R$ and a 
Graph that encodes side information $G$. In this section, we will introduce the Graph DNA algorithm for encoding deep neighborhood information in $G$. In the next section, we will show how this encoded information can be applied to various graph based recommender systems. 

\subsection{Bloom Filter}
The Bloom filter~\cite{bloom1970space} is a probabilistic data structure designed
to represent a set of elements. Thanks to its space-efficiency and
simplicity, Bloom filters are applied in many real-world applications such as
database systems \cite{borthakur2011apache, chang2008bigtable}.
A Bloom filter $\cB$
consists of $k$ independent hash functions
$h_t(x) \rightarrow \cbr{1, \ldots, c}$. 
The Bloom filter $\cB$ of size $c$ can be represented as a length $c$ bit-array $\bb$. More details about Bloom filters can be found in~\cite{broder2004network}.
Here we highlight a few desirable properties of Bloom filters essential to
our graph DNA encoding: 
\begin{enumerate}
  \item Space efficiency: classic Bloom filters use $1.44 \log_2 (1/\epsilon)$
    of space per inserted key, where $\epsilon $ is the false positive rate
    associated with this Bloom filter.
  \item Support for the union operation of two Bloom filters: the Bloom filter
    for the union of two sets can be obtained by performing bitwise `OR' operations
    on the underlying bit-arrays of the two Bloom filters.
  \item Size of the Bloom filter can be approximated by the number of
    nonzeros in the underlying bit array: in particular, given a Bloom filter
    representation $\cB(A)$ of a set $A$: the number of elements of $A$ can be
    estimated as
        $\abs{A} \approx - \frac{c}{k} \log\rbr{ 1 - \frac{\nnz(\bb)}{c}}$,
    where $\nnz(\bb)$ is the number of non-zero elements in array $\bb$.
    As a result, the number of common nonzero bits of $\cB(A_1)$ and $\cB(A_2)$
    can be used as a proxy for $\abs{A_1 \cap A_2}$.
\end{enumerate}

\begin{algorithm}[H]
  \caption{Graph DNA Encoding with Bloom Filters}
\label{alg:graph-bloom}
\begin{algorithmic}[1]
\Require $G$: a graph of $n$ nodes, $c$: the length of codes, $k$: the number of hash functions, $d$: the number of iterations, $\theta$: tuning parameter to control the number of elements hashed.
\Ensure $B \in \cbr{0,1}^{n \times c}$: a boolean matrix to denote the bipartite relationship between $n$ nodes and $c$ bits.
  \begin{compactitem}
    \item $\cH \leftarrow \cbr{\mathtt{h_t}(\cdot): t = 1,\ldots, k}$ \Comment{Pick $k$ hash functions}
    \item ${\bf for}$ $i = 1,\ldots,n$: \Comment{\graphbloom Initialization}
      \begin{compactitem}
        \item $\cB^{0}\mathtt{[i]} \leftarrow \mathtt{BloomFilter}(c, \cH)$
        \item $\cB^{0}\mathtt{[i].add}(i)$
      \end{compactitem}
    \item ${\bf for}$ $s = 1,\ldots,d$: \Comment{$d$ times neighborhood propagations }
      \begin{compactitem}
      \item ${\bf for}$ $i = 1,\ldots,n$:
          \begin{compactitem}
            \item ${\bf for } $ $j \in \cN_1(i)$: \Comment{degree-1 neighbors}
              \begin{compactitem}
                \item ${\bf if}$ $|\cB^{s}\mathtt{[i]}| > \theta$: break;
                \item $\cB^{s}\mathtt{[i].union}(\cB^{s-1}\mathtt{[j]})$
              \end{compactitem}
          \end{compactitem}
      \end{compactitem}
    \item $B_{ij}\leftarrow \cB^{d}\mathtt{[i].b[j]}\ \forall (i,j) \in [n] \times [c]$
  \end{compactitem}
\end{algorithmic}
\end{algorithm}

\newcommand{\BF}[1]{\cB\mathtt{[#1]}}
\subsection{Graph DNA Encoding Via Bloom Filters}
\label{sec:graph-bloom}

Now we introduce our Graph DNA encoding. The main idea is to encode the deep (multi-hop) neighborhood aware embedding for each node in the graph
approximately using the Bloom filter, which helps avoid performing computationally expensive graph adjacency matrix multiplications.
In Graph DNA, we have Bloom filters $\BF{i}, i=1,...,n$ for the $n$ graph nodes. All the Bloom filters $\BF{i}$ share the same $k$ hash functions. The role of $\BF{i}$ is to store the deep neighborhood information of the $i$-th node.  Taking advantage of the union operations of Bloom filters, one node's neighborhood information can be propagated to its neighbors in an iterative manner using gossip algorithms \cite{shah2009gossip}. Initially, each $\BF{i}$ contains only the node itself. At the $s$-th iteration, $\BF{i}$ is updated by taking union with node $i$'s immediate neighbors' Bloom filters $\BF{j}$. By induction, we see that after the $d$ iterations, $\BF{i}$ represents $\cN_d(i) :=  \cbr{j: \text{distance}_{G}(i, j) \le d}$,
where $\text{distance}_{G}(i, j)$ is the shortest path distance between nodes $i$ and $j$ in $G$.
As the last step, we stack array representations of all Bloom filters and form a sparse matrix $B \in \cbr{0, 1}^{n \times c}$, where 
the $i$-th row of $B$ is the bit representation of $\BF{i}$. As a practical measure, to prevent over-saturation of Bloom filters for popular nodes in the graph, we add a hyper-parameter $\theta$ to control the max saturation level allowed for Bloom filters. This would also prevent hub nodes dominating in graph DNA encoding. The pseudo-code for the proposed encoding algorithm is given in Algorithm~\ref{alg:graph-bloom}. 
We use graph DNA-$d$ to denote our obtained graph encoding after applying Algorithm~\ref{alg:graph-bloom} with $s$ looping from 1 to $d$. We also give a simple example to illustrate how the graph DNA is encoded into Bloom filter representations in Figure~\ref{fig:graph-bloom}. Our usage of Bloom filters is very different from previous works in \cite{pozo2016item, serra2017getting, shinde2016user}, which use Bloom filter for standard hashing and is unrelated to graph encoding. 

\begin{figure*}[ht]
\begin{center}
\centerline{\includegraphics[width=1.0\columnwidth]{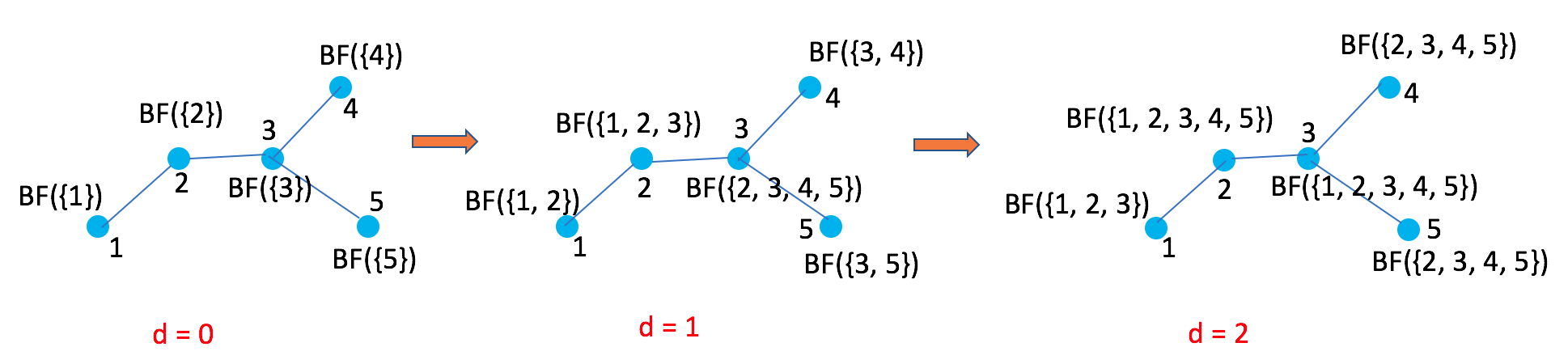}}
\end{center}
\caption{Illustration of Algorithm~\ref{alg:graph-bloom}: the graph DNA encoding procedure. The curly brackets at each node indicate the nodes encoded at a particular step.  At $d=0$ each node's Bloom filter only encodes itself, and multi-hop neighbors are included as d increases.}
\label{fig:graph-bloom}
\end{figure*}

\begin{figure}[ht]
\begin{center}
\centerline{\includegraphics[width=1.0\columnwidth]{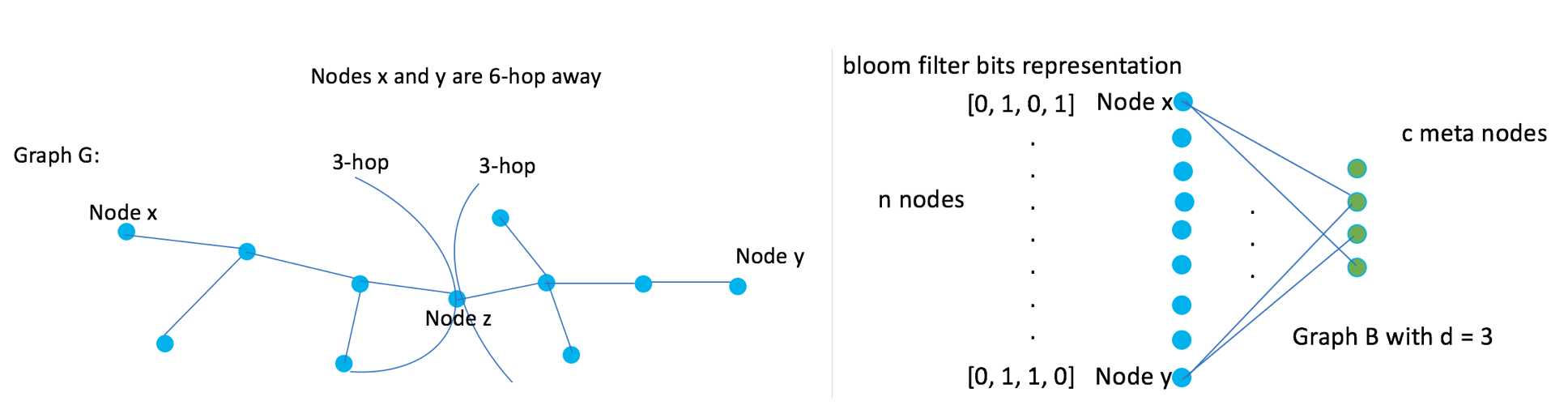}}
\end{center}
\caption{Illustration of our proposed DNA encoding method (DNA-3), with the corresponding bipartite graph representation.}
\label{fig:DNA}
\vspace{-10pt}\end{figure}

\section{Collaborative Filtering with Graph DNA}
\label{sec:app}
Suppose we are given the sparse rating matrix $R \in \dR^{n \times m}$  with $n$ users and $m$ items, and a graph $G \in \dR^{n\times n}$ encoding relationships between users. For simplicity, we do not assume a graph on the $m$ items, though including it should be straightforward. 
%

\subsection{Graph Regularized Matrix Factorization}\label{sec:grmf}


\paragraph{Explicit Feedback :}
The objective function of Graph Regularized Matrix Factorization (GRMF)  \cite{cai2011graph, rao2015collaborative, zhou2012kernelized} is: 
\begin{align}\label{eq:grmf}
\min_{U,V} \sum_{(i,j) \in \Omega} &\left(R_{i,j} - u_i^\top v_j \right)^2 + \frac{\lambda}{2} (\|U\|_F^2 + \|V\|_F^2) \\
\notag
 & +\mu \trace(U^\top \text{Lap}(G) U)
\end{align}
where $U \in \dR^{n \times r}, V \in \dR^{m \times r}$ are the embeddings associated with users and items respectively, $n$ is the number of users and $m$ is the number of items, $R \in \dR^{n \times m}$ is the sparse rating matrix, $\trace()$ is the trace operator, $\lambda, \mu$ are tuning coefficients, and $\text{Lap}(\cdot)$ is the graph Laplacian operator.  

The last term is called graph regularization, which
tries to enforce similar nodes (measured by edge weights in $G$) to have similar embeddings. 
One naive way \cite{cao2015grarep} to extend this to higher-order graph regularization is to replace the graph $G$ with $\sum_{i=1}^K w_i \cdot G^i$ and then use the graph Laplacian of $\sum_{i=1}^K w_i \cdot G^i$ to replace $G$ in \eqref{eq:grmf}. Computing $G^i$ for even small $i$ is computationally infeasible for most real-world applications, and we will soon lose the sparsity of the graph, leading to memory issues. 
Sampling or thresholding could mitigate the problem but suffers from performance degradation.

In contrast, our graph DNA obtained from Algorithm~\ref{alg:graph-bloom} does not suffer from any of these issues. 
The space complexity of our method is only of order $O(n \log n)$ for a graph with $n$ nodes, instead of $O(n^2)$. The reduced number of non-zero elements using graph DNA leads to a significant speed-up in many cases.  


We can easily use graph DNA in GRMF as follows: we treat the $c$ bits as $c$ new pseudo-nodes and add them to the original graph $G$. We then have $n+c$ nodes in a modified graph $\dot{G}$: 

\begin{equation}\label{eq:G1}
\dot{G} =
\begin{bmatrix}
    G \in \dR^{n \times n}      & B\in \dR^{n \times c} \\
    B^\top\in \dR^{c \times n}       & \mathbf{0}\in \dR^{c \times c}
\end{bmatrix}.
\end{equation}

To account for the $c$ new nodes, we expand $U \in \dR^{n \times r}$ to $\dot{U} \in \dR^{(n+c) \times r}$ by appending parameters for the meta-nodes.
The objective function for GRMF with Graph DNA with be the same as \eqref{eq:grmf} except replacing $U$ and $G$ with $\dot{U}$ and $\dot{G}$.
At the prediction stage, we discard the meta-node embeddings. 


\paragraph{Implicit Feedback :}
For implicit feedback data, when $R$ is a 0/1 matrix, weighted matrix factorization is a widely used algorithm~\cite{hu2008collaborative, hsieh2015pu}. The only difference is that the loss function in \eqref{eq:grmf} 
is replaced by $\sum_{(i, j): R_{ij}=1} (R_{ij}-u_i^T v_j)^2 + \sum_{(i, j): R_{ij}=0} \rho (R_{ij}-u_i^T v_j)^2$ where $\rho< 1$ is a hyper-parameter reflecting the confidence of zero entries. In this case, we can apply the Graph DNA encoding as before trivially. 


\subsection{Co-Factorization with Graph Information}

Co-Factorization of Rating and Graph Information (Co-Factor) \cite{singh2008relational, liang2016factorization} is ideologically very different from GRMF and GRWMF, because it does not use graph information as regularization term. Instead it treats the graph adjacency matrix as another rating matrix, sharing one-sided latent factors with the original rating matrix. Co-Factor minimizes the following objective function:
$
\min_{U,V} \sum_{(i,j) \in \Omega_R} \left(R_{i,j} - u_i^\top v_j \right)^2 + \frac{\lambda}{2} (\|U\|_F^2 + \|V\|_F^2 + \|V'\|_F^2) 
\notag + \sum_{(i,j) \in \Omega_G} \left(G_{i,j} - u_i^\top v'_j \right)^2,
$
where $U \in \dR^{n \times r}, V \in \dR^{m \times r}, V' \in \dR^{n \times r}$.
We can extend Co-Factor to incorporate our DNA-d 
by replacing $G$ with $B$ in the equation above, where $B\in \dR^{n \times c}$ is the Bloom filter bipartite graph adjacency matrix of $n$ real-user nodes and $c$ pseudo-user nodes, similar to $B$ as in \eqref{eq:G1}. We call the extension Co-Factor\_DNA-$d$.

\subsection{Graph Convolutional Matrix Completion}
Graph Convolutional Matrix Completion (GC-MC) is a graph convolutional network (GCN) based geometric matrix completion method \cite{berg2017graph}. In \cite{berg2017graph}, 
the rating matrix $R$ is treated as an adjacency matrix in GCN while side information $G$ is treated as feature matrix for nodes — each user has an n-dimensional 0/1 feature that corresponds to a column of $G$.  The GCN model then performs convolutions of these features on the bipartite rating graph. 
Convolutions of these features are performed on the bipartite rating graph. We find in our experiments that using these one-hot encodings of the graph as feature is an inferior choice both in terms of performance and speed. To capture higher order side graph information, it is better to use $G + \alpha G^2$ for some constant $\alpha$ and this alternate choice usually gives smaller generalization error than the original GC-MC method. However, it is hard to explicitly calculate $G + \alpha G^2$ and store the entire matrix for a large graph for the same reason described in Section~\ref{sec:grmf}. Again, we can use graph DNA to efficiently encode and store the higher order information before feeding it into GC-MC. We show in our experiments that this outperforms current state-of-the-art GCN methods \cite{berg2017graph, monti2017geometric} as well as GC-MC with graph encoding methods that require training, such as Node2vec \cite{grover2016node2vec} and Deepwalk \cite{perozzi2014deepwalk}. 
Our encoding scheme does not require training and therefore is a lot faster than previous encoding methods. More details are discussed in the experiment section~\ref{sec:co-factor}.

\section{Experiments}
\label{sec:exp}
We show that our Graph DNA encoding technique can improve the performance of 4 popular graph-based recommendation algorithms: graph-regularized matrix factorization, co-factorization, weighted matrix factorization, and GCN-based graph convolution matrix factorization.  
All experiments except GCN are conducted on a server with Intel Xeon E5-2699 v3 @ 2.30GHz CPU and 256$G$ RAM. The GCN experiments are conducted on Google Cloud with Nvidia V100 GPU.

\subsection{Simulation Study}

\label{sec:exp-syn}

We first simulate a user/item rating dataset with user graph as side information, generate its graph DNA, and use it on a downstream task: matrix factorization. 

We randomly generate user and item embeddings from standard Gaussian distributions, and construct an Erd\H{o}s-R\'{e}nyi Random graph of users. User embeddings are generated using Algorithm~\ref{alg:sim} in Appendix: at each propagation step, each user's embedding is updated by an average of its current embedding and its neighbors' embeddings.
Based on user and item embeddings after $T=3$ iterations of propagation,
we generate the underlying ratings for each user-item pairs according to the inner product of their embeddings, and then sample a small portion of the dense rating matrix as  training and test sets.

We implement our graph DNA encoding algorithm in python using a scalable python library \cite{almeida2007scalable}
to generate Bloom filter matrix $B$.  We adapt the GRMF C++ code to solve the objective function of GRMF\_DNA-K with our Bloom filter enhanced graph $\dot{G}$.
We compare the following variants:
\begin{enumerate}
    \item MF: classical matrix factorization only with $\ell_2$ regularization without graph information.
    \item GRMF\_$G^d$: GRMF with $\ell_2$ regularization and using $G$, $G^2$, \dots, $G^d$ \cite{cao2015grarep}.
    \item GRMF\_DNA-$d$: GRMF with $\ell_2$ but using our proposed graph DNA-$d$.
   
\end{enumerate}

We report the prediction performance with Root Mean Squared Error (RMSE) on test data.
All results are reported on the test set, with all relevant hyperparameters tuned on a held-out validation set. 
To accurately measure how large the relative gain is from using deeper information, we introduce a new metric called Relative Graph Gain (RGG) for using information $X$, which is defined as:
\begin{align}\label{eq:rgg}
    &\text{RGG}(X) \% =  \\
    \notag
    &\left(\frac{\text{RMSE without Graph} - \text{RMSE with } X}{\text{RMSE without Graph} - \text{RMSE with } G} 
    - 1 \right) \times 100, 
\end{align} where RMSE is measured for the same method with different graph information. This metric would be 0 if only first order graph information is utilized and is only defined when the denominator is positive.

In Table~\ref{tab:mf_res}, we can easily see that using a deeper neighborhood helps the recommendation performances on this synthetic dataset. Graph DNA-3's gain is 166\%  larger than that of using first-order graph $G$. We can see an increase in performance gain for an increase in depth $d$ when $d \leq 3$. This is expected because we set $T = 3$ during our creation of this dataset.


\subsection{Graph Regularized Matrix Factorization for Explicit Feedback}
\label{sec:exp-explicit}

Next, we show that graph DNA can improve the performance of GRMF for explicit feedback. We conduct experiments on two real datasets:  Douban  \cite{ma2011recommender} and Flixster  \cite{Zafarani+Liu:2009}.
Both datasets contain explicit feedback with ratings from 1 to 5. There are 129,490 users, 58,541 items in Douban. There are 147,612 users, 48,794 items in Flixster. Both datasets have a graph defined on the respective sets of users.  


We pre-processed Douban and Flixster following the same procedure in \cite{rao2015collaborative, wu2017large}. The experimental setups and comparisons are almost identical to the synthetic data experiment (see details in section~\ref{sec:exp-syn}). Due to the exponentially growing non-zero elements in the graph as we go deeper (see Table~\ref{tab:nnz}), we are unable to run full
GRMF\_$G^4$ and GRMF\_$G^5$ for these datasets. In fact,  GRMF\_$G^3$ itself is too slow so we thresholded $G^3$ by only considering entries whose values are equal to or larger than 4. For the Bloom filter, we set a false positive rate of 0.1 and use capacity of 500 for Bloom filters, resulting in $c = 4,796$. 


We can see from Table~\ref{tab:mf_res} that deeper graph information always helps. For Douban, graph DNA-3 is most effective, giving a relative graph gain of 82.79\% compared to only 2\% gain when using $G^2$ or $G^3$ naively. Interestingly for Flixster, using $G^2$ is better than using $G^3$. However, Graph DNA-3 and DNA-4 yield $10$x and $15$x performance improvements respectively, lending credence to the implicit regularization property of graph DNA. 
For a fixed size Bloom filter, the computational complexity of graph DNA scales linearly with depth $d$, as compared to exponentially for GRMF\_$G^d$.
We measure the speed in Table~\ref{tab:dnaspeed}. The memory cost is only a fraction of $n^2$ after hashing. Such low memory and computational complexity allow us to scale to larger $d$, compared to baseline methods.

\begin{table*}
  \caption{Comparison of Graph Regularized Matrix Factorization Variants for Explicit Feedback on Synthetic, Douban and Flixster data. We use rank $r = 10$. RGG is the Relative Graph Gain defined in  \eqref{eq:rgg}.} 
  \label{tab:mf_res}
  \resizebox{1.0\textwidth}{!}{
  \begin{tabular}{ccccccc}
    \toprule
    &\multicolumn{2}{c}{Synthetic} & \multicolumn{2}{c}{Douban} & \multicolumn{2}{c}{Flixster}\\
    \cmidrule(r){2-3} \cmidrule(r){4-5}  \cmidrule(r){6-7} 
    Dataset & RMSE ($ \times 10^{-1}$)  & \% RGG & RMSE ($ \times 10^{-1}$) & \% RGG & RMSE ($ \times 10^{-1}$)  & \% RGG \\
    \midrule
    
      MF        & 2.9971  & -  & 7.3107   & -  & 8.8111     & - \\
    GRMF\_$G$   & 2.7823 & 0  & 7.2398  & 0  & 8.8049   & 0\\
    GRMF\_$G^2$  & 2.6543  & 59.5903 & 7.2381  & 2.3977 & 8.7849  & 322.5806 \\
    GRMF\_$G^3$ & 2.5687  & 99.4413  & 7.2432  & -4.7954 & 8.7932  & 188.7097 \\
    GRMF\_$G^4$ & 2.5562 & 105.2607 & - & - & - & -\\
    GRMF\_$G^5$  & 2.4853  & 138.2682 & - & - & - & -\\
    GRMF\_$G^6$  & 2.4852  & 138.3147 & - & - & - & -\\
    GRMF\_DNA-1   & 2.4303  & 163.8734  & 7.2191  & 29.1960 & 8.8013   & 58.0645\\
    GRMF\_DNA-2  & 2.4510  & 154.2365 & 7.2359 & 5.5007 & 8.8007  & 67.7419 \\
    GRMF\_DNA-3  & \bfseries{2.4247}  &\bfseries{166.4804}  & \bfseries{7.1811}  & \bfseries{82.7927} & 8.7383  & 1074.1935\\
    GRMF\_DNA-4  & 2.4466   & 156.2849 & 7.1971 & 60.2257 & \bfseries{8.7122}  & \bfseries{1495.1613} \\
    \midrule
    Co-Factor\_$G$    & - & -   & 7.2743    & 0  & 8.7957    & 0   \\
    Co-Factor\_DNA-3  & - & - &\bfseries{7.2623}    & \bfseries{32.9670} & \bfseries{8.7354}   & \bfseries{391.5584} \\
    
    
  \bottomrule
\end{tabular}
}
\end{table*}

\begin{table*}
  \caption{Graph DNA (Algorithm~\ref{alg:graph-bloom}) Encoding Speed. We set number $c = 500$ and implement Graph DNA using single-core python. We can scale up linearly in terms of depth $d$ for a fixed $c$. } 
  \label{tab:dnaspeed}
\resizebox{0.8\textwidth}{!}{
  \begin{tabular}{ccccccc}
    \toprule
    & \multicolumn{2}{c}{Graph Statistics} & \multicolumn{4}{c}{Graph DNA Encoding Time (secs)}\\
    \cmidrule(r){2-3} \cmidrule(r){4-7}
    Dataset & Number of Nodes & Graph Density & DNA-1 & DNA-2 & DNA-3 & DNA-4  \\
    \midrule
   
    Douban & 129,490 & 0.0102\% &132.2717 &   266.3740 & 403.9747 & 580.1547   \\
   Flixster & 147,612 & 0.0117\% & 157.3103  &  317.7706   & 482.0360   & 686.8048    \\
  \bottomrule
\end{tabular}
 }
\end{table*}

\begin{table}
  \caption{Comparison of GRWMF Variants for Implicit Feedback on Douban and Flixster datasets. P stands for precision and N stands for NDCG. We use rank $r = 10$ and all results are in $\%$.}
  \label{tab:implicit_res}
   \resizebox{0.8\textwidth}{!}{
  \begin{tabular}{llccccccc}
    \toprule
    Dataset & Methods & MAP & HLU & P@$1$ & P@$5$ & N@$1$ & N@$5$  \\
    \midrule
    \multirow{2}{*}{Douban}
    &GRWMF\_$G$               & 8.340 & 13.033 & 14.944 & 10.371 & 14.944 & 12.564    \\
    &GRWMF\_DNA-3 & \bfseries{8.400} & \bfseries{13.110} & \bfseries{14.991} & \bfseries{10.397} & \bfseries{14.991} & \bfseries{12.619}     \\
     \midrule
    \multirow{2}{*}{Flixster}
    &GRWMF\_$G$              & 10.889 & 14.909 & 12.303 & 7.9927 & 12.303 & 12.734   \\
     &GRWMF\_DNA-3 & \bfseries{11.612} & \bfseries{15.687} & \bfseries{12.644} & \bfseries{8.1583} & \bfseries{12.644} & \bfseries{13.399}    \\
  \bottomrule
\end{tabular}
 }
\end{table}




\begin{table}
  \caption{Comparison of GCN Methods for Explicit Feedback on Douban, Flixster and Yahoo Music datasets (3000 by 3000 as in \cite{berg2017graph, monti2017geometric}). All the methods except GC-MC utilize side graph information. } 
  \label{tab:gcn}
  \resizebox{0.8\textwidth}{!}{
  \begin{tabular}{llcccc}
    \toprule
    Dataset & Methods & Test RMSE ($ \times 10^{-1}$)   & \% RGG   \\
    \midrule
    \multirow{2}{*}{Douban}
     & SRGCNN (reported by \cite{berg2017graph})   & - & - \\
     &GC-MC                  &  \bfseries{7.3109} $\pm$ \bfseries{0.0150}   & -   \\
    &GC-MC\_$G$                    & 7.3698 $\pm$  0.0737   & N/A   \\
    &GC-MC\_$G^2$                 & 7.3123  $\pm$  0.0139  & N/A     \\
     & GC-MC\_Node2vec  & 7.3666 $\pm$ 0.0218 & N/A \\
    & GC-MC\_Deepwalk & 7.3394 $\pm$  0.0343   & N/A  \\
    &GC-MC\_DNA-2  & 7.3117 $\pm$ 0.0129    & N/A   \\
    \midrule
    \multirow{2}{*}{Flixster}
    & SRGCNN (reported by \cite{berg2017graph}) & 9.2600   & -  \\
     &GC-MC                   & 9.2614  $\pm$ 0.0578     & -     \\
     &GC-MC\_$G$                    &  9.2374 $\pm$  0.1045    & 0     \\
      &GC-MC\_$G^2$                    & \bfseries{8.9344} $\pm$  \bfseries{0.0333}   &   \bfseries{1262.4999}  \\
       & GC-MC\_Node2vec  & 12.0370 $\pm$ 1.9474 & N/A \\
   
      & GC-MC\_Deepwalk & 9.0507  $\pm$  0.1692   & 777.9167 \\
    &GC-MC\_DNA-2  & 8.9536  $\pm$ 0.0770   & 1182.4999 \\
    \midrule
    \multirow{2}{*}{Yahoo Music}
     & SRGCNN (reported by \cite{berg2017graph})  & - & - \\
     
      &GC-MC  &  22.6697 $\pm$  0.3530   & -    \\
     &GC-MC\_$G$  &  21.3672 $\pm$  0.4190  & 0   \\
    &GC-MC\_$G^2$    & 20.2189 $\pm$  0.8664  & 88.1612    \\
     & GC-MC\_Node2vec  & 19.8901 $\pm$ 0.7948 & 113.4050\\
   
   & GC-MC\_Deepwalk & 20.1603 $\pm$ 0.9342  &  92.6603\\
    &GC-MC\_DNA-2  & \bfseries{19.3879} $\pm$ \bfseries{0.2874}     & \bfseries{151.9616}  \\
  \bottomrule
\end{tabular}
}
\end{table}

\subsection{Co-Factorization with Graph for Explicit Feedback}\label{sec:co-factor} 
We show our graph DNA can improve Co-Factor \cite{singh2008relational, liang2016factorization} as well. 
The results are in Table~\ref{tab:mf_res}. We find that applying DNA-3 to the Co-Factor method improves performance on both the datasets, more so for Flixster. 
This is consistent with our observations for GRMF in Table~\ref{tab:mf_res}: deep graph information is more helpful for Flixster than Douban. Applying Graph DNA to Co-Factor is detailed in the Appendix.

\subsection{Graph Regularized Weighted Matrix Factorization for Implicit Feedback}
\label{sec:exp-implicit}
We follow the same procedure as in \cite{wu2018sql} to set ratings of 4 and above to 1, and the rest to 0. 
We compare the baseline graph based weighted matrix factorization \cite{hu2008collaborative, hsieh2015pu} with our proposed weighted matrix factorization with DNA-3. We do not compare with Bayesian personalized ranking \cite{rendle2009bpr} and the recently proposed SQL-rank \cite{wu2018sql} 
as they cannot easily utilize graph information.



The results are summarized in Table~\ref{tab:implicit_res} with experimental details in the Appendix. Again, using DNA-3 achieves better prediction results over the baseline in terms of every single metric on both Douban and Flixster datasets.

\subsection{Graph Convolutional Matrix Factorization}
We can use graph DNA instead to efficiently encode and store the higher order information before feeding it into GC-MC. 

\begin{figure}
  \begin{center}
    \includegraphics[width=80mm, height = 60mm]{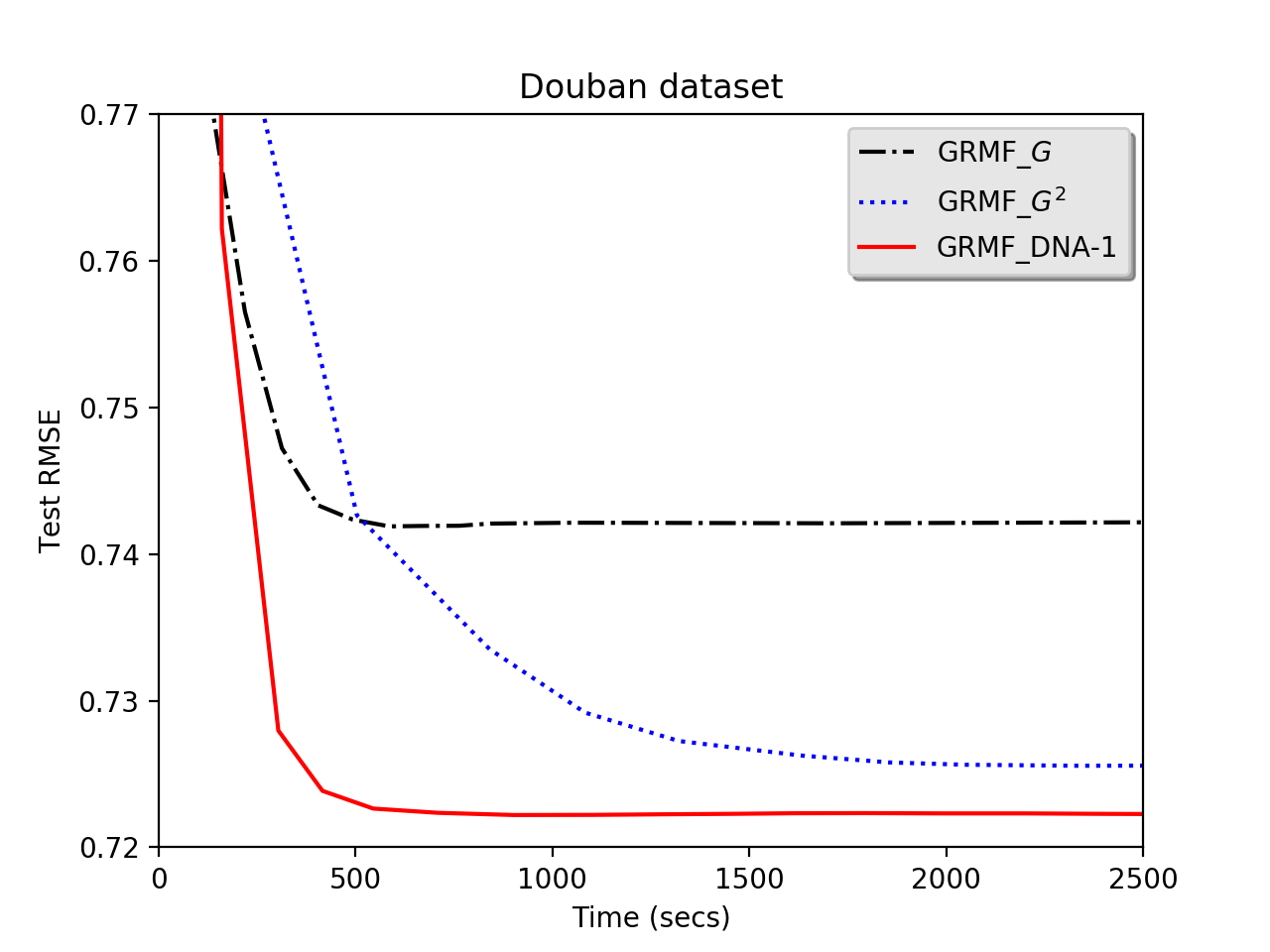}
  \end{center}
  \vspace{-5pt}
  \caption{Compare Training Speed of GRMF, with and without Graph DNA.}
  \label{fig:speed}
\end{figure}
We use the same split of three real-world datasets and follow the exact procedures as in \cite{berg2017graph, monti2017geometric}.
We tuned hyperparameters using a validation dataset and obtain the best test results found within 200 epochs using optimal parameters. We repeated the experiments 6 times and report the mean and standard deviation of test RMSE. After some tuning, we use the capacity of 10 Bloom filters for Douban and 60 for Flixster, as the latter has a much denser second-order graph. With a false positive rate of 0.1, this implies that we use 96-bits Bloom filters for Douban and 960 bits for Flixster. We use the resulting bloom filter bitarrays as the node features, and pass that as the input to GC-MC. Using Graph DNA-2, the input feature dimensions are thus reduced from 3000 to 96 and 960, 
which leads to a significant speed-up. The original GC-MC method did not scale up well beyond 3000 by 3000 rating matrices with the user and the item side graphs as it requires using normalized adjacency matrix as user/item features. PinSage \cite{ying2018graph}, while scalable, does not utilize the user/item side graphs. Furthermore, it is not feasible to have $O(n)$ dimensional features for the nodes, where $n$ is the number of nodes in side graphs. In contrast, our method only requires $O(\log(n))$ dimensional features. We can see from Table~\ref{tab:gcn} that we outperform both GCN-based methods \cite{berg2017graph} and \cite{monti2017geometric} in terms of performance by a large margin. 

Note that another potential way to improve over GC-MC is to use other graph encoding schemes like Node2Vec \cite{grover2016node2vec} and DeepWalk \cite{perozzi2014deepwalk} to encode the user-user graph into node features. One clear drawback is that those graph embedding methods are time-consuming. 
Using the official Node2vec implementation, excluding reading and writing, it takes 416.13 seconds to encode the 3K by 3K subsampled Yahoo-Music item graph and obtain resulting 760-d node embeddings. For our method, it only takes 7.55 seconds to obtain the same 760-d features. Similarly, it takes over 15 mins to run the official C++ codes for DeepWalk \cite{perozzi2014deepwalk} using the same parameters as Node2Vec to encode the graph. In fact, fast encoding via hashing and bitwise-or that does not require training is one of the main advantages of our method.  

Furthermore, even without considering the time overhead, we found our graph DNA encoding outperforms Node2Vec and DeepWalk in terms of test RMSE. Details can be found in Table~\ref{tab:gcn}. This could be due to that encoding higher-order information is more important for graph-regularized recommendation tasks, and graph DNA is a better and more direct way to encode higher order information compared with Node2Vec and DeepWalk. 
\vspace{-0.5em}

\paragraph{Speed Comparisons}\label{sec:speed}
Next, we compare the speed-ups obtained by graph DNA-$d$ with GRMF $G^d$ (a naive way to encode higher order information by computing powers of G).
Figure 3 suggests that graph DNA-1 (which encodes hop-2 information) scales better than directly computing $G^2$ in GRMF.

\begin{table}
  \caption{Comparison of GRMF Methods of different ranks for Explicit Feedback on Flixster Dataset.}
  \label{tab:rank}
  \resizebox{0.6\textwidth}{!}{
  \begin{tabular}{llccc}
    \toprule
    Rank & methods & test RMSE ($ \times 10^{-1}$)  & \% gain   \\
    \midrule
    \multirow{2}{*}{10}
    &GRMF\_$G^2$                    & 8.7849    & -     \\
    &GRMF\_DNA-3  & \bfseries{8.7383}  & \bfseries{0.8262}\\
    \midrule
    \multirow{2}{*}{20}
    &GRMF\_$G^2$                    & 8.9179   & -     \\
    &GRMF\_DNA-3  & \bfseries{8.7565}  & \bfseries{1.8098}\\
    \midrule
    \multirow{2}{*}{30}
    &GRMF\_$G^2$                   & 9.0865    & -     \\
    &GRMF\_DNA-3  & \bfseries{8.9255}    & \bfseries{1.7719}\\
  \bottomrule
\end{tabular}
}
\end{table}

\paragraph{Exploring Effects of Rank}\label{rank}
Finally,  we investigate whether the proposed DNA coding can achieve consistent improvements when varying the rank in the GRMF algorithm.
In Table~\ref{tab:rank}, we compare the proposed GRMF\_DNA-3 with GRMF\_$G^2$, which achieves the best RMSE without using DNA coding in the previous tables.
The results clearly show that the improvement of the proposed DNA coding is consistent over different ranks and works even better when rank is larger.

\vspace{-1em}
 \section{Conclusion}
 \label{sec:conclusion}
In this chapter, we proposed Graph DNA, a deep neighborhood aware encoding scheme for collaborative filtering with graph information. We make use of Bloom filters to incorporate higher order graph information, without the need to explicitly minimize a loss function. The resulting encoding is extremely space and computationally efficient, and lends itself well to multiple algorithms that make use of graph information, including Graph Convolutional Networks. Experiments show that Graph DNA encoding outperforms several baseline methods on multiple datasets in both speed and performance.

\clearpage
\newpage

\chapter{Large-scale Pairwise Collaborative Ranking in Near-Linear Time}

\section{Introduction}





In online retail and online content delivery applications, it is commonplace to have embedded recommendation systems-- algorithms that recommend items to users based on previous user behaviors and ratings.
Online retail companies develop sophisticated recommendation systems based on purchase behavior, item context, and shifting trends. 
The Netflix prize \cite{Bennett07thenetflix}, in which competitors utilize user ratings to recommend movies, accelerated research in recommendation systems.
While the winning submissions agglomerated several existing methods, one essential methodology, latent factor models, emerged as a critical component.
The latent factor model means that the approximated rating for user $i$ and item $j$ is given by $u_i^\top v_j$ where $u_i,v_j$ are $k$-dimensional vectors.
One interpretation is that there are $k$ latent topics and the approximated rating can be reconstructed as a combination of factor weights.
By minimizing the square error loss of this reconstruction we arrive at the incomplete SVD,
\begin{equation}\label{eq:MF}
\min_{U,V} \sum_{i,j \in \Omega} \left(R_{i,j} - u_i^\top v_j \right)^2 + \frac{\lambda}{2} (\|U\|_F^2 + \|V\|_F^2),
\end{equation}
where $\Omega$ contains sampled indices of the rating matrix, $R$.

Often the performance of recommendation systems is not measured by the quality of rating prediction, but rather the ranking of the items that the system returns for a given user.
The task of finding a ranking based on ratings or relative rankings is called Collaborative Ranking.
Recommendation systems can be trained with ratings, that may be passively or actively collected, or by relative rankings, in which a user is asked to rank a number of items.
A simple way to unify the framework is to convert the ratings into rankings by making pairwise comparisons of ratings.
Specifically, the algorithm takes as input the pairwise comparisons, $Y_{i,j,k}$
for each user $i$ and item pairs $j,k$.
This approach confers several advantages.
Users may have different standards for their ratings, some users are more generous with their ratings than others.
This is known as the calibration drawback, and to deal with this we must make a departure from standard matrix factorization methods.
Because we focus on ranking and not predicting ratings, we can expect improved performance when recommending the top items.
Our goal in this chapter is to provide a collaborative ranking algorithm that can scale to the size of the full Netflix dataset, a heretofore open problem.


The existing collaborative ranking algorithms, (for a summary see section~\ref{sec:related}), are limited by the number of observed ratings per user in the training data and cannot scale to massive datasets, therefore, making the recommendation results less accurate and less useful in practice. 
This motivates our algorithm, which can make use of the entire Netflix dataset without sub-sampling. 
Our contribution can be summarized below: 
\begin{itemize}
\item For input data in the form of pairwise preference comparisons, we propose a new algorithm Primal-CR that alternatively minimizes latent factors
using Newton's method in the primal space. By carefully designing the computation of gradient and Hessian vector product, our algorithm reduces the sample complexity per iteration to $O(|\Omega| + d_1 \bar{d}_2 r)$, while the state-of-the-art approach~\cite{park2015preference} have $O(|\Omega| r)$ complexity. Here $|\Omega|$ (total number of pairs), is much larger than $d_1 \bar{d}_2$ ($d_1$ is number of users and $\bar{d}_2$ is averaged number of items rated by a user). For the Netflix problem, $|\Omega|=2\times 10^{10}$ while $d_1\bar{d}_2=10^8$. 
\item For input data in the form of ratings, we can further exploit the structure to speedup the gradient and Hessian computation. The resulting
algorithm, Primal-CR++, can further reduce the time complexity to $O(d_1 \bar{d}_2(r+\log \bar{d}_2))$
per iteration. 
In this setting, our algorithm has time complexity near-linear to the input size, and have comparable speed with classical matrix factorization model that takes $O(d_1 \bar{d}_2 r)$ time, 
while we can achieve much better recommendation by minimizing the ranking loss. 
\end{itemize}
We show that our algorithms outperform existing algorithms on real world datasets and can be easily parallelized. 

\section{Related Work} \label{sec:related}
Collaborative filtering methodologies are summarized in \cite{schafer2007collaborative} (see \cite{deshpande2004item} for an early work).
Among them, matrix factorization~\cite{koren2009matrix} has been widely used due to the success in the Netflix Prize. 
Many algorithms have been developed based on matrix factorization~\cite{rendle2009bpr,rendle2010factorization,chiang2015matrix,hsieh2015pu,si2016goal}, and many scalable algorithms have been developed~\cite{koren2009matrix,gemulla2011large}. 
However, they are not suitable for ranking top items for a user due to the fact that their goal is to minimize the mean-square error (MSE) instead of 
ranking loss. 
In fact, MSE is not a good metric for recommendation when we want to recommend the top $K$ items to a user. 
This has been pointed out in several papers~\cite{balakrishnan2012collaborative} which argue normalized discounted
cumulative gain (NDCG) should be used instead of MSE, 
and our experimental results also confirm this finding by showing that minimizing the ranking loss results in better precision and NDCG compared with the traditional matrix factorization approach that is targeting squared error. 

Ranking is a well studied problem, and there has been a long line of research focuses on learning one ranking function, which is called Learning to Rank.
For example, RankSVM~\cite{TJ02a} is a well-known pair-wise model, and an efficient solver has been proposed in~\cite{chapelle2010efficient} for solving rankSVM. 
\cite{cao2007learning} is a list-wise model implemented using neural networks. 
Another class of point-wise models fit the ratings explicitly but has the issue of calibration drawback (see~\cite{hacker2009matchin}). 

The collaborative ranking (CR) problem is essentially trying to learn multiple rankings together, 
and several models and algorithms have been proposed in literature. 
The Cofirank algorithm~\cite{weimer2007maximum}, which tailors maximum margin matrix factorization~\cite{srebro2004maximum} for collaborative ranking, is a point-wise model for CR, and is regarded as the performance benchmark for this task. 
If the ratings are 1-bit, a weighting scheme is proposed to improve the usual point-wise Matrix Factorization approach~\cite{pan2008one}.
List-wise models for Learning to Rank can also be extended to many rankings setting,  
\cite{shi2010list}. 
However it is still quite similar to a point-wise approach since they only consider the top-1 probabilities. 

For pairwise models in collaborative ranking, it is well known that they do not encounter the calibration drawback as do point-wise models, but they are computationally intensive and cannot scale well to large data sets~\cite{shi2010list}. 
The scalability problem for pairwise models is mainly due to the fact that their time complexity is at least proportional to $|\Omega|$, the number of pairwise preference comparisons, which grows quadratically with number of rated items for each user. 
Recently, \cite{park2015preference} proposed a new Collrank algorithm, and they showed that 
Collrank has better precision and NDCG as well as being much faster compared with other CR methods on real world datasets, including Bayesian Personalized Ranking (BPR)~\cite{rendle2009bpr}. 
Unfortunately their scalability is still constrained by number of pairs, so they can only run on subsamples for large datasets, such as Netflix. 
In this chapter, our algorithm Primal-CR and Primal-CR++ also belong to the family of pairwise models, but due to cleverly re-arranging the computation, 
we are able to have much better time complexity than existing ones, and as a result our algorithm can scale to very large datasets. 


There are many other algorithms proposed for many rankings setting but none of these mentioned below can scale up to the extent of using all the ratings in the full Netflix data. There are a few using Bayesian frameworks to model the problem  \cite{rendle2009bpr}, \cite{pan2013gbpr}, \cite{volkovs2012collaborative}, the last of which requires many specified parameters. Another one proposed retargeted matrix factorization to get ranking by monotonically transforming the ratings \cite{koyejo2013retargeted}.  \cite{gunasekar2016preference} proposes a similar model without making generative assumptions on ratings besides assuming low-rank and correctness of the ranking order. 


\section{Problem Formulation}

We first formally define the collaborative ranking problem using the example of
item recommender system. 
Assume we have $d_1$ users and $d_2$ items, the input data is given in the form of
``for user $i$, item $j$ is preferred over item $k$" and thus can be represented by a set of tuples $(i, j, k)$. 
We use $\Omega$ to denote the set of observed
tuples, and the observed pairwise preferences are denoted as 
$\{Y_{ijk} \mid (i,j,k)\in \Omega\}$, 
where 
$Y_{ijk} = 1$ denotes that item $j$ is preferred over item $k$ for a particular user $i$ and $Y_{ijk} = -1$ to denote that item $k$ is preferred over item $j$ for user $i$.  

The goal of collaborative ranking is to rank all the unseen items for each user $i$
based on these partial observations, which can be done by 
fitting a scoring matrix $X \in \mathbb{R}^{d_1 \times d_2}$. If the scoring matrix has $X_{ij} > X_{ik}$, it implies that item $j$ is preferred over item $k$ by the particular user $i$ and therefore we should give higher rank for item $j$ than item $k$. After we estimate the scoring matrix $X$ by solving the optimization problem described below, 
we can then recommend top $k$ items for any particular user. 

The Collaborative Ranking Model referred to in this chapter is the one proposed recently in~\cite{park2015preference}. It belongs to the family of pairwise models for  collaborative ranking because it uses pairwise training losses~\cite{balakrishnan2012collaborative}. 
%
The model is given as 
\begin{equation} \label{eq:convex3}
\min_{X} \sum_{(i, j, k) \in \Omega} \\L(Y_{ijk} (X_{ij} - X_{ik})) + \lambda ||X||_*, 
\end{equation}
where $\\L(.)$ is the loss function, $\|X\|_*$ is the nuclear norm regularization defined by the sum of all the singular value of the matrix $X$, and $\lambda$ is a regularization parameter. 
The ranking loss defined in the first term of~\eqref{eq:convex3} penalizes the pairs when $Y_{ijk}=1$ but $X_{ij}-X_{ik}$ is positive but small, and penalizes even more when the difference is negative. The second term in the loss function is based on the assumption that there are only a small
number of latent factors contributing to the users' preferences which is analogous to the idea behind incomplete SVD for matrix factorization mentioned in the introduction.
In general we can use any loss function, but since $\\L_2$-hinge loss defined as 
\begin{equation}
\\L(a) = \max(0, 1-a)^2
\label{eq:l2_hinge}
\end{equation} 
gives the best performance in practice~\cite{park2015preference} and enjoys many nice
properties, such as smoothness and differentiable, we will focus on $L_2$-hinge loss in this chapter. In fact, our first algorithm Primal-CR
can be applied to any loss function, while Primal-CR++ can only be applied to $L_2$-hinge loss. 

Despite the advantage of the objective function in equation~\eqref{eq:convex3} being convex, it is still not feasible for large-scale problems since $d_1$ and $d_2$ can be very large so that the scoring matrix $X$ cannot be stored in memory, not to mention how to solve it. 
Therefore, in practice people usually transform~\eqref{eq:convex3} to a non-convex form by replacing
$X=UV^T$, and in that case since $\|X\|_* = min_{X = U^T V} \frac{1}{2}(\|U\|_F^2 + \|V\|_F^2)$ \cite{srebro2004maximum}, problem~\eqref{eq:convex3} can be reformulated as
%
\begin{equation}
\min_{U, V} \sum_{(i,j,k)\in \Omega}  \\L(Y_{ijk} \cdot u_i^T (v_j - v_k)) + \frac{\lambda}{2} (\|U\|^2_F + \|V\|^2_F), 
\label{eq:obj}
\end{equation}
We use $u_i$ and $v_j$ denote columns of $U$ and $V$ respectively. 
Note that~\cite{park2015preference} also solves the non-convex form~\eqref{eq:obj} 
in their experiments, and in the rest of the paper we will propose a faster algorithm 
for solving~\eqref{eq:obj}. 

\section{Proposed Algorithms}
\subsection{Motivation and Overview}

Although collaborative ranking assumes that input data is given in the form of
pairwise comparisons, in reality almost all the datasets (Netflix, Yahoo-Music, MovieLens, etc) contain user ratings to items in the form 
of $\{R_{ij} \mid (i,j)\in \bar{\Omega} \}$, where $\bar{\Omega}$ is the subset of observed user-item pairs. 
Therefore, in practice we have to transform the rating-based data into pair-wise comparisons by generating
all the item pairs rated by the same user: 
\begin{equation}
   \Omega = \{ (i,j,k) \mid j,k \in \bar{\Omega}_i \}, 
   \label{eq:pairs}
\end{equation}
where $\bar{\Omega}_i:=\{j\mid (i,j)\in \bar{\Omega} \}$ is the set of items rated by user $i$. 
Assume there are averagely $\bar{d_2}$ items rated by a user (i.e., $\bar{d_2}=\text{mean}(|\bar{\Omega}_i|)$), 
then the collaborative ranking problem will have $O(d_1 \bar{d_2}^2)$ pairs and thus the size of $\Omega$ 
grows quadratically. 

Unfortunately, all the existing algorithms have $O(|\Omega| r)$ complexity, so they cannot scale to large number
of items. For example, the AltSVM (or referred to as Collrank) Algorithm in~\cite{park2015preference} will run out of memory when we subsample 500 rated items per user on Netflix dataset since its implementation\footnote{Collrank code is available on \url{https://github.com/dhpark22/collranking}.} stores all the pairs in memory and therefore requires $O(|\Omega|)$ memory. So it cannot be used for the full Netflix dataset which has
more than {\it 20 billion pairs} and requires 300GB memory space. 
To the best of our knowledge, no collaborative ranking algorithms have been applied to the full Netflix data set. But in real life, we hope to make use of as much information as possible to make better recommendation. As shown in our experiments later, using full training data instead of sub-sampling (such as selecting a fixed number of rated items per user) achieves higher prediction and recommendation accuracy for the same test data.

To overcome this scalability issue, we propose two novel algorithms for solving problem~\eqref{eq:obj}, and both of them significantly
reduce the time complexity over existing methods. 
If the input file is in the form of $|\Omega|$ pairwise comparisons, our proposed algorithm, Primal-CR, can reduce the time and space complexity from $O( |\Omega| r) $ to $O( |\Omega| + d_1 \bar{d_2} r)$, where $\bar{d_2}$ is the average number of items compared by one user. 
If the input data is given as user-item ratings (e.g., Netflix, Yahoo-Music), the complexity is reduced 
from $O(d_1  \bar{d_2}^2  r)$ to $O( d_1 \bar{d_2} r + d_1 \bar{d_2} ^2  ) $. 

If the input file is given in ratings, we can further reduce the time complexity to $O( d_1 \bar{d_2} r + d_1 \bar{d_2} \log\bar{d_2} )$ using 
exactly the same optimization algorithm but smarter ways to compute gradient and Hessian vector product. This time complexity is much smaller than the number of comparisons $|\Omega|=O(d_1\bar{d}_2^2)$, and we call this algorithm Primal-CR++. 

We will first introduce Primal-CR in Section~\ref{sec:Primal}, and then present Primal-CR++ in Section~\ref{sec:Primal++}. 

\subsection{Primal-CR: the proposed algorithm for pairwise input data}
\label{sec:Primal}
\begin{algorithm}
\caption{Primal-CR / Primal-CR++: General Framework \label{alg:primal_general}}
\begin{algorithmic}[1]
\Require $\Omega$, $\{Y_{ijk}: (i, j, k) \in \Omega \}$, $\lambda\in \R^+$ \Comment{for Primal-CR}
\Require $M \in \R^{d_1 \times d_2}$, $\lambda\in \R^+$ \Comment{for Primal-CR++}
\Ensure $U\in \R^{r \times d_1}$ and $V \in \R^{r \times d_2}$
\State Randomly initialize $U, V$ from Gaussian Distribution
\While{not converged}
\Procedure{Fix $U$ and update $V$}{}
	\While{not converged}
	\State Apply truncated Newton update (Algorithm~\ref{alg:newton})
	\EndWhile
\EndProcedure
\Procedure{Fix $V$ and update $U$}{}
	\While{not converged}
	\State Apply truncated Newton update (Algorithm~\ref{alg:newton})
	\EndWhile
\EndProcedure
\EndWhile
\State \textbf{return} $U, V$\Comment{recover score matrix $X$}
\end{algorithmic}
\end{algorithm}

\begin{algorithm}
\caption{Truncated Newton Update for $V$ (same procedure can be used for updating $U$)
\label{alg:newton}}
\begin{algorithmic}[1]
\Require Current solution $U, V$
\Ensure $V$
\State Compute $g = \text{vec}(\nabla f(V))$
\State Let $H=\nabla^2 f(V)$ (do not explicitly compute $H$)
\Procedure{Linear Conjugate Gradient}{g, F}
\State Initialize $\delta_0 = 0$ 
\State $r_0 = H \delta_0 - g$, $p_0 = - r_0$
	\For{$k = 0, 1, ..., maxiter$}
	    \State Compute the Hessian-vector product $q = H p_k$ 
		\State $\alpha_k = - r_k ^ T p_k / p_k^T q$ 
		\State $\delta_{k+1} = \delta_k + \alpha_k p_k$
		\State $r_{k + 1} = r_k + \alpha_k q$
		\If{$||r_{k + 1}||_2 < ||r_0||_2 \cdot 10^{-2}$}
			\State \textbf{break}
		\EndIf
		\State $\beta_{k + 1} = (r_{k + 1} q) / p_k^T q$
		\State $p_{k + 1} = -r_{k + 1} + \beta_{k + 1} p_k$
	\EndFor
	\State \textbf{return} $\delta$ 
\EndProcedure
\State  $V = V - s\delta$ (stepsize $s$ found by line search) 
\State \textbf{return} $U$ or $V$ 
\end{algorithmic}
\end{algorithm}

In the first setting, we consider the case where the pairwise comparisons $\{Y_{ijk} \mid (i,j,k)\in \Omega\}$ are given as input. 
To solve problem~\eqref{eq:obj}, 
we alternatively minimize $U$ and $V$ in the primal space (see Algorithm~\ref{alg:primal_general}). 
First, we fix U and update V, and the subproblem for V while U is fixed can be written as follows: 
\begin{equation} 
V = \argmin_{V \in \R^{r \times d_2}} \bigg\{ \frac{\lambda}{2} ||V||^2_F + \sum_{(i,j,k)\in \Omega}   \\L(Y_{ijk}\cdot u_i^T (v_j - v_k)) \bigg\} := f(V)
 \label{eq:problem_V} 
\end{equation} 
In~\cite{park2015preference}, this subproblem is solved by stochastic dual coordinate descent, which requires  $O( |\Omega| r)$ time and $O(|\Omega|)$ 
space complexity. Furthermore, the objective function decreases for the dual problem sometimes does not imply the decrease of primal objective function value, which often results in slow convergence.  
We therefore propose to solve this subproblem for $V$ using the primal truncated Newton method (Algorithm~\ref{alg:newton}).

Newton method is a classical second-order optimization algorithm. 
For minimizing a vector-valued function $f(x)$, Newton method iteratively updates the 
solution by $x\leftarrow x- (\nabla^2 f(x))^{-1} \nabla f(x)$. However, the matrix inversion is usually 
hard to compute, so a truncated Newton method computes the update direction by solving the linear system
$\nabla^2 f(x) a = \nabla f(x) $ up to a certain accuracy, usually using a linear conjugate gradient method. 
If we vectorized the problem for updating $V$ in eq~\eqref{eq:problem_V}, the gradient is a $(rd_2)$-sized vector
and the Hessian is an $(rd_2)$-by-$(rd_2)$ matrix, so explicitly forming the Hessian is impossible. Below we discuss
how to apply the truncated Newton method to solve our problem, and discuss efficient computations for each part. 

{\bf Derivation of Gradient. }
When applying the truncated Newton method, 
the gradient $\nabla f(V)$ is a $\R^{r\times d_2}$ matrix and can be computed explicitly: 
\begin{equation} 
\nabla f(V) = \sum_{i=1}^{d_1} \sum_{(j, k) \in \Omega_i}   \\L'(Y_{ijk} \cdot u_i^T (v_j - v_k))(u_i e_j^T - u_i e_k^T) Y_{ijk} + \lambda V,  
\label{eq:gradient}
\end{equation}
where $\nabla f(V) \in \R^{r \times d_2}$, 
$\Omega_i:=\{(j,k)\mid (i,j,k)\in \Omega\}$ is the subset of pairs that associates with user $i$, 
and $e_j$ is the indicator vector used to add the $u_i$ vector to the $j$-th column of the output matrix. 
The first derivative for $L_2$-hinge loss function~\eqref{eq:l2_hinge} is 
\begin{equation}
\\L'(a) = 2 \min(a - 1, 0)
\end{equation}
For convenience, we define $g:=\text{vec} (\nabla f(V))$ to be the vectorized form of gradient. 
One can easily see that computing $g$ naively by going through all the pairwise comparisons $(j, k)$ and adding up arrays is time-consuming and has $O(|\Omega| r)$
time complexity, which is the same with Collrank~\cite{park2015preference}. 

\paragraph{\bf Fast computation for gradient}
Fortunately, we can reduce the time complexity to $O(|\Omega| + d_1 \bar{d}_2 r)$ by smartly rearranging the computations, so that the time is only linear to 
$|\Omega|$  and $r$, but not to $|\Omega|r$. The method is described below. 

First, for each $i$, the first term of~\eqref{eq:gradient} can be represented by
\begin{equation}
    \sum_{(j, k) \in \Omega_i}   \\L'(Y_{ijk} \cdot u_i^T (v_j - v_k))(u_i e_j^T - u_i e_k^T) Y_{ijk} = \sum_{j\in \bar{d}_2(i)} t_j u_i e_j^T, 
    \label{eq:gradient_inner}
\end{equation}
where $\bar{d}_2(i) := \{j\mid \exists k \text{ s.t. } (i,j,k)\in \Omega \}$ and
$t_j$ is some coefficient computed by summing over all the pairs in $\Omega_i$. 
If we have  $t_{j}$, the overall gradient
can be computed by $O(\bar{d}_2(i) r)$ time for each $i$. 
To compute $t_j$, we first compute $u_i^T v_j$ for all $j\in \bar{d}_2(i)$ in $O(\bar{d}_2(i)r)$ time, and then go through all the $(j,k)$ pairs while keep adding the coefficient related
to this pair to $t_j$ and $t_k$. Since there is no vector operations when we go through all pairs, this step only
takes $O(\Omega_i)$ time. After getting all $t_j$, we can then conduct $\sum_{j\in\bar{d}_2(i)} t_j u_i e_j^T$ in $O(\bar{d}_2(i)r)$ time.  
Therefore, the overall complexity can be reduced to 
$O(|\Omega| + d_1 \bar{d_2} r)$. The pseudo code is presented in Algorithm~\ref{alg:compute_g}. 


\begin{algorithm}
\caption{Primal-CR: efficient way to  compute $\nabla f(V)$ \label{alg:compute_g}}
\begin{algorithmic}[1]
\Require $\Omega$, $\{Y_{ijk}: (i, j, k) \in \Omega \}$, $\lambda\in \R^+$, 
current variables $U,V$
\Ensure $g,m$ \Comment{$g\in \R^{d_2 r}$ is the gradient for $f(V)$}
\State Initialize $g = 0$ \Comment{$g \in \R^{r \times d_2}$}
\For {$i = 1, 2, \dots, d_1$}
	\ForAll{$j \in \bar{d_2}(i)$}
	\State precompute $u_i^T v_j$ and store in a vector $m_i$ 
	\EndFor
		\State Initialize a zero array $t$ of size $d_2$
 	\ForAll {$(j, k) \in \bar{d_2}(i)$}
	\If{$Y_{ijk}(m_i[j] - m_i[k]) < 1$}
		\State $s = 2 (Y_{ijk} (m_i[j]-m_i[k])-1)$
		\State $t[j] \pluseq Y_{ijk}s$
		\State $t[k] \mineq Y_{ijk}s$  \Comment{O(1) time per for loop iteration}
	\EndIf
	\EndFor
	
	\ForAll{$j\in \bar{d}_2(i)$ }
		\State $g[:, j] \pluseq  t[j] \cdot u_i$
	\EndFor
\EndFor
\State $g =  vec(g + \lambda V)$ \Comment{vectorize matrix $g\in \R^{r \times d_2}$}
\State Form a sparse matrix $m=[m_1 \dots m_{d_1}]$  \Comment{$m$ can be reused later}
\State \textbf{return} $g,m$
\end{algorithmic}
\end{algorithm}

\paragraph{\bf Derivation of Hessian-vector product}
Now we derive the Hessian $\nabla^2 f(V)$ for $f(V)$. We define $\nabla_j f(V):=\frac{\partial}{\partial v_j} f(V) \in \R^r $ and $\nabla^2_{j,k}f(V) := \frac{\partial^2}{\partial v_j \partial v_k} f(V) \in\R^{r\times r}$ in the following derivations. 
From the gradient derivation, we have
\begin{equation*} 
\nabla_j f(V) = \sum_{i : j\in \bar{d_2}(i)} \sum_{\substack{k \in \bar{d_2}(i) \\ k \neq j}}   \\L'(Y_{ijk} \cdot u_i^T (v_j - v_k)) u_i  Y_{ijk} + \lambda v_j.
\end{equation*}
Taking derivative again we can obtain
\begin{align*}
&\nabla^2_{j, k} f(V) = \\
&\begin{cases} 
\sum_{i : (j, k)\in \bar{d_2}(i)}   \\L''(Y_{ijk} \cdot u_i^T (v_j - v_k)) (-u_i u_i^T) 
&\text{ if $j\neq k$ } \\
\sum_{i : j\in \bar{d_2}(i)} \sum_{\substack{k \in \bar{d_2}(i),  k \neq j}}   \\L''(Y_{ijk} \cdot u_i^T (v_j - v_k)) u_i u_i^T + \lambda I_{r \times r} &\text{ if $j=k$} 
\end{cases}
\end{align*}
and the second derivative for $\\L_2$ hinge loss function is given by:
\begin{equation}
\\L''(a) = \begin{cases}
2 &\text{if $a \leq 1$} \\ 0 &\text{if $a > 1$}. 
\end{cases}
\end{equation}
Note that if we write the full Hessian $H$ as a $(d_2 r)$ by $(d_2 r)$ matrix, then $\nabla^2_{j, k} f(V) $
is an $r \times r$ block in $H$, where there are totally $d_2^2$ of these blocks. 
In the CG update for solving $H^{-1}g$, we only need to compute $H\cdot a$ for some $a  \in \R^{d_2 r}$. 
For convenience, we also partition this $a$ into $d_2$ blocks, each subvector $a_j$ has size $r$, 
so $a=[a_1; \cdots; a_j]$. 
Similarly we can use subscript to denote the subarray $(H\cdot a)_j$ of the array $H\cdot a$, which becomes
\begin{align}
(H \cdot a)_j &= \sum_{k \neq j} \nabla^2_{j, k} f(V) \cdot a_k + \nabla^2_{j, j} f(V) \cdot a_j \\
&= \lambda a_j + \sum_{i: j\in \bar{d}_2(i) } u_i \sum_{\substack{k\in \bar{d}_2(i) \\ k\neq j}} 
\\
&+ L''(Y_{ijk}\cdot u_i^T (v_j - v_k)) (u_i^T a_j - u_i^T a_k). 
\end{align}
Therefore, we have
\begin{align}
H \cdot a &= \sum_j E_j (H\cdot a)_j \nonumber\\
&= \lambda a + \sum_i \sum_{j \in \bar{d_2}(i)} E_j u_i \sum_{\substack{k \in \bar{d_2}(i)\\ k \neq j}} \\
&= L''(Y_{ijk}\cdot u_i^T (v_j - v_k)) (u_i^T a_j - u_i^T a_k)
\label{eq:hessian_full}
\end{align} 
where $E_j$ is the projection matrix to the $j$-th block, indicating that we are only adding $(H\cdot a)_j$ to the $j$-th block of matrix, and setting $0$ elsewhere. 

\begin{algorithm}
\caption{Primal-CR: efficient way to compute Hessian vector product \label{alg:compute_H}}
\begin{algorithmic}[1]
\Require $\Omega$, $\{Y_{ijk}: (i, j, k) \in \Omega \}$, $\lambda\in \R^+$, $a \in \R^{d_2 r}, m$, $U,V$
\Ensure $Ha$ \Comment{$Ha\in \R^{d_2 r}$ is needed in Linear CG}
\State $Ha = 0$  	\Comment{$Ha \in \R^{d_2 r}$}
\For {$i = 1, 2, \dots, d_1$}
	\ForAll{$j \in \bar{d_2}(i)$}
		\State precompute $u_i^T a_j$ and store it in array $b$
	\EndFor
	\State Initialize a zero array t of size $d_2$
 	\ForAll {$(j, k) \in \bar{d_2}(i)$} 		
		\If{$Y_{ijk} (m_i[j] - m_i[k]) < 1.0$} 
			\State $s_{jk} = 2.0 \cdot (b[j] - b[k])$
			\State $t[j] \pluseq s_{jk}$
			\State $t[k] \mineq s_{jk}$	\Comment{O(1) time per for loop iteration}
		\EndIf

	\EndFor
	
	\ForAll{$j\in \bar{d}_2(i)$}
		\State $(Ha)[(p - 1)\cdot r + 1: p \cdot r] \pluseq t[j] \cdot u_i$
	\EndFor
\EndFor
\State \textbf{return} $Ha$
\end{algorithmic}
\end{algorithm}

\paragraph{\bf Fast computation for Hessian-vector product }
Similar to the case of gradient computation, using a naive way to compute $H\cdot a$ requires 
$O(|\Omega|r)$ time since we need to go through all the $(i,j,k)$ tuples, and each 
of them requires $O(r)$ time. 
However, we can apply the similar trick in gradient computation to reduce the time complexity to $O(|\Omega|+ d_1 \bar{d}_2 r)$
by pre-computing $u_i^T a_j$ and caching the coefficient using the array $t$. The detailed algorithm is given in Algorithm~\ref{alg:compute_H}. 

Note that in Algorithm~\ref{alg:compute_H}, we can reuse the $m$ (sparse array storing the current prediction) which has been pre-computed in the gradient computation
(Algorithm~\ref{alg:compute_g}), and that will cost only $O(d_1 \bar{d}_2)$ memory. Even without storing the $m$ matrix, we can compute $m$ in the loop of line 4 in Algorithm~\ref{alg:compute_H}, which will not increase the overall computational complexity. 



\paragraph{\bf Fix $V$ and Update $U$}
After updating $V$ by truncated Newton, 
we need to fix $V$ and update $U$. The subproblem for $U$ can be written as:
\begin{equation}\label{eq:solveU}
U = \argmin_{U \in \R^{r \times d_2}} \{ \frac{\lambda}{2} ||U||^2_F + \sum_{i=1}^{d_1} \sum_{(j, k) \in \bar{d}_2(i)}   \\L(Y_{ijk}\cdot u_i^T (v_j - v_k)) \}  
\end{equation}

Since $u_i$, the $i$-th column of $U$, is independent from the rest of columns, equation~\ref{eq:solveU} can be decomposed into $d_1$ independent problems for $u_i$: 
\begin{equation} 
u_i = \argmin_{u \in \R^r} \frac{\lambda}{2} ||u||^2_2 + \sum_{(j, k) \in \bar{d_2}(i)}  \\L(Y_{ijk} \cdot u^T (v_j - v_k)) 
:= h(u)
\label{eq:ranksvm}
\end{equation}
Eq~\eqref{eq:ranksvm} is equivalent to an $r$-dimensional rankSVM problem. Since $r$ is usually small, 
the problems are easy to solve. 
In fact, we can directly apply an efficient rankSVM algorithm proposed in~\cite{chapelle2010efficient} to solve each $r$-dimensional rankSVM problem. 
This algorithm requires $O(|\Omega_i| + r|\bar{d}_2(i)|)$
time for solving each subproblem with respect to $u_i$, so the overall
complexity is $O(|\Omega| + rd_1 \bar{d}_2)$ time per iteration. 


\paragraph{\bf Summary of time and space complexity}
When updating $V$, we first compute gradient by Algorithm~\ref{alg:compute_g}, which takes $O(|\Omega|+d_1\bar{d}_2 r)$
time, and each Hessian-vector product in~\ref{alg:compute_H} also takes the same time. The updates for $U$ takes the same
time complexity with updating $V$, so the overall time complexity is $O(|\Omega|+d_1\bar{d}_2 r)$ per iteration. The whole algorithm only needs
to store size $d_1\times r$ and $d_2\times r$ matrices for gradient and conjugate gradient method. 
The $m$ matrix in Algorithm~\ref{alg:compute_g} is not needed, but in practice we find it can speedup the code
by around 25\%, and it only takes $d_1 \bar{d}_2\leq |\Omega|$ memory space (less than the input size). Therefore, our
algorithm is very memory-efficient. 

Before going to Primal-CR++, we discuss the time complexity of Primal-CR when the input data is the user-item rating matrix. Assume $\bar{d}_2$ is the averaged number of rated items per user, then there will be $|\Omega|= O(d_1 \bar{d}_2^2)$ pairs, leading to $O(d_1\bar{d}_2^2 + d_1 \bar{d}_2 r)$ time complexity for Primal-CR. This is much better than the $O(d_1\bar{d}_2^2 r)$ complexity for all the existing algorithms. 

\begin{algorithm}
\caption{Primal-CR++: compute gradient part for $f(V)$}
\begin{algorithmic}[1]
\Require $M \in \R^{d_1 \times d_2}$, $\lambda\in \R^+$, current $U, V$ 
\Ensure $g$ \Comment{$g\in \R^{d_2 r}$ is the gradient for $f(V)$}
\State Initialize $g = 0$ \Comment{$g \in \R^{r \times d_2}$}
\For {$i = 1, 2,  ..., d_1$}
\State Let $\bar{d}_2 = |\bar{d}_2(i)|$ and $r_j=R_{i,j}$ for all $j$.  
	\State Compute $m[j] = u_i^T v_j$ for all $j\in \bar{d}_2(i)$
	\State Sort $\bar{d}_2(i)$ according to the ascending order of $m_i$, so
	$m[\pi(1)] \leq \dots \leq m[\pi(\bar{d}_2)]$
	\State Initialize $s[1], \dots, s[L]$ and $c[1], \dots, c[L]$ with 0
	\State (Store in segment tree or Fenwick tree. )
	\State $p\leftarrow 1$ 
	\ForAll{$j =1, \dots, \bar{d}_2 $}
	\While{$m[\pi(p)] \leq m_j+1$}
	    \State $s[r_{\pi(p)}] += m[\pi(p)]$, $c[r_{\pi(p)}]+= 1$  
	    \State $p+=1$
	\EndWhile
	\State $\text{S} = \sum_{\ell\geq r_{\pi(p)}} s[\ell], \text{C} = \sum_{\ell\geq r_{\pi(p)}} c[\ell]$ 
	\State $t^+[\pi(j)] = 2 (C \cdot (m[\pi(j)]+1) - S)$
	\EndFor
		
	\State Do another scan $j$ from $\bar{d}_2$ to $1$ to compute $t^-[\pi(j)]$ for all $j$
		\State $g[:, j] \pluseq (t^+[j]+t^-[j]) \cdot u_i$ for all $j$
	\EndFor
\State $g = vec(g+\lambda V)$ 	 \Comment{vectorize matrix $g\in \R^{r \times d_2}$}
\State \textbf{return} $g$  
\end{algorithmic}
\end{algorithm}

\subsection{Primal-CR++: the proposed algorithm for rating data}\label{sec:Primal++}

Now we discuss a more realistic scenario, where the input data is a rating matrix 
$\{R_{ij}\mid (i,j) \in \bar{\Omega}\}$ and $\bar{\Omega}$ is the observed set of user-item
ratings. We assume there are only $L$ levels of ratings, so $R_{ij}\in \{1, 2, \dots, L\}$. 
Also, we use $\bar{d}_2(i):=\{j\mid (i,j)\in \bar{\Omega}\}$ to denote the rated items for 
user $i$.

Given this data, the goal is to solve the collaborative ranking problem~\eqref{eq:obj} with 
all the pairwise comparisons in the rating dataset as defined in~\eqref{eq:pairs}. There are totally $O(d_1 \bar{d}_2^2)$ pairs, and the question is: Can we have an algorithm with near-linear
time with respect to number of observed ratings $|\bar{\Omega}|=d_1 \bar{d}_2$? 
We answer
this question in the affirmative by proposing Primal-CR++, a near-linear time
algorithm for solving problem~\eqref{eq:obj} with L2-hinge loss. 

The algorithm of Primal-CR++ is exactly the same with Primal-CR, but we use a smarter
algorithm to compute gradient and Hessian vector product in near-linear time, by exploiting
the structure of the input data. 

We first discuss how to speed up the gradient computation of~\eqref{eq:gradient}, where
the main computation is to compute~\eqref{eq:gradient_inner} for each $i$. 
When the loss function is L2-hinge loss, 
we can explicitly write down the coefficients $t_j$ in~\eqref{eq:gradient_inner} by
\begin{equation}
    t_j = \sum_{k\in \bar{d}_2(i)} 2 (m_j - m_k-Y_{ijk}) I[Y_{ijk} (m_j-m_k) \leq 1], 
    \label{eq:aabb}
\end{equation}
where $m_j:=u_i^T v_j$ and $I[\cdot]$ is an indicator function such that $I[a\leq b]=1$ if $a\leq b$, and $I[a\leq b]=0$ otherwise. 
By splitting the cases of $Y_{ijk}=1$ and $Y_{ijk}=-1$, we get
\begin{align}
    t_j &= t_j^{+} + t_j^{-} \nonumber\\
    &= \!\!\!\!\sum_{\substack{k\in \bar{d}_2(i) \\ m_k\leq m_j+1, \  Y_{ijk}=-1} }
    \!\!\!\! 2(m_j - m_k+1) +\!\!\!\! \sum_{\substack{k\in \bar{d}_2(i) \\ m_k\geq m_j-1 , \ Y_{ijk}=1} }
    \!\!\!\!2(m_j - m_k-1). 
    \label{eq:ccdd}
\end{align}
Assume the indexes in $\bar{d}_2(i)$ are sorted by the the ascending order of $m_j$. Then we can scan from left to right, and maintain
the current accumulated sum $s_1, ..., s_L$ and the current index counts $c_1, ..., c_L$ for each rating level. If the current pointer is $p$, then
\begin{equation*}
    s_\ell[p] = \sum_{j: m_j\leq p, R_{ij}=\ell} m_j \ \text{ and } \ c_\ell[p] = |\{j: m_j\leq p, R_{ij} = \ell\}|. 
\end{equation*}
Since we scan from left to right, these numbers can be maintained in constant time
at each step. Now assume we scan over the numbers $m_1+1, m_2+1, \dots$, then at each 
point we can compute
\begin{equation*}
    t_j^{+} = \sum_{\ell = R_{i,j}+1}^L  2\{(m_j+1) c_\ell[m_j+1] - s_\ell[m_j+1]\},  
\end{equation*}
which can be computed in $O(L)$ time. 

Although we observe that $O(L)$ time is already small in practice (since $L$ usually smaller than 10), in the following we show there is a way to remove the dependency on $L$ by using 
a simple Fenwick tree~\cite{PF94a}, F+tree~\cite{yu2015scalable} or segment tree. 
If we store the set $\{s_1, \dots, s_L\}$ in Fenwick tree, 
then each query of 
$\sum_{\ell \geq r} s_i$ can be done in 
in $O(\log L)$ time, and since each step we only need to change one element into the set, the updating
time is also $O(\log L)$. Note that $t_j^-$ can be computed in the same way by scanning from largest $m_j$ to the smallest one. 

To sum up, the algorithm first computes all $m_j$ in $O(\bar{d}_2 r)$ time, then sort these numbers using $O(\bar{d}_2 \log \bar{d}_2)$ time, and then compute $t_j$ for all $j$ using two linear scans in $O(\bar{d}_2 \log L)$ time. Here $\log L$ is dominated by $\log \bar{d}_2$
since $L$ can be the number of unique rating levels in the current set $\bar{d}_2(i)$. Therefore, after computing this for all users $i=1, \dots, d_1$, the
time complexity for computing gradient is
\begin{equation*}
    O(d_1 \bar{d}_2 \log \bar{d}_2 + d_1 \bar{d}_2 r ) 
    = O(|\bar{\Omega}| ( \log \bar{d}_2 + r )). 
\end{equation*}

\begin{figure*}
\begin{tabular}{ccc}
\hspace{-8pt}\includegraphics[width=0.33\linewidth]{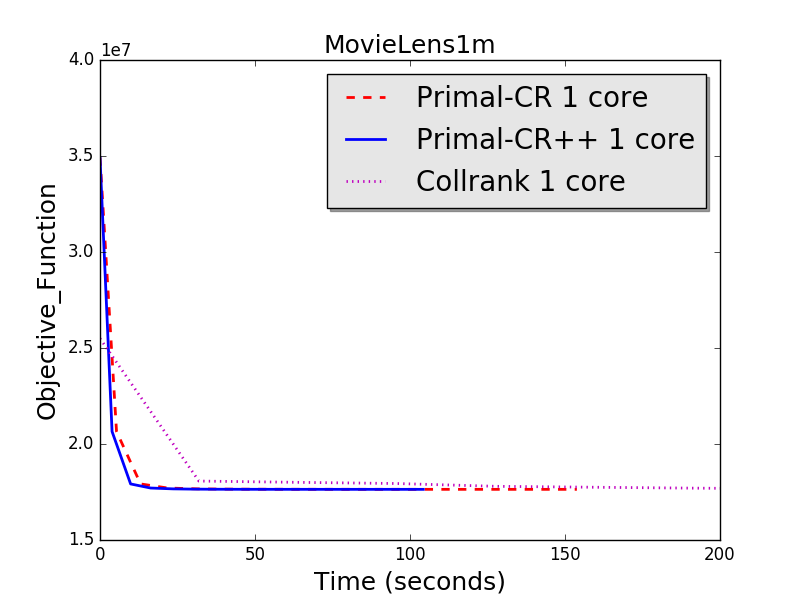} &
\hspace{-14pt}\includegraphics[width=0.33\linewidth]{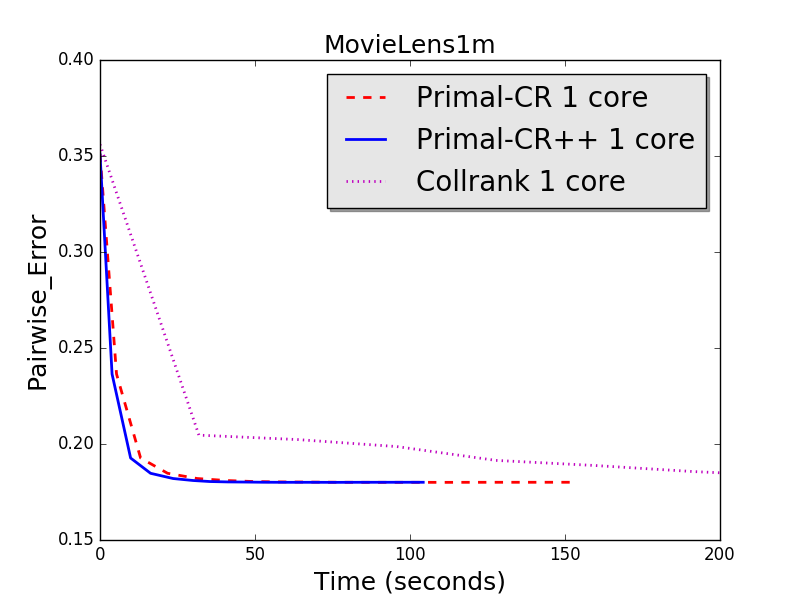} &
\hspace{-14pt}\includegraphics[width=0.33\linewidth]{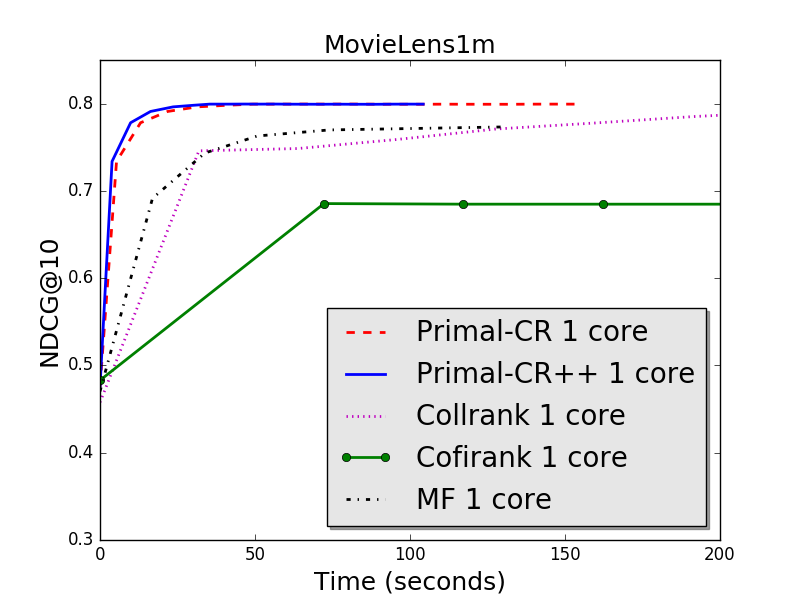}
\end{tabular}
\caption{Comparing Primal-CR, Primal-CR++ and Collrank, MovieLens1m data, 200 ratings/user, rank 100, lambda = 5000 \label{fig:serial_ml1m}}
\end{figure*}

\begin{figure*}
\begin{tabular}{ccc}
\hspace{-8pt}\includegraphics[width=0.33\linewidth]{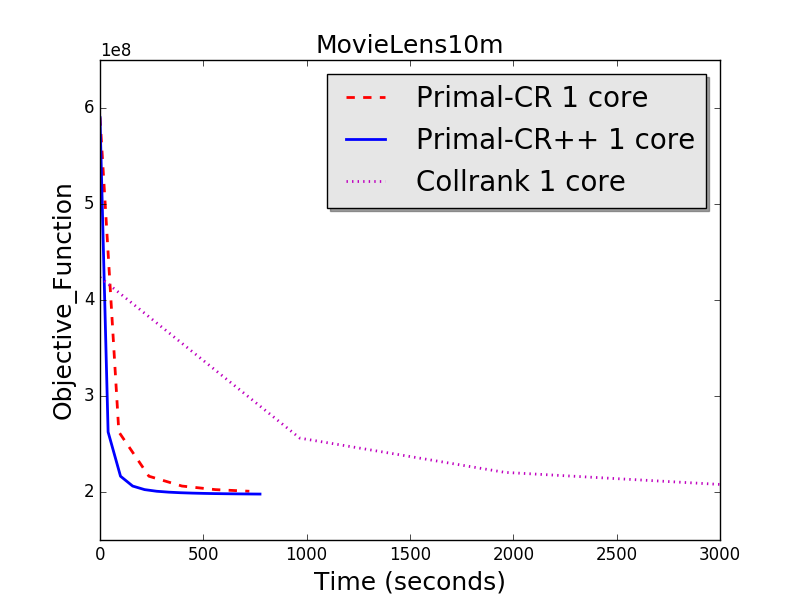} &
\hspace{-14pt}\includegraphics[width=0.33\linewidth]{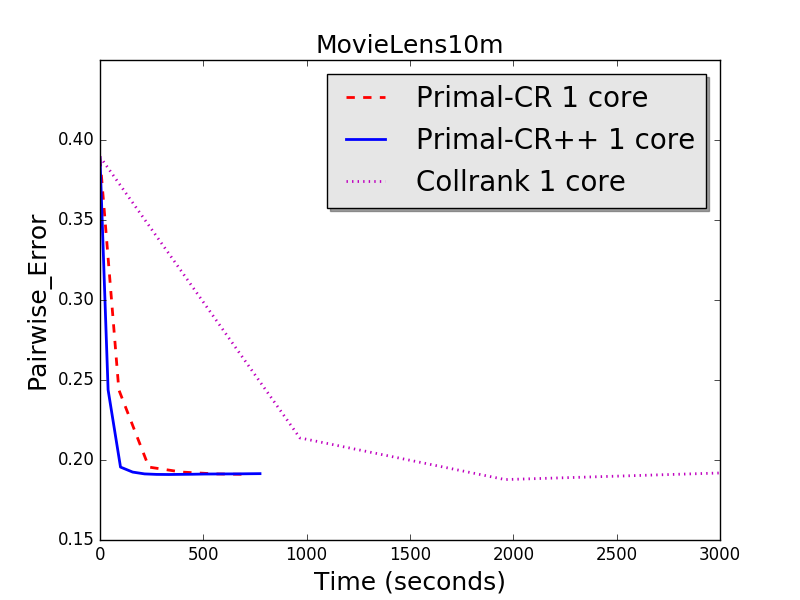} &
\hspace{-14pt}\includegraphics[width=0.33\linewidth]{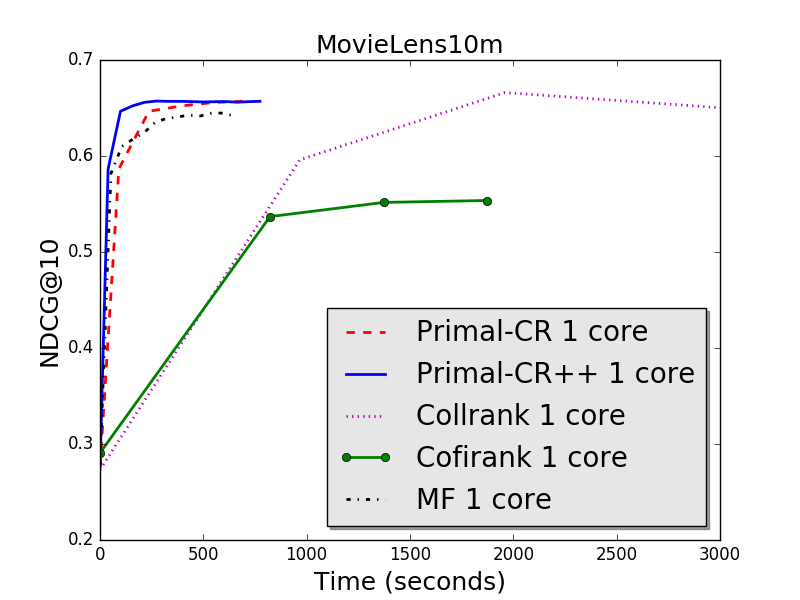}
\end{tabular}
\caption{Comparing Primal-CR, Primal-CR++ and Collrank, MovieLens10m data, 500 ratings/user, rank 100, lambda = 7000 \label{fig:serial_ml10m}}
\end{figure*}

\begin{figure*}
\begin{tabular}{ccc}
\hspace{-8pt}\includegraphics[width=0.33\linewidth]{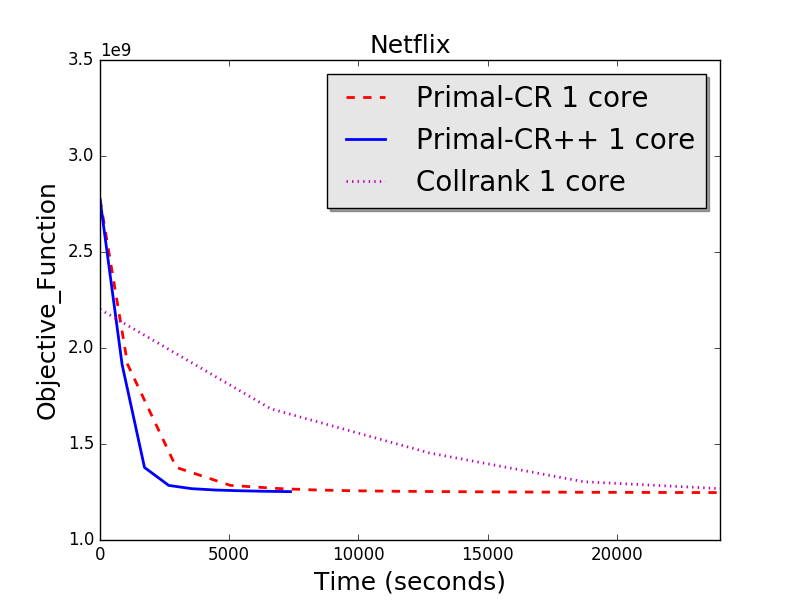} &
\hspace{-14pt}\includegraphics[width=0.33\linewidth]{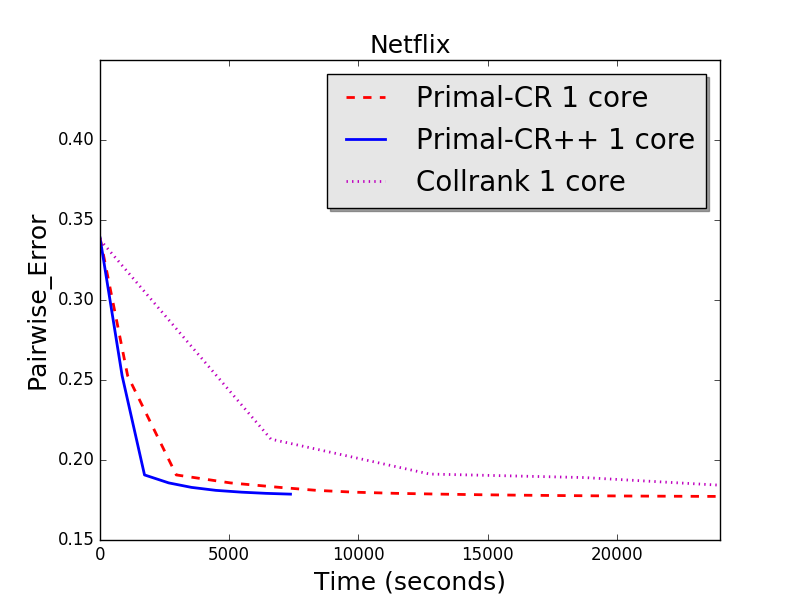} &
\hspace{-14pt}\includegraphics[width=0.33\linewidth]{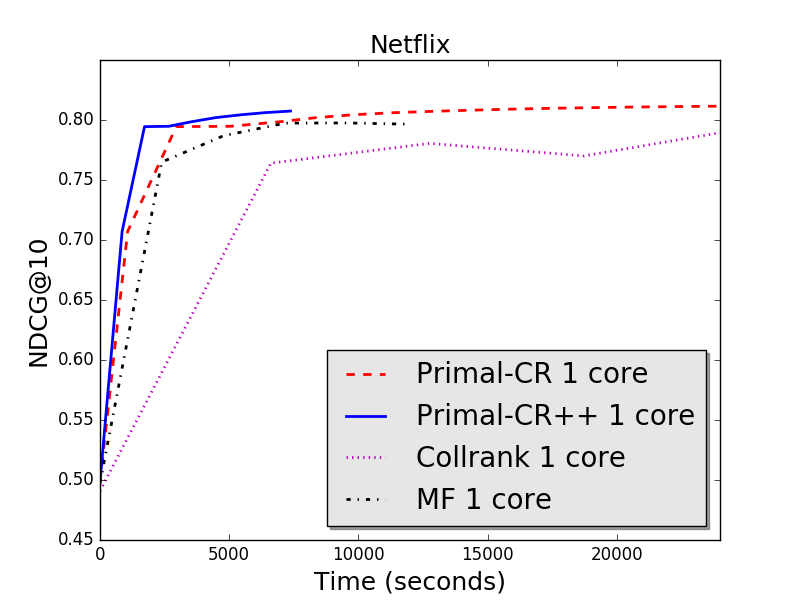}
\end{tabular}
\caption{Comparing Primal-CR, Primal-CR++ and Collrank, Netflix data, 200 ratings/user, rank 100, lambda = 10000 \label{fig:serial_netflix}}
\end{figure*}

A similar procedure can also be used for computing the Hessian-vector product, and the computation of updating $U$ with fixed $V$ is simplier since the problem becomes decomposable to $d_1$ independent problems, see eq~\eqref{eq:ranksvm}.  Due to the page
limit we omit the details here; interesting readers can check our code on github. 

Compared with the classical matrix factorization, where both ALS and SGD requires
$O(|\bar{\Omega}|r)$ time per iteration~\cite{koren2009matrix}, our algorithm 
has almost the same complexity, since $\log \bar{d}_2$ is usually smaller than $r$ (typically
$r=100$). 
Also, since all the temporary memory when computing user $i$ can be released immediately, 
the only memory cost is still the same with Primal-CR++, which is $O(d_1r + d_2 r)$.

\subsection{Parallelization}
\label{sec:parallel}
Updating $U$ while fixing $V$ can be parallelized easily because each column of $U$ is independent and we can actually solve $d_1$ independent subproblems at the same time. For the other side, updating $V$ while fixing $U$ can also be parallelized by parallelizing ``computing $g$" part and ``computing $Ha$" part respectively. We implemented the algorithm using parallel computing techniques in Julia by computing $g$ and $Ha$ distributedly and summing them up in the end. We show in section~\ref{sec:cparallel} that our parallel version of the proposed new algorithm works better than the paralleled version of Collrank algorithm \cite{park2015preference}.

\section{Experiments}

\begin{figure}\label{fig:parallel1}
\begin{tabular}{cc}
\hspace{-12pt}\includegraphics[width=0.5\linewidth]{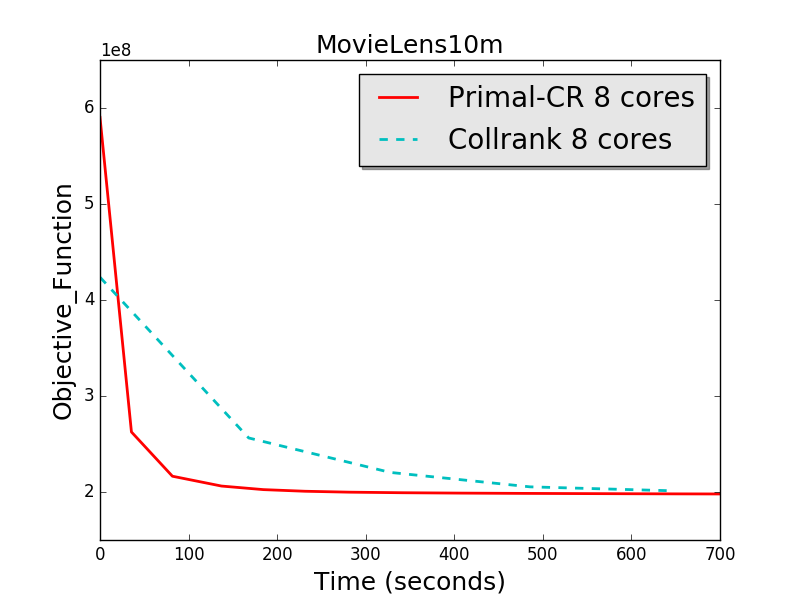} &
\hspace{-14pt}\includegraphics[width=0.5\linewidth]{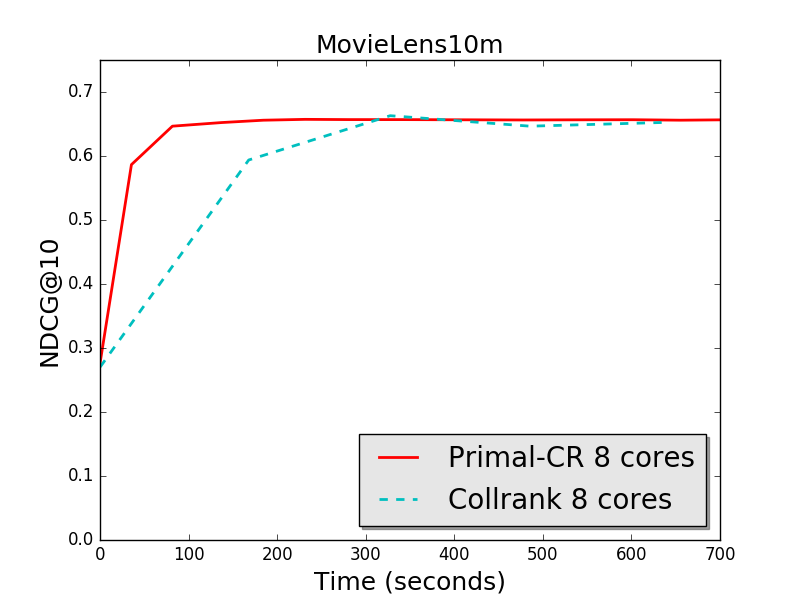}
\end{tabular}
\caption{Comparing parallel version of Primal-CR and Collrank, MovieLens10m data, 500 ratings/user, rank 100, lambda = 7000} 
\end{figure}


\begin{figure}\label{fig:parallel2}
\begin{tabular}{cc}
\hspace{-14pt}\includegraphics[width=0.5\linewidth]{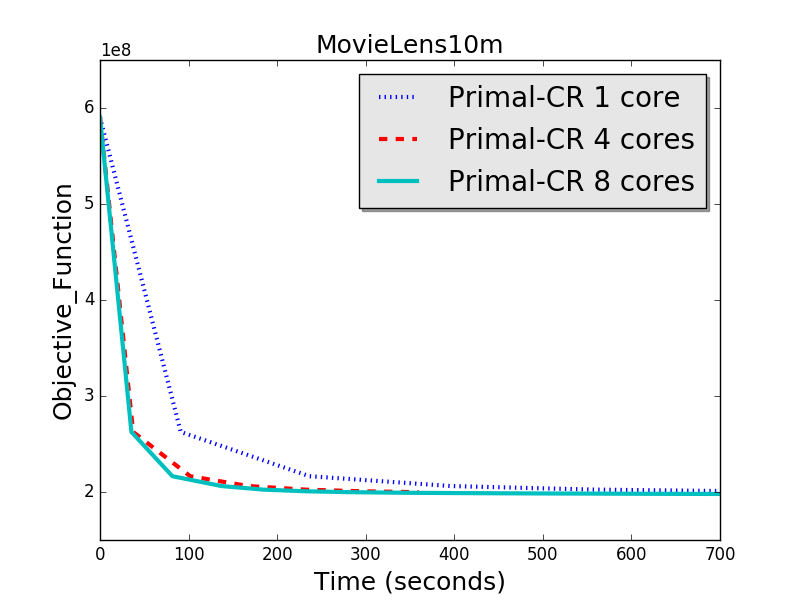} &
\hspace{-14pt}\includegraphics[width=0.5\linewidth]{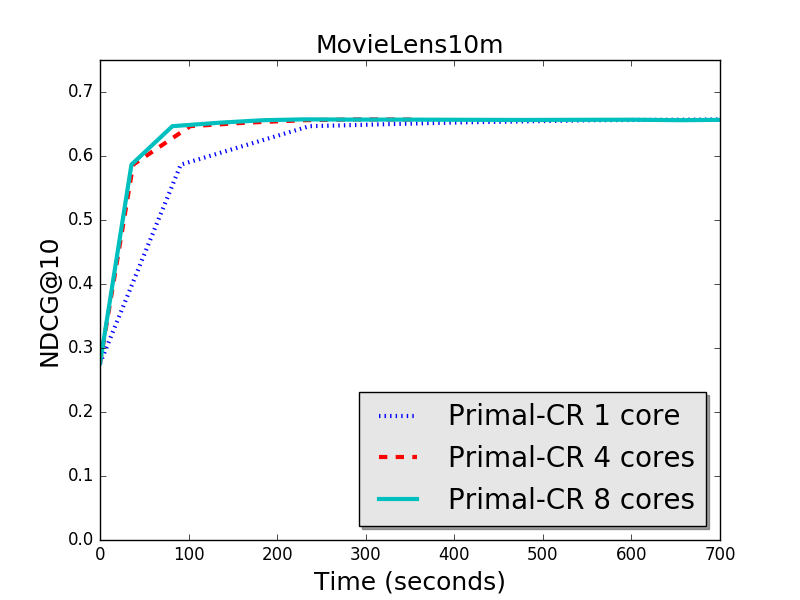}
\end{tabular}
\caption{Speedup of Primal-CR, MovieLens10m data, 500 ratings/user, rank 100, lambda = 7000} 
\end{figure}

In this section, we test the performance of our proposed algorithms Primal-CR and Primal-CR++ on real world datasets, and compare with existing methods.  All experiments are conducted on the UC Davis Illidan server with an Intel Xeon E5-2640 2.40GHz CPU and 64G RAM. We compare the following methods: 
\begin{itemize}
    \item Primal-CR and Primal-CR++: our proposed methods implemented in Julia. 
    \footnote{Our code is available on \url{https://github.com/wuliwei9278/ml-1m}.}
    \item Collrank: the collaborative ranking algorithm proposed in~\cite{park2015preference}. We use the C++ code released
    by the authors, and they parallelized their algorithm using OpenMP. 
    \item Cofirank: the classical collaborative ranking algorithm proposed in~\cite{weimer2007maximum}. We use the C++ code released
    by the authors.
    \item MF: the classical matrix factorization model in \eqref{eq:MF} solved by SGD~\cite{koren2009matrix}. 
\end{itemize}
We used three data sets (MovieLens1m, Movielens10m, Netflix data) to compare these algorithms. The dataset statistics
are summarized in Table~\ref{tab:data}. The regularization parameter $\lambda$ used for each datasets are chosen by a random sampled validation set. 
For the pair-wise based algorithms, we covert the ratings into pair-wise comparisons, 
by saying that item $j$ is preferred over item $k$ by user $i$ if user $i$ gives a higher rating to item $j$ over item $k$, and there will be no pair between two items if they have the same rating. 

We compare the algorithms in the following three different ways: 
\begin{itemize}
\item Objective function: since Collrank, Primal-CR, Primal-CR++ have the same objective function, we can compare the convergence speed in terms
of the objective function~\eqref{eq:obj} with squared hinge loss. 
\item Predicted pairwise error: the proportion of pairwise preference comparisons that we predicted correctly out of all the pairwise comparisons in the testing data:
\begin{equation} 
\text{pairwise error} = \frac{1}{|\T|} \sum_{\substack{(i, j, k) \in \T \\ Y_{ijk} = 1}}  \mathbb{1}(X_{ij} > X_{ik}), 
\end{equation}
where $\T$ represents the test data set and $|\T|$ denotes the size of test data set. 
\item NDCG$@k$: a standard performance measure of ranking, defined as:
\begin{equation} 
\text{NDCG}@k = \frac{1}{d_1} \sum_{i = 1}^{d_1} \frac{\text{DCG}@k(i, \pi_i)}{\text{DCG}@k(i, \pi_i^*)}, 
\end{equation}
where $i$ represents $i$-th user and
\begin{equation} 
\text{DCG}@k(i, \pi_i)= \sum_{l = 1}^{k} \frac{2^{M_i\pi_i(l)} - 1}{log_2(l + 1)}. 
\end{equation}
In the DCG definition, $\pi_i(l)$ represents the index of the $l$-th ranked item for user $i$ in test data based on the score matrix $X = U^T V$ generated, $M$ is the rating matrix and $M_{ij}$ is the rating given to item $j$ by user $i$. $\pi_i^*$ is the ordering provided by the underlying ground truth of the rating.
\end{itemize}

\subsection{Compare single thread versions using the same subsamples}

Since Collrank cannot scale to the full dataset of Movielens10m and Netflix, we sub-sample data using the 
same approach in their paper~\cite{park2015preference} and compare all the methods using the smaller training sets. 
More specifically, for each data set, we subsampled $N$ ratings for training data and used the rest of ratings as test data. For this subsampled data, we discard users with less than $N+10$ ratings, 
since we need at least 10 ratings for test data to compute the NDCG@10. 

As shown in Figure~\ref{fig:serial_ml1m}, \ref{fig:serial_ml10m}, \ref{fig:serial_netflix}, both Primal-CR and Primal-CR++ perform considerably better than the existing Collrank algorithm. As data size increases, the performance gap becomes larger.  As one can see, for Netflix data where $N = 200$, the speedup is more than 10 times compared to Collrank.

For Cofirank, we observe that it is even slower than Collrank, which confirms the experiments conducted in~\cite{park2015preference}. Furthermore, Cofirank cannot scale to larger datasets, so we omit the results in Figure~\ref{fig:serial_ml10m}
and \ref{fig:serial_netflix}. 

We also include the classical matrix factorization algorithm in the NDCG comparisons. 
As shown in our complexity analysis, our proposed algorithms are competitive with MF 
in terms of speed, and MF is much faster than other collaborative ranking algorithms. Also, we observe that MF converges to a slightly worse solution in MovieLens10m
and Netflix datasets, and converges to a much worse solution in MovieLens1m. The reason
is that MF minimizes a simple mean square error, while our algorithms are minimizing
ranking loss. 
Based on the experimental results, our algorithm Primal-CR++ should be able 
to replace MF in many real world recommender systems. 

\begin{table}
\begin{tabular}{| l | l | l | l |}
    \hline
     & MovieLens1m &  Movielens10m & Netflix \\ \hline
    Users & 6,040  &  71,567 & 2,649,430 \\ \hline
    Items & 3,952 & 65,134 & 17,771 \\ \hline
    Ratings & 1,000,209 & 10,000,054  & 99,072,112  \\ \hline
     $\lambda$ & 5,000 & 7,000  & 10,000 \\
    
    \hline
    \end{tabular}
    \caption{Datasets used for experiments}\label{tab:data}

    \end{table}

\begin{table}
\begin{tabular}{| l | l | l | l |}
    \hline
     \# cores & 1 &  4 & 8 \\ \hline
    Speedup for Primal-CR & 1x & 2.46x  &  3.12x  \\ \hline
    Speedup for Collrank & 1x & 2.95x & 3.47x \\ \hline
    \end{tabular}
    \caption{Scability of Primal-CR and Collrank on Movielens10m} \label{tab:speedup}
\end{table}

\begin{figure*}
\begin{tabular}{cc}
\hspace{-12pt}\includegraphics[width=0.5\linewidth]{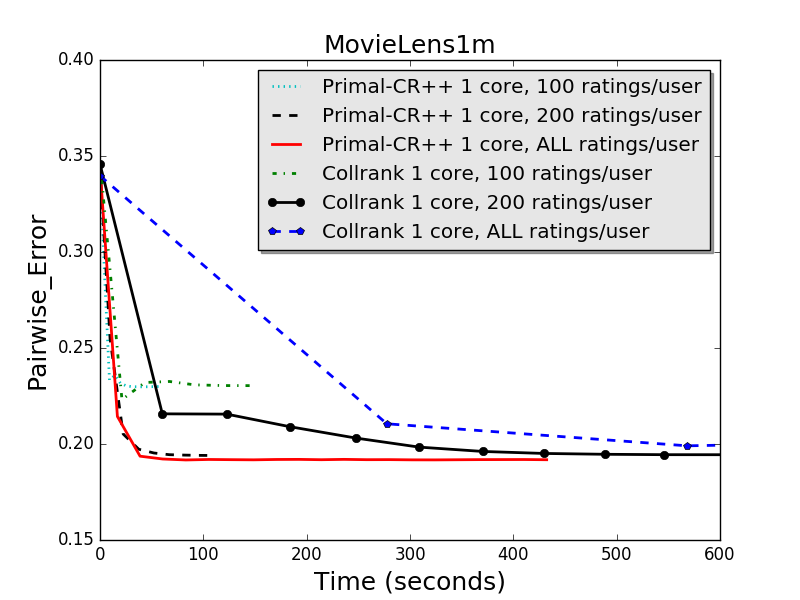} &
\hspace{-14pt}\includegraphics[width=0.5\linewidth]{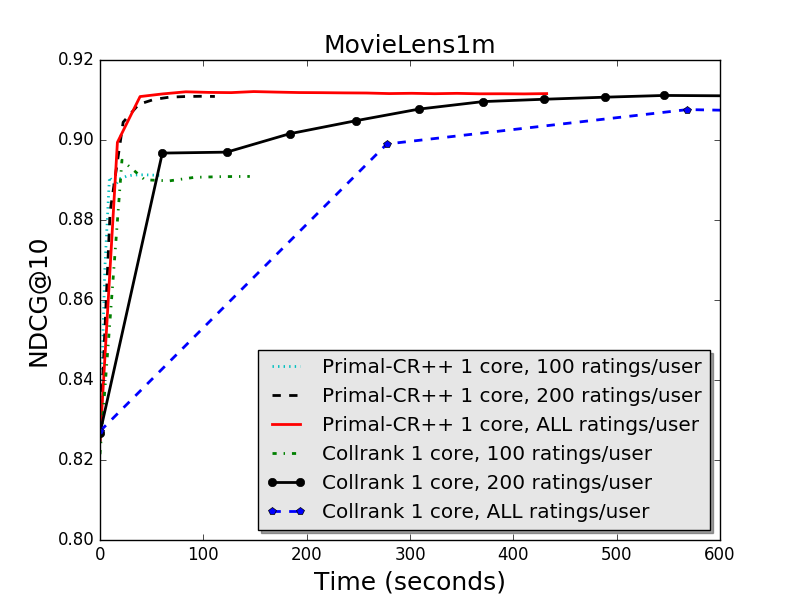} 
\end{tabular}
\caption{Varies number of ratings per user in training data, MovieLens1m data, rank 100} \label{fig:vary1m}
\end{figure*}


\begin{figure*}
\begin{tabular}{cc}
\hspace{-14pt}\includegraphics[width=0.5\linewidth]{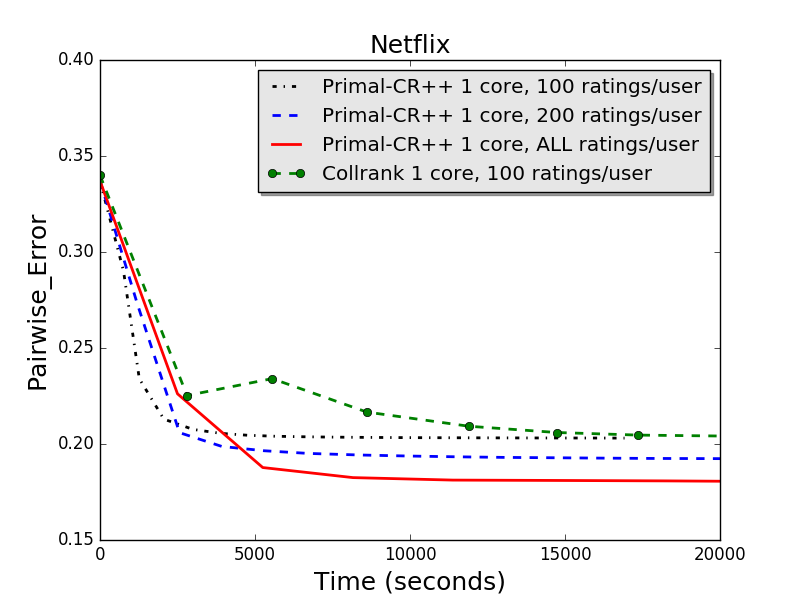} &
\hspace{-14pt}\includegraphics[width=0.5\linewidth]{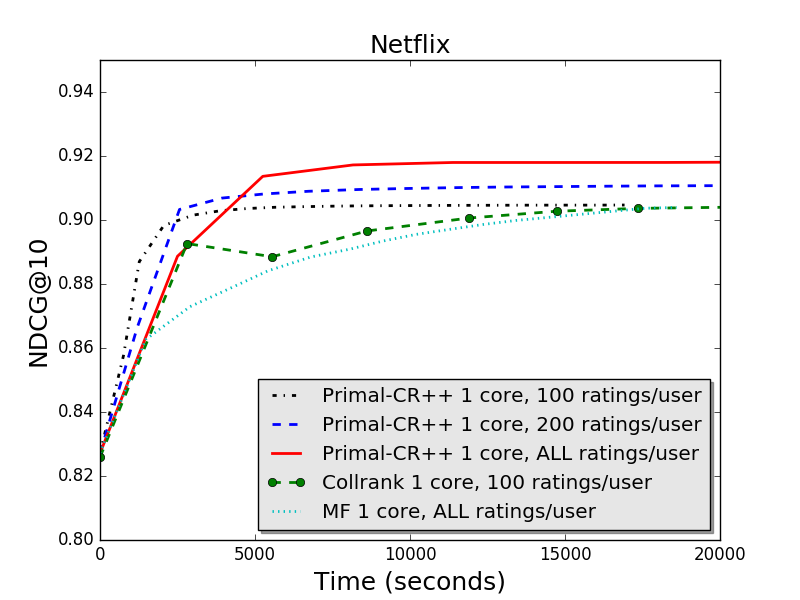} 
\end{tabular}
\caption{Varies number of ratings per user in training data, Netflix data, rank 100} \label{fig:varynetflix}
\end{figure*}

\subsection{Compare parallel versions} \label{sec:cparallel}
Since Collrank can be implemented in a parallel fashion, we also implemented the parallel version of our algorithm in Julia. We want to show our algorithm scales up well and is still much faster than Collrank in the multi-core shared memory setting. As shown in Figure~\ref{fig:parallel1}, Primal-CR is still much faster than  Collrank when 8 cores are used. Comparing our Primal-CR algorithm in 1 core, 4 cores and 8 cores on the same machine in Figure~\ref{fig:parallel2}, the speedup is desirable. The speedup of Primal-CR and Collrank is summarized in the Table~\ref{tab:speedup}. One can see from the table that the speedup of our Primal-CR algorithm is comparable to Collrank.

\subsection{Performance using Full Training Data}

Due to the $O(|\Omega| k)$ complexity, existing algorithms cannot deal with large number of pairs, so they 
always sub-sample a limited number of pairs per user when solving MovieLens10m or Netflix data. 
%
For example, for Collrank, the authors fixed number of ratings per user in training as $N$ and only reported $N$ up to 100 for Netflix data. When we tried to apply their code for $N = 200$, the algorithm gets very slow and reports memory error for $N = 500$. 


Using our algorithm, we have the ability to solve the full Netflix problem, 
so a natural question to ask is: 
Does using more training data help us predict and recommend better? The answer is yes! 
We conduct the following experiments to verify this: 
For all the users with more than 20 ratings, we randomly choose 10 ratings as test data and out of the rest ratings we randomly choose up to $C$ ratings per user as training data. One can see in Figure~\ref{fig:vary1m}, for the same test data, more training data leads to better prediction performance in terms of pairwise error and NDCG. Using all available ratings ($C = d_2$) gives lowest pairwise error and highest NDCG@10, using up to 200 ratings per user ($C = 200$) gives second lowest pairwise error and second highest NDCG@10, and using up to 100 ratings per user ($C = 100$) has the highest pairwise error and lowest NDCG@10. Similar phenomenon is observed for Netflix data in Figure~\ref{fig:varynetflix}. Collrank code does not work for $C = 200$ and $C = d_2$ and even for $C = 100$, it takes more than $20,000$ secs to converge while our Primal-CR++ takes less than $5,000$ secs for the full Netflix data. The speedup of our algorithm will be even more for a larger $C$ or larger data size $d_1$ and $d_2$. We tried to create input file without subsampling for Collrank, we created 344GB input data file and Collrank reported memory error message "Segmentation Fault". We also tried $C = 200$, still got the same error message. It is possible to implement Collrank algorithm
by directly working on the rating data, but the time complexity remains the same, so it is clear that our proposed Primal-CR and Primal-CR++ algorithms are much faster. 

To the best of our knowledge, our algorithm is the first ranking-based algorithm that can scale to full Netflix data set using a single core, and without sub-sampling. Our proposed algorithm makes the Collaborative Ranking Model in \eqref{eq:obj} a clear better choice for large-scale recommendation system over standard Matrix Factorization techniques, since we have the same scalability but achieve much better accuracy. Also, our experiments suggest that in practice, when we are given a set of training data, we should try to use all the training data instead of doing sub-sampling as existing algorithms do, and only Primal-CR and Primal-CR++ can scale up to all the ratings. 

\section{Conclusions}
We considered the collaborative ranking problem setting in which a low-rank matrix is fitted to the data in the form of pairwise comparisons or numerical ratings. We proposed our new optimization algorithms Primal-CR and Primal-CR++ where the time complexity is much better than all the existing approaches. We showed that our algorithms are much faster than state-of-the-art collaborative ranking algorithms on real data sets (MovieLens1m, Movielens10m and Netflix) using same subsampling scheme, and moreover our algorithm is the only one that can scale to the full Movielens10m and Netflix data. We observed that 
our algorithm has the same efficiency with matrix factorization, while
achieving better NDCG since we minimize ranking loss. As a result, we expect our
algorithm to be able to replace matrix factorization in many real applications. 


\chapter{SQL-Rank: A Listwise Approach to Collaborative Ranking}

\section{Introduction}
\label{intro}

We study a novel approach to collaborative ranking---the personalized ranking of items for users based on their observed preferences---through the use of listwise losses, which are dependent only on the observed rankings of items by users.
We propose the SQL-Rank algorithm, which can handle ties and missingness, incorporate both explicit ratings and more implicit feedback, provides personalized rankings, and is based on the relative rankings of items.
To better understand the proposed contributions, let us begin with a brief history of the topic.

\subsection{A brief history of collaborative ranking}

Recommendation systems, found in many modern web applications, movie streaming services, and social media, rank new items for users and are judged based on user engagement (implicit feedback) and ratings (explicit feedback) of the recommended items.
A high-quality recommendation system must understand the popularity of an item and infer a user's specific preferences with limited data. 
Collaborative filtering, introduced in \cite{hill1995recommending}, refers to the use of an entire community's preferences to better predict the preferences of an individual (see \cite{schafer2007collaborative} for an overview).
In systems where users provide ratings of items, collaborative filtering can be approached as a point-wise prediction task, in which we attempt to predict the unobserved ratings \cite{pan2017transfer}.
Low rank methods, in which the rating distribution is parametrized by a low rank matrix (meaning that there are a few latent factors) provides a powerful framework for estimating ratings \cite{mnih2008probabilistic, koren2008factorization}.
There are several issues with this approach.
One issue is that the feedback may not be representative of the unobserved entries due to a sampling bias, an effect that is prevalent when the items are only `liked' or the feedback is implicit because it is inferred from user engagement.
Augmenting techniques like weighting were introduced to the matrix factorization objective to overcome this problem \cite{hsieh2015pu, hu2008collaborative}. Many other techniques are also introduced \cite{kabbur2013fism, wang2017irgan, wu2016collaborative}.
Another methodology worth noting is the CofiRank algorithm of \cite{weimer2008cofi} which minimizes a convex surrogate of the normalized discounted cumulative gain (NDCG).
The pointwise framework has other flaws, chief among them is that in recommendation systems we are not interested in predicting ratings or engagement, but rather we must rank the items. 

Ranking is an inherently relative exercise.
Because users have different standards for ratings, it is often desirable for ranking algorithms to rely only on relative rankings and not absolute ratings.
A ranking loss is one that only considers a user's relative preferences between items, and ignores the absolute value of the ratings entirely, thus deviating from the pointwise framework.
Ranking losses can be characterized as pairwise and listwise.
A pairwise method decomposes the objective into pairs of items $j,k$ for a user $i$, and effectively asks `did we successfully predict the comparison between $j$ and $k$ for user $i$?'.
The comparison is a binary response---user $i$ liked $j$ more than or less than $k$---with possible missing values in the event of ties or unobserved preferences.
Because the pairwise model has cast the problem in the classification framework, then tools like support vector machines were used to learn rankings; \cite{joachims2002optimizing} introduces rankSVM and efficient solvers can be found in \cite{chapelle2010efficient}.
Much of the existing literature focuses on learning a single ranking for all users, which we will call simple ranking \cite{freund2003efficient,agarwal2006ranking, pahikkala2009efficient}.
This work will focus on the personalized ranking setting, in which the ranking is dependent on the user.

Pairwise methods for personalized ranking have seen great advances in recent years, with the AltSVM algorithm of \cite{park2015preference}, Bayesian personalized ranking (BPR) of \cite{rendle2009bpr}, and the near linear-time algorithm of \cite{wu2017large}.
Nevertheless, pairwise algorithms implicitly assume that the item comparisons are independent, because the objective can be decomposed where each comparison has equal weight.
Listwise losses instead assign a loss, via a generative model, to the entire observed ranking, which can be thought of as a permutation of the $m$ items, instead of each comparison independently.
The listwise permutation model, introduced in \cite{cao2007learning}, can be thought of as a weighted urn model, where items correspond to balls in an urn and they are sequentially plucked from the urn with probability proportional to $\phi(X_{ij})$ where $X_{ij}$ is the latent score for user $i$ and item $j$ and $\phi$ is some non-negative function.
They proposed to learn rankings by optimizing a cross entropy between the probability of $k$ items being at the top of the ranking and the observed ranking, which they combine with a neural network, resulting in the ListNet algorithm. 
\cite{shi2010list} applies this idea to collaborative ranking, but uses only the top-1 probability because of the computational complexity of using top-k in this setting.  
This was extended in \cite{huang2015listwise} to incorporate neighborhood information.
\cite{xia2008listwise} instead proposes a maximum likelihood framework that uses the permutation probability directly, which enjoyed some empirical success.

Very little is understood about the theoretical performance of listwise methods.  
\cite{cao2007learning} demonstrates that the listwise loss has some basic desirable properties such as monotonicity, i.e.~increasing the score of an item will tend to make it more highly ranked.
\cite{lan2009generalization} studies the generalizability of several listwise losses, using the local Rademacher complexity, and found that the excess risk could be bounded by a \smash{$1/\sqrt n$} term (recall, $n$ is the number of users).
Two main issues with this work are that no dependence on the number of items is given---it seems these results do not hold when $m$ is increasing---and the scores are not personalized to specific users, meaning that they assume that each user is an independent and identically distributed observation.
A simple open problem is: can we consistently learn preferences from a single user's data if we are given item features and we assume a simple parametric model?  ($n = 1, m\rightarrow \infty$.)

\subsection{Contributions of this work}

We can summarize the shortcomings of the existing work: current listwise methods for collaborative ranking rely on the top-$1$ loss, algorithms involving the full permutation probability are computationally expensive, little is known about the theoretical performance of listwise methods, and few frameworks are flexible enough to handle explicit and implicit data with ties and missingness.
This chapter addresses each of these in turn by proposing and analyzing the SQL-rank algorithm.

\begin{compactitem}
    \item We propose the SQL-Rank method, which is motivated by the permutation probability, and has advantages over the previous listwise method using cross entropy loss. 
    \item We provide an $O(\text{iter} \cdot (|\Omega| r))$ linear algorithm based on stochastic gradient descent, where $\Omega$ is the set of observed ratings and $r$ is the rank.
    \item The methodology can incorporate both implicit and explicit feedback, and can gracefully handle ties and missing data.
    \item We provide a theoretical framework for analyzing listwise methods, and apply this to the simple ranking and personalized ranking settings, highlighting the dependence on the number of users and items.
\end{compactitem}

\section{Methodology}
\label{method}
\subsection{Permutation probability}
The permutation probability, \cite{cao2007learning}, is a generative model for the ranking parametrized by latent scores. 
First assume there exists a ranking function that assigns scores to all the items. Let's say we have $m$ items, then the scores assigned can be represented as a vector $s = (s_1, s_2, ..., s_m)$. Denote a particular permutation (or ordering) of the $m$ items as $\pi$, which is a random variable and takes values from the set of all possible permutations $S_m$ (the symmetric group on $m$ elements). $\pi_1$ denotes the index of highest ranked item and $\pi_m$ is the lowest ranked.
The probability of obtaining $\pi$ is defined to be
\begin{equation}
\label{eq:permprob}
    P_{s}(\pi) := \prod_{j=1}^m \frac{\phi(s_{\pi_j})}{\sum_{l=j}^{m} \phi(s_{\pi_l})},
\end{equation}
where $\phi(.)$ is an increasing and strictly positive function. 
An interpretation of this model is that each item is drawn without replacement with probability proportional to $\phi(s_i)$ for item $i$ in each step.
One can easily show that $P_{s}(\pi)$ is a valid probability distribution, i.e.~$\sum_{\pi \in S_m} P_{s}(\pi) = 1, P_{s}(\pi) > 0, \forall \pi$. Furthermore, this definition of permutation probability enjoys several favorable properties (see \cite{cao2007learning}). 
For any permutation $\pi$ if you swap two elements ranked at $i<j$ generating the permutation $\pi'$ ($\pi'_i = \pi_j$, $\pi'_j = \pi_i$, $\pi_k = \pi'_k, k\ne i,j$), if $s_{\pi_i} > s_{\pi_j}$ then $P_s(\pi) > P_s(\pi')$.
Also, if permutation $\pi$ satisfies $s_{\pi_i} > s_{\pi_{i+1}}$, $\forall i$, then we have $\pi = \arg\max_{\pi' \in S_m} P_{s}(\pi')$.
Both of these properties can be summarized: larger scores will tend to be ranked more highly than lower scores.
These properties are required for the negative log-likelihood to be considered sound for ranking \cite{xia2008listwise}.


In recommendation systems, the top ranked items can be more impactful for the performance.
In order to focus on the top $k$ ranked items, we can compute the partial-ranking marginal probability,
\begin{equation}
\label{eq:topk_prob}
    P^{(k,\bar{m})}_{s}(\pi) = \prod_{j=1}^{\min\{k, \bar{m}\}} \frac{\phi(s_{\pi_j})}{\sum_{l=j}^{\bar{m}} \phi(s_{\pi_l})}.
\end{equation}
It is a common occurrence that only a proportion of the $m$ items are ranked, and in that case we will allow $\bar{m} \le m$ to be the number of observed rankings (we assume that $\pi_1,\ldots,\pi_{\bar{m}}$ are the complete list of ranked items).
When $k=1$, the first summation vanishes and top-$1$ probability can be calculated straightforwardly, which is why $k=1$ is widely used in previous listwise approaches for collaborative ranking. 
Counter-intuitively, we demonstrate that using a larger $k$ tends to improve the ranking performance.

We see that computing the likelihood loss is linear in the number of ranked items, which is in contrast to the cross-entropy loss used in \cite{cao2007learning}, which takes exponential time in $k$.
The cross-entropy loss is also not sound, i.e.~it can rank worse scoring permutations more highly, but the 
negative log-likelihood is sound.
We will discuss how we can deal with ties in the following subsection, namely, when the ranking is derived from ratings and multiple items receive the same rating, then there is ambiguity as to the order of the tied items.
This is a common occurrence when the data is implicit, namely the output is whether the user engaged with the item or not, yet did not provide explicit feedback.
Because the output is binary, the cross-entropy loss (which is based on top-$k$ probability with $k$ very small) will perform very poorly because there will be many ties for the top ranked items.
To this end, we propose a collaborative ranking algorithm using the listwise likelihood that can accommodate ties and missingness, which we call Stochastic Queuing Listwise Ranking, or SQL-Rank.

\begin{figure}[ht]
\begin{center}
\centerline{\includegraphics[width=\columnwidth]{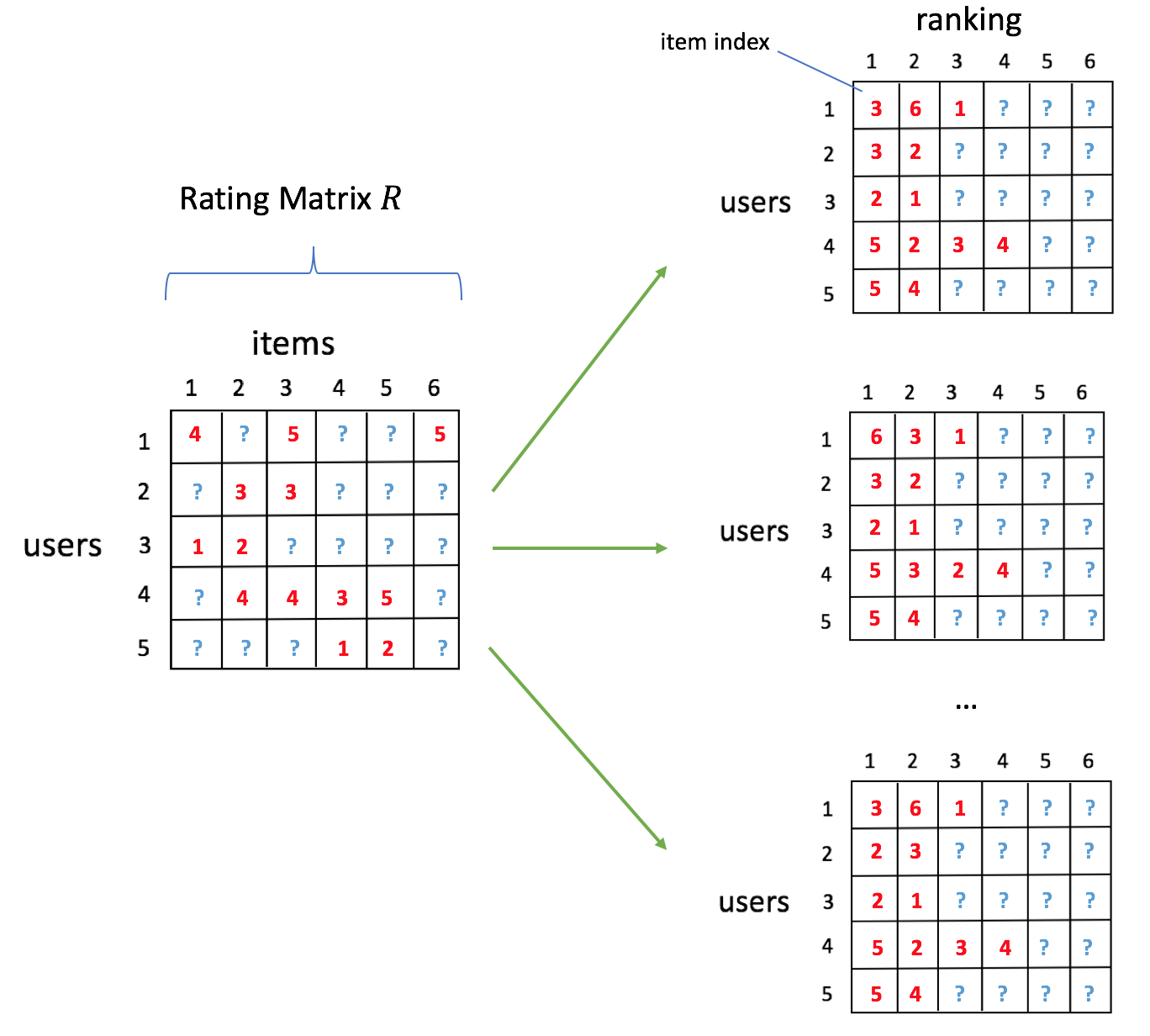}}
\caption{Demonstration of Stochastic Queuing Process---the rating matrix $R$ (left) generates multiple possible rankings $\Pi$'s (right), $\Pi \in \mathcal S(R,\Omega)$ by breaking ties randomly.}
\label{demo}
\end{center}
\end{figure}

\subsection{Deriving objective function for SQL-Rank}
The goal of collaborative ranking is to predict a personalized score $X_{ij}$ that reflects the preference level of user $i$ towards item $j$, where $1 \leq i \leq n$ and $1 \leq j \leq m$. It is reasonable to assume the matrix $X \in \R^{n \times m}$ to be low rank because there are only a small number of latent factors contributing to users' preferences. The input data is given in the form of ``user $i$ 
gives item $j$ a relevance score $R_{ij}$''. 
Note that for simplicity we assume all the users have the same number $\bar{m}$ of ratings, but this can be easily generalized
to the non-uniform case 
by replacing $\bar{m}$ with $m_i$ (number of ratings for user $i$).

With our scores $X$ and our ratings $R$, we can specify our collaborative ranking model using the permutation probability \eqref{eq:topk_prob}.
Let $\Pi_i$ be a ranking permutation of items for user $i$ (extracted from $R$), we can stack $\Pi_1, \dots \Pi_n$, row by row, to get the permutation matrix $\Pi \in \R^{n\times m}$. 
Assuming users are independent with each other, the probability of observing a particular $\Pi$ given the scoring matrix $X$
can be written as
\begin{equation}
    P_{X}^{(k, \bar{m})}(\Pi) = \prod_{i=1}^{n} P^{(k,\bar{m})}_{X_i}(\Pi_i).
\end{equation}
 We will assume that $\log \phi(x) = 1/(1 + \exp(-x))$ is the sigmoid function.
This has the advantage of bounding the resulting weights, \smash{$\phi(X_{ij})$}, and maintaining their positivity without adding additional constraints.

Typical rating data will contain many ties within each row.
In such cases, the permutation $\Pi$ is no longer unique and there is a set of permutations
that coincides with rating because with any candidate $\Pi$ we can arbitrarily shuffle the ordering
of items with the same relevance scores to generate a new candidate matrix $\Pi'$ which is still valid (see Figure~\ref{demo}). We denote the set of valid permutations as  $\mathcal S (R, \Omega)$, where $\Omega$ is the set of all pairs $(i,j)$ such that $R_{i,j}$ is observed.
We call this shuffling process the {\it Stochastic Queuing Process}, since one can imagine that by permuting ties we are stochastically queuing new $\Pi$'s for future use in the algorithm.


The probability of observing $R$ therefore should be defined as
    $P_{X}^{(k, \bar{m})}(R) =  \sum_{\Pi \in \mathcal S (R, \Omega)}P_X(\Pi)$. 
To learn the scoring matrix $X$, we can naturally solve the following maximum likelihood estimator with low-rank constraint: 
\begin{equation}
   \min_{X \in \mathcal X} - \log  \sum_{\Pi \in \mathcal S (R, \Omega)} P_X^{(k, \bar{m})}(\Pi), 
    \label{eq:aaa}
\end{equation}
 where $\mathcal{X}$ is the structural constraint of the scoring matrix. 
 To enforce low-rankness, we use the nuclear norm regularization $\mathcal{X}=\{X: \|X\|_*\leq r\}$. 

Eq~\eqref{eq:aaa} is hard to optimize since there is a summation inside the $\log$. 
But by Jensen's inequality and convexity of $-\log$ function, we can move the summation outside $\log$ and obtain an upper bound of the original negative log-likelihood, leading to the following optimization problem: 
\begin{equation}
\min_{X \in \mathcal X} -\sum_{\Pi \in \mathcal S (R, \Omega)} \log P_X^{(k, \bar{m})}(\Pi)
    \label{eq:convex}
\end{equation}
This upper bound is much easier to optimize and can be solved using Stochastic Gradient Descent (SGD). 

Next we discuss how to apply our model for explicit and implicit feedback settings. In the explicit feedback setting, it is assumed that the matrix $R$ is partially observed
and the observed entries are explicit ratings in a range (e.g., $1$ to $5$). 
We will show in the experiments that $k=\bar{m}$ (using the full list) leads to the best results. 
\cite{huang2015listwise} also observed that increasing $k$ is useful for their cross-entropy loss, but they were not able to increase $k$ since their model
has time complexity exponential to $k$. 

In the implicit feedback setting each element of $R_{ij}$ is either $1$ or $0$, where $1$ means positive actions (e.g., click or like) and $0$ means no action is observed.
Directly solving~\eqref{eq:convex} will be expensive since $\bar{m}=m$ and the computation
will involve all the $mn$ elements at each iteration. Moreover, the $0$'s in the matrix 
could mean either a lower relevance score or missing, thus should contribute less to the objective function. 
Therefore, we adopt the idea of negative sampling \cite{mikolov2013distributed} in our list-wise formulation. 
For each user (row of $R$), assume there are $\tilde{m}$ $1$'s, we then 
sample $\rho \tilde{m}$ unobserved entries uniformly from the same row and append to the back of the list. This then becomes the problem with $\bar{m}=(1+\rho) \tilde{m}$ and then we use the same algorithm
in explicit feedback setting to conduct updates. 
We then repeat the sampling process at the end of each iteration, so the update will be based on different set of $0$'s at each time. 

\begin{algorithm}[tb]
   \caption{SQL-Rank: General Framework}
   \label{alg:sqlrank}
\begin{algorithmic}
   \State {\bfseries Input:} $\Omega$, $\{R_{ij}: (i, j) \in \Omega \}$, $\lambda\in \R^+$, $ss$, $rate$, $\rho$
   \State {\bfseries Output:} $U\in \R^{r \times n}$ and $V \in \R^{r \times m}$
   \State Randomly initialize $U, V$ from Gaussian Distribution
   \Repeat
   \State Generate a new permutation matrix $\Pi$ \Comment{see alg~\ref{alg:sq}} 
   \State Apply gradient update to U while fixing V 
   \State Apply gradient update to V while fixing U \Comment{see alg~\ref{alg:update}}
   
   \Until{performance for validation set is good}
   \State \textbf{return} $U, V$\Comment{recover score matrix $X$}
\end{algorithmic}
\end{algorithm}

\begin{algorithm}[tb]
   \caption{Stochastic Queuing Process}
   \label{alg:sq}
\begin{algorithmic}
   \State {\bfseries Input:} $\Omega$, $\{R_{ij}: (i, j) \in \Omega \}$, $\rho$  
   \State {\bfseries Output:} $\Pi \in \R^{n \times m}$
   \For{$i = 1$ {\bfseries to} $n$}
    \State Sort items based on observed relevance levels $R_i$
    \State Form $\Pi_i$ based on indices of items in the sorted list
    \State Shuffle $\Pi_i$ for items within the same relevance level
    \If{Dataset is implicit feedback}
        \State Uniformly sample $\rho \tilde{m}$ items from unobserved items 
        \State Append sampled indices to the back of $\Pi_i$
    \EndIf
   \EndFor 
   \State Stack $\Pi_i$ as rows to form matrix $\Pi$
   \State {\bfseries Return} $\Pi$ \Comment{Used later to compute gradient}
\end{algorithmic}
\end{algorithm}

\subsection{Non-convex implementation}
Despite the advantage of the objective function in equation~\eqref{eq:convex} being convex, it is still not feasible for large-scale problems since the scoring matrix $X\in \R^{n\times m}$ leads to high computational and memory cost. We follow a common trick to transform~\eqref{eq:convex} to the non-convex form by replacing $X=U^TV$:
with  $U \in \R^{r \times n}, V\in \R^{r \times m}$ so that the objective is,
\begin{align*}
    \! \sum_{\Pi \in \mathcal S (R, \Omega)}\!\! \underbrace{-\sum_{i=1}^{n} \sum_{j=1}^{\bar{m}}   \log  \frac{\phi(u_i^T v_{\Pi_{ij}})}{\sum_{l=j}^{\bar{m}} \phi(u_i^T v_{\Pi_{il}})}}_{f(U, V)} 
   \! +\! \frac{\lambda}{2} (\|U\|_F^2 \!+\! \|V\|_F^2), 
   \label{eq:nonconvex}
\end{align*}
where $u_i, v_j$ are columns of $U, V$ respectively. We apply stochastic gradient descent to solve this problem. At each step, we choose a permutation matrix $\Pi\in\mathcal S (R, \Omega) $ 
using the stochastic queuing process (Algorithm \ref{alg:sq}) and then update $U, V$ by $\nabla f(U, V)$. For example, the gradient with respect to $V$ is ($g = \log \phi$ is the sigmoid function),
\begin{align*}
\frac{\partial f}{\partial v_j} &= \sum_{i\in \Omega_j} \sum_{t=1}^{\bar{m}}\bigg\{ -g'(u_i^T v_t) u_i \\
& + \frac{\mathbbm{1}(\text{rank}_i(j)\geq t) \phi(u_i^T v_j)}{\sum_{l=t}^{\bar{m}}\phi(u_i^T v_{\Pi_{il}})} g'(u_i^T v_j)u_i\bigg\}
\end{align*}
where $\Omega_j$ denotes the set of users that have rated the item $j$ and $\text{rank}_i(j)$ is a function gives the rank of the item $j$ for that user $i$. Because $g$ is the sigmoid function, $g' = g \cdot (1 - g)$. 
The gradient with respect to $U$ can be derived similarly. 

As one can see, a naive way to compute the gradient of $f$ requires  $O(n \bar{m}^2 r)$ time, which is very slow even for one iteration. However, 
we show in Algorithm~\ref{alg:gradv} (in the appendix) that there is a smart way to re-arranging the computation so that $\nabla_V f(U, V)$ can be computed in $O(n \bar{m} r)$
time, which makes our SQL-Rank a linear-time algorithm (with the same per-iteration complexity as classical matrix factorization). 

\section{Experiments}
\label{exp}
In this section, we compare our proposed algorithm (SQL-Rank) with other state-of-the-art algorithms on real world datasets. 
Note that our algorithm works for both implicit feedback and explicit feedback settings. In the implicit feedback setting, all the ratings are $0$ or $1$; 
in the explicit feedback setting, explicit ratings (e.g., $1$ to $5$) are given but only to a subset of user-item pairs. 
Since many real world recommendation systems follow the implicit feedback setting (e.g., purchases, clicks, or checkins), we will first compare SQL-Rank on implicit feedback datasets and show it outperforms state-of-the-art algorithms. Then we will verify that our algorithm also performs well on explicit feedback problems. 
All experiments are conducted on a server with an Intel Xeon E5-2640 2.40GHz CPU and 64G RAM.

\subsection{Implicit Feedback}
In the implicit feedback setting we compare the following methods:
\begin{compactitem}
    \item SQL-Rank: our proposed algorithm implemented in Julia \footnote{\url{https://github.com/wuliwei9278/SQL-Rank}}. 
    \item Weighted-MF: the weighted matrix factorization algorithm by putting different weights on $0$ and $1$'s \cite{hu2008collaborative,hsieh2015pu}.
    \item BPR: the Bayesian personalized ranking method motivated by MLE \cite{rendle2009bpr}. For both Weighted-MF and BPR, we use the C++ code by Quora \footnote{\url{https://github.com/quora/qmf}}. 
\end{compactitem}
Note that other collaborative ranking methods such as Pirmal-CR++~\cite{wu2017large} and List-MF~\cite{shi2010list}
do not work for implicit feedback data, and we will compare with them later in the explicit feedback experiments. 
For the performance metric, we use precision@$k$ for $k = 1, 5, 10$ defined by 
\begin{equation}
    \text{precision}@k = \frac{\sum_{i=1}^{n}|\{1 \leq l \leq k: R_{i\Pi_{il}} = 1\}|}{n \cdot k},
\end{equation} where $R$ is the rating matrix and $\Pi_{il}$ gives the index of the $l$-th ranked item for user $i$ among all the items not rated by user $i$ in the training set. 

We use rank $r = 100$ and tune regularization parameters for all three algorithms using a random sampled validation set. For Weighted-MF, we also tune the confidence weights on unobserved data. For BPR and SQL-Rank, we fix the ratio of subsampled unobserved $0$'s versus observed $1$'s to be $3:1$, which gives the best performance for both BPR and SQL-rank in practice. 

We experiment on the following four datasets. 
Note that the original data of Movielens1m, Amazon and Yahoo-music are
ratings from $1$ to $5$, so we follow the procedure in \cite{rendle2009bpr,yu2017selection} to preprocess the data. 
We transform ratings of $4, 5$ into $1$'s and the rest entries (with rating $1, 2, 3$ and unknown) as $0$'s.
Also, we remove users with very few $1$'s in the corresponding row to make sure there are enough $1$'s for both training and testing. 
For Amazon, Yahoo-music and Foursquare, we discard users with less than $20$ ratings and randomly select $10$ $1$'s as training and use the rest as testing. 
Movielens1m has more ratings than others, so we keep users with more than $60$ ratings, and randomly sample 50 of them as training. 
\begin{compactitem}
\item Movielens1m: a popular movie recommendation data with $6,040$ users and $3,952$ items.
\item Amazon: the Amazon purchase rating data for musical instruments \footnote{\url{http://jmcauley.ucsd.edu/data/amazon/}} with $339,232$ users and $83,047$ items. 
\item Yahoo-music: the Yahoo music rating data set \footnote{\url{https://webscope.sandbox.yahoo.com/catalog.php?datatype=r&did=3}} which contains $15,400$ users and $1,000$ items. 
\item Foursquare: 
a location check-in data\footnote{\url{https://sites.google.com/site/yangdingqi/home/foursquare-dataset}}. The data set contains $3,112$ users and $3,298$ venues with $27,149$ check-ins. The data set is already in the form of ``0/1'' so we do not need to do any transformation. 
\end{compactitem}

The experimental results are shown in Table~\ref{implicit-combined}. We find that SQL-Rank outperforms both Weighted-MF and BPR in most cases.

\begin{table}[t]
\caption{Comparing implicit feedback methods on various datasets.}
\label{implicit-combined}
\begin{center}
\begin{small}
\begin{sc}
\resizebox{0.8\textwidth}{!}{
\begin{tabular}{lccccr}
\toprule
Dataset & Method &  P@1 & P@5 & P@10 \\
\midrule
\multirow{3}{*}{Movielens1m} & SQL-Rank   &  {\bfseries 0.73685}  & {\bfseries 0.67167}   & 0.61833 \\
& Weighted-MF & 0.54686 & 0.49423 &  0.46123  \\
& BPR   & 0.69951 & 0.65608 &  {\bfseries 0.62494}  \\
\midrule
\multirow{3}{*}{Amazon} & SQL-Rank   &  0.04255  & {\bfseries 0.02978}   & {\bfseries 0.02158} \\
& Weighted-MF & 0.03647 & 0.02492 &  0.01914  \\
& BPR   & {\bfseries 0.04863} & 0.01762 &   0.01306  \\
\midrule
\multirow{3}{*}{Yahoo music} & SQL-Rank   &  {\bfseries 0.45512}  & {\bfseries 0.36137}   & {\bfseries 0.30689} \\
& Weighted-MF &  0.39075 & 0.31024 & 0.27008  \\
& BPR   & 0.37624 & 0.32184 &   0.28105  \\
\midrule
\multirow{3}{*}{Foursquare} & SQL-Rank   &  {\bfseries 0.05825}  & {\bfseries 0.01941}   & {\bfseries 0.01699} \\
& Weighted-MF & 0.02184 & 0.01553 &  0.01407  \\
& BPR   & 0.03398 & 0.01796 &   0.01359  \\
\bottomrule
\end{tabular}
}
\end{sc}
\end{small}
\end{center}
\end{table}

\subsection{Explicit Feedback}
Next we compare the following methods in the explicit feedback setting: 
\begin{compactitem}
    \item SQL-Rank: our proposed algorithm implemented in Julia. Note that in the explicit feedback setting our algorithm
    only considers pairs with explicit ratings. 
    \item List-MF: the listwise algorithm using the cross entropy loss between observed rating and top $1$ probability \cite{shi2010list}. We use the C++ implementation on github\footnote{\url{https://github.com/gpoesia/listrankmf}}.
    \item MF: the classical matrix factorization algorithm in \cite{koren2008factorization} utilizing a pointwise loss solved by SGD. We implemented SGD in Julia. 
    \item Primal-CR++: the recently proposed pairwise algorithm in \cite{wu2017large}. We use the Julia implementation released by the authors\footnote{\url{https://github.com/wuliwei9278/ml-1m}}.
\end{compactitem}

Experiments are conducted on Movielens1m and Yahoo-music datasets. We perform the same procedure as in implicit feedback setting except that we do not need to mask the ratings into ``0/1''. 

We measure the performance in the following two ways:
\begin{compactitem}
    \item NDCG$@k$: defined as:
        \begin{equation*} 
            \text{NDCG}@k = \frac{1}{n} \sum_{i = 1}^{n} \frac{\text{DCG}@k(i, \Pi_i)}{\text{DCG}@k(i, \Pi_i^*)}, 
        \end{equation*} where $i$ represents $i$-th user and
        \begin{equation*} 
        \text{DCG}@k(i, \Pi_i)= \sum_{l = 1}^{k} \frac{2^{R_{i\Pi_{il}}} - 1}{log_2(l + 1)}. 
        \end{equation*}
        In the DCG definition, $\Pi_{il}$ represents the index of the $l$-th ranked item for user $i$ in test data based on the learned score matrix $X$. $R$ is the rating matrix and $R_{ij}$ is the rating given to item $j$ by user $i$. $\Pi_i^*$ is the ordering provided by the ground truth rating.
    \item Precision@$k$: defined as a fraction of relevant items among the top $k$ recommended items: 
        \begin{equation*}
            \text{precision}@k = \frac{\sum_{i=1}^{n}|\{1 \leq l \leq k: 4 \leq R_{i\Pi_{il}} \leq 5\}|}{n \cdot k},
        \end{equation*}
        here we consider items with ratings assigned as $4$ or $5$ as relevant. $R_{ij}$ follows the same definitions above but unlike before $\Pi_{il}$ gives the index of the $l$-th ranked item for user $i$ among all the items that are not rated by user $i$ in the training set (including both rated test items and unobserved items). 
\end{compactitem}

As shown in Table~\ref{explicit-combined}, our proposed listwise algorithm SQL-Rank outperforms previous listwise method List-MF in both NDCG@$10$ and precision@$1,5,10$. It verifies the claim that log-likelihood loss outperforms the cross entropy loss if we use it correctly. When listwise algorithm SQL-Rank is compared with pairwise algorithm Primal-CR++, the performances between SQL-Rank and Primal-CR++ are quite similar, slightly lower for NDCG@$10$ but higher for precision@$1,5,10$. Pointwise method MF is doing okay in NDCG but really bad in terms of precision. 
Despite having comparable NDCG, the predicted top $k$ items given by MF are quite different from those given by other algorithms utilizing a ranking loss. The ordered lists based on SQL-Rank, Primal-CR++ and List-MF, on the other hand, share a lot of similarity and only have minor difference in ranking of some items. It is an interesting phenomenon that we think is worth exploring further in the future. 

\begin{table}[t]
\caption{Comparing explicit feedback methods on various datasets.}
\label{explicit-combined}
\begin{center}
\begin{small}
\begin{sc}
\resizebox{0.8\textwidth}{!}{
\begin{tabular}{lccccr}
\toprule
Dataset & Method & NDCG@$10$ & P@1 & P@5 & P@10 \\
\midrule
\multirow{3}{*}{Movielens1m} & SQL-Rank &  0.75076   &  {\bfseries 0.50736}  & {\bfseries 0.43692}   & {\bfseries 0.40248} \\
& List-MF & 0.73307 & 0.45226 &  0.40482   & 0.38958 \\
& Primal-CR++    & {\bfseries 0.76826} & 0.49365 & 0.43098   & 0.39779 \\
& MF    & 0.74661 & 0.00050 & 0.00096   & 0.00134 \\
\midrule
\multirow{3}{*}{Yahoo music} & SQL-Rank &  0.66150   &  {\bfseries 0.14983}  & {\bfseries 0.12144}   & {\bfseries 0.10192} \\
& List-MF & 0.67490 & 0.12646 &  0.11301   & 0.09865 \\
& Primal-CR++    &  0.66420 & 0.14291 & 0.10787   & 0.09104 \\
& MF    & {\bfseries 0.69916} & 0.04944 & 0.03105   & 0.04787 \\

\bottomrule
\end{tabular}
}
\end{sc}
\end{small}
\end{center}
\end{table}


\subsection{Training speed}

To illustrate the training speed of our algorithm, we plot precision@$1$ versus training time for the Movielen1m dataset and the Foursquare dataset. Figure~\ref{time} and Figure~\ref{time2} (in the appendix) show that our algorithm SQL-Rank is faster than BPR and Weighted-MF. Note that our algorithm is implemented in Julia while BPR and Weighted-MF are highly-optimized C++ codes (usually at least 2 times faster than Julia) released by Quora.  
This speed difference makes sense as our algorithm takes $O(n \bar{m} r)$ time, which is linearly to the observed ratings. In comparison, pair-wise model such as BPR
has $O(n\bar{m}^2)$ pairs, so will take $O(n \bar{m}^2 r)$ time for each epoch. 

 \begin{figure}[ht]
\begin{center}
\centerline{\includegraphics[width=0.8\columnwidth]{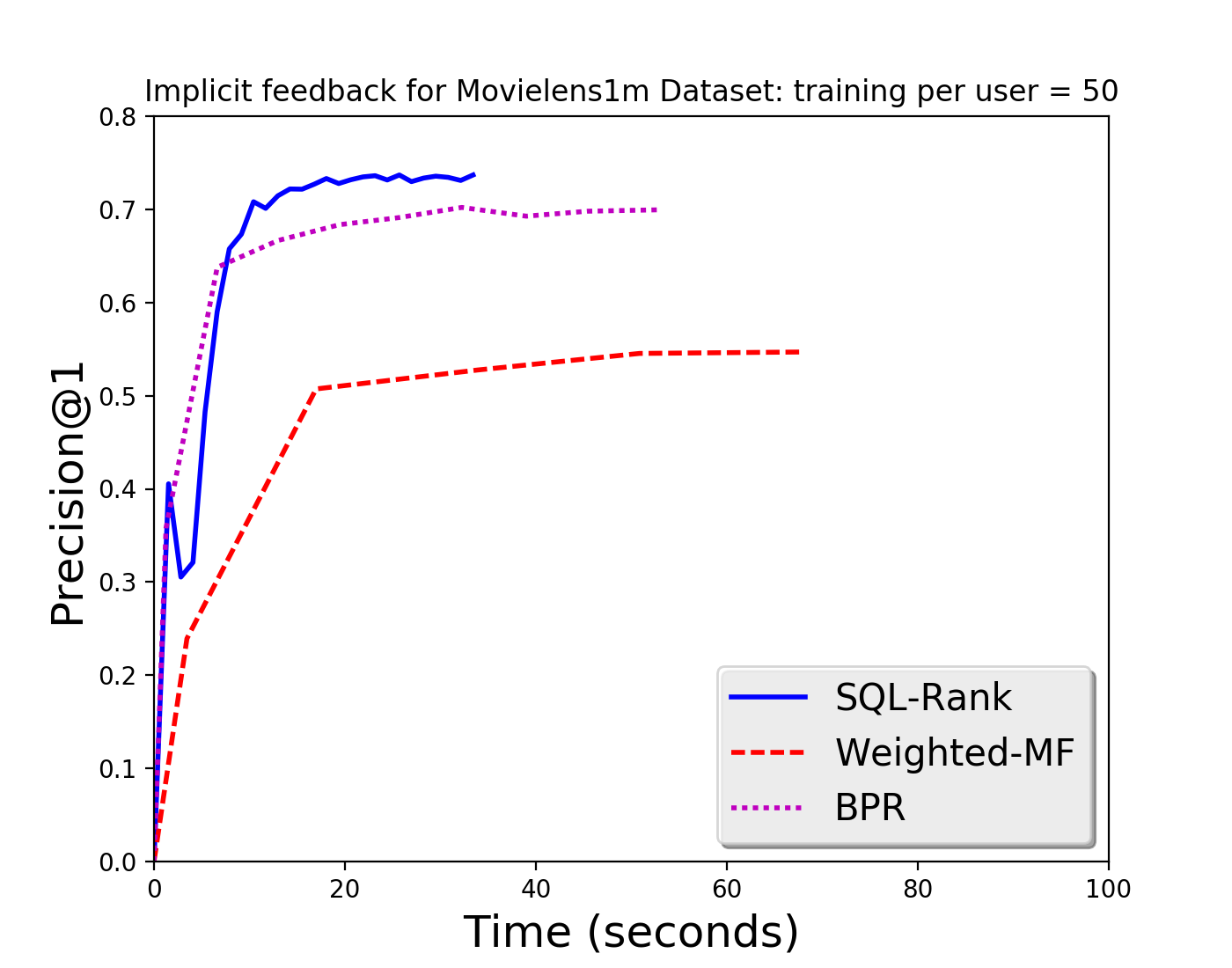}}
\caption{Training time of implicit feedback methods.}
\label{time}
\end{center}
\end{figure}

\subsection{Effectiveness of Stochastic Queuing (SQ)}

One important innovation in our SQL-Rank algorithm is the Stochastic Queuing (SQ) Process for handling ties. 
To illustrate the effectiveness of the SQ process, we compare our algorithm with and without SQ. Recall that without SQ means we fix a certain permutation matrix $\Pi$ and optimize with respect to it throughout all iterations without generating new $\Pi$, while SQ allows us to update using a new permutation at each time. As shown 
Table~\ref{implicit-sq} and Figure~\ref{sqp} (in the appendix),
the performance gain from SQ in terms of precision is substantial (more than $10\%$) on Movielen1m dataset. It verifies the claim that our way of handling ties and missing data is very effective and improves the ranking results by a lot. 

\begin{table}[t]
\caption{Effectiveness of Stochastic Queuing Process.}
\label{implicit-sq}
\begin{center}
\begin{small}
\begin{sc}
\resizebox{0.5\textwidth}{!}{
\begin{tabular}{lcccr}
\toprule
Method &  P@1 & P@5 & P@10 \\
\midrule
With SQ   &  {\bfseries 0.73685}  & {\bfseries 0.67167}   & {\bfseries 0.61833} \\
Without SQ & 0.62763 & 0.58420 &  0.55036  \\
\bottomrule
\end{tabular}
}
\end{sc}
\end{small}
\end{center}
\end{table}

\begin{table}[th]
\caption{Comparing different $k$ on Movielens1m data set using 50 training data per user.}
\label{explicit-ml1m-topk}
\begin{center}
\begin{small}
\begin{sc}
\resizebox{0.8\textwidth}{!}{
\begin{tabular}{lcccr}
\toprule
$k$ & NDCG@$10$ & P@1 & P@5 & P@10 \\
\midrule
5 &  0.64807   &  0.39156  & 0.33591   & 0.29855 \\
10 & 0.67746 & 0.43118 &  0.34220   & 0.33339 \\
25    & 0.74589 & 0.47003 & 0.42874   & 0.39796 \\
50 (full list) & {\bfseries 0.75076} & {\bfseries 0.50736} & {\bfseries 0.43692}  & {\bfseries 0.40248} \\
\bottomrule
\end{tabular}
}
\end{sc}
\end{small}
\end{center}
\end{table}

%


\subsection{Effectiveness of using the Full List}

Another benefit of our algorithm is that we are able to minimize top $k$ probability with much larger $k$ and without much overhead. Previous approaches 
\cite{huang2015listwise}
already pointed out increasing $k$ leads to better ranking results, but
their complexity is exponential to $k$ so they were not able to have $k > 1$. 
To show the effectiveness of using permutation probability for full lists rather than using the top $k$ probability for top $k$ partial lists in the likelihood loss, we fix everything else to be the same and only vary $k$ in Equation~\eqref{eq:convex}. We obtain the results in Table~\ref{explicit-ml1m-topk} and Figure~\ref{explicit_topk} (in the appendix). It shows that the larger $k$ we use, the better the results we can get. Therefore, in the final model, we set $k$ to be the maximum number (length of the  observed list.)   


\section{Conclusions}

In this chapter, we propose a listwise approach for collaborative ranking and provide
an efficient algorithm to solve it. 
Our methodology can incorporate both implicit and explicit feedback, and can gracefully handle ties and missing data. 
In experiments, we demonstrate our algorithm outperforms existing state-of-the art methods in terms of top $k$ recommendation precision. 
We also provide a theoretical framework for analyzing listwise methods highlighting the dependence on the number of users and items.

\chapter{Stochastic Shared Embeddings: Data-driven Regularization of Embedding Layers}

\section{Introduction}
Recently, embedding representations have been widely used in almost all AI-related fields, from feature maps \cite{krizhevsky2012imagenet} in computer vision, to word embeddings \cite{mikolov2013distributed, pennington2014glove}  in natural language processing, to user/item embeddings \cite{mnih2008probabilistic, hu2008collaborative} in recommender systems. Usually, the embeddings are high-dimensional vectors. Take language models for example, in GPT \cite{radford2018improving} and Bert-Base model \cite{devlin2018bert}, 768-dimensional vectors are used to represent words. Bert-Large model utilizes 1024-dimensional vectors and GPT-2 \cite{radford2019language} may have used even higher dimensions in their unreleased large models. In recommender systems, things are slightly different: the dimension of user/item embeddings are usually set to be reasonably small, 50 or 100. But the number of users and items is on a much bigger scale. Contrast this with the fact that the size of word vocabulary that normally ranges from 50,000 to 150,000, the number of users and items can be millions or even billions in large-scale real-world commercial recommender systems \cite{bennett2007netflix}. 

Given the massive number of parameters in modern neural networks with embedding layers, mitigating over-parameterization can play a big role in preventing over-fitting in deep learning. 
We propose a regularization method, Stochastic Shared Embeddings (SSE), that uses prior information about similarities between embeddings, such as semantically and grammatically related words in natural languages or real-world users who share social relationships. Critically, SSE progresses by stochastically transitioning between embeddings as opposed to a more brute-force regularization such as graph-based Laplacian regularization and ridge regularization.
Thus, SSE integrates seamlessly with existing stochastic optimization methods and the resulting regularization is data-driven.


We will begin the paper with the mathematical formulation of the problem, propose SSE, and provide the motivations behind SSE. 
We provide a theoretical analysis of SSE that can be compared with excess risk bounds based on empirical Rademacher complexity.
We then conducted experiments for a total of 6 tasks from simple neural networks with one hidden layer in recommender systems, to the transformer and BERT in natural languages and find that when used along with widely-used regularization methods such as weight decay and dropout, our proposed methods can further reduce over-fitting, which often leads to more favorable generalization results.  

\section{Related Work}
Regularization techniques are used  to control model complexity and avoid over-fitting. 
$\ell_2$ regularization~\cite{hoerl1970ridge} is the most widely used approach and has been used in many matrix factorization models in recommender systems; $\ell_1$ regularization~\cite{tibshirani1996regression} is used when a sparse model is preferred. For deep neural networks, it has been shown that $\ell_p$  regularizations are often too weak, while dropout~\cite{hinton2012improving,srivastava2014dropout} is more effective in practice.  There are many other regularization techniques, including parameter sharing \cite{goodfellow2016deep}, max-norm regularization \cite{srebro2005maximum}, gradient clipping \cite{pascanu2013difficulty}, etc.

Our proposed SSE-graph is very different from graph Laplacian regularization \cite{cai2011graph}, in which the distances of any two embeddings connected over the graph are directly penalized.
Hard parameter sharing uses one embedding to replace all distinct embeddings in the same group, which inevitably introduces a significant bias. 
Soft parameter sharing \cite{nowlan1992simplifying} is similar to the graph Laplacian, penalizing the $l_2$ distances between any two embeddings. 
These methods have no dependence on the loss, while the proposed SSE-graph method is data-driven in that the loss influences the effect of regularization.
Unlike graph Laplacian regularization, hard and soft parameter sharing, our method is stochastic by nature.
This allows our model to enjoy similar advantages as dropout \cite{srivastava2014dropout}.


Interestingly, in the original BERT model's pre-training stage \cite{devlin2018bert}, a variant of SSE-SE is already implicitly used for token embeddings but for a different reason. 
In \cite{devlin2018bert}, the authors masked 15\% of words and 10\% of the time replaced the [mask] token with a random token. 
In the next section, we discuss how SSE-SE differs from this heuristic.
Another closely related technique to ours is the label smoothing \cite{szegedy2016rethinking}, which is widely used in the computer vision community.
We find that in the classification setting if we apply SSE-SE to one-hot encodings associated with output $y_i$ only, our SSE-SE is closely related to the label smoothing, which can be treated as a special case of our proposed method. 
 


\begin{algorithm}[tb]

  \caption{SSE-Graph for Neural Networks with Embeddings}
  \label{alg:see-graph}

\begin{algorithmic}[1]
  \State {\bfseries Input:} input $x_i$, label $y_i$, backpropagate $T$ steps, mini-batch size $m$, knowledge graphs on embeddings $\{E_1, \dots, E_M \}$
  \State Define $p_l(., . | \Phi)$ based on knowledge graphs on embeddings, $l = 1,\ldots, M$
  \For{$t=1$ {\bfseries to} $T$}
    \State Sample one mini-batch $\{x_1, \dots, x_m\}$
     \For{$i=1$ {\bfseries to} $m$}
        \State Identify the set of embeddings $\mathcal{S}_i = \{E_1[j^i_1], \dots, E_M[j^i_M]\}$ for input $x_i$ and label $y_i$
        \For{each embedding $E_l[j^i_l] \in \mathcal{S}_i$}
            \State  Replace $E_l[j^i_l]$ with $E_l[k_l]$, where $k_l \sim p_l(j^i_l, . | \Phi)$
        \EndFor
     \EndFor
    \State Forward and backward pass with the new embeddings
  \EndFor
  \State Return embeddings $\{E_1, \dots, E_M \}$, and neural network parameters $\Theta$
 \end{algorithmic}

\end{algorithm}


\section{Stochastic Shared Embeddings}

Throughout this chapter, the network input $x_i$ and label $y_i$ will be encoded into indices $j_1^i,\ldots, j_M^i$ which are elements of $\mathcal I_1 \times \ldots \mathcal I_M$, the index sets of embedding tables.
A typical choice is that the indices are the encoding of a dictionary for words in natural language applications, or user and item tables in recommendation systems. 
Each index, $j_l$, within the $l$th table, is associated with an embedding $E_l[j_l]$ which is a trainable vector in $\mathbb R^{d_l}$.
The embeddings associated with label $y_i$ are usually non-trainable one-hot vectors corresponding to label look-up tables while embeddings associated with input $x_i$ are trainable embedding vectors for embedding look-up tables. 
In natural language applications, we appropriately modify this framework to accommodate sequences such as sentences.

The loss function can be written as the functions of embeddings:
\begin{equation}\label{eq:obj5}
    R_n(\Theta) = \sum_i \ell(x_i, y_i | \Theta) = \sum_i \ell(E_1[j^i_1], \dots, E_M[j_M^i] | \Theta),
\end{equation}
\vspace{-.1cm}
where $y_i$ is the label and $\Theta$ encompasses all trainable parameters including the embeddings, $\{ E_l[j_l]: j_l \in \mathcal I_l\}$.
The loss function $\ell$ is a mapping from embedding spaces to the reals.
For text input, each $E_l[j^i_l]$ is a word embedding vector in the input sentence or document. For recommender systems, usually there are two embedding look-up tables: one for users and one for items \cite{he2017neural}. So the objective function, such as mean squared loss or some ranking losses, will comprise both user and item embeddings for each input. 
We can more succinctly write the matrix of all embeddings for the $i$th sample as $\mathbf{E}[\mathbf{j}^i] = (E_1[j_1^i], \ldots, E_M[j_M^i])$ where $\mathbf{j}^i = (j^i_1,\ldots, j^i_M) \in \mathcal{I}$.
By an abuse of notation we write the loss as a function of the embedding matrix, $\ell(\mathbf{E}[\mathbf{j}^i] | \Theta)$.

\begin{figure}
  \centering
  \includegraphics[width=\linewidth]{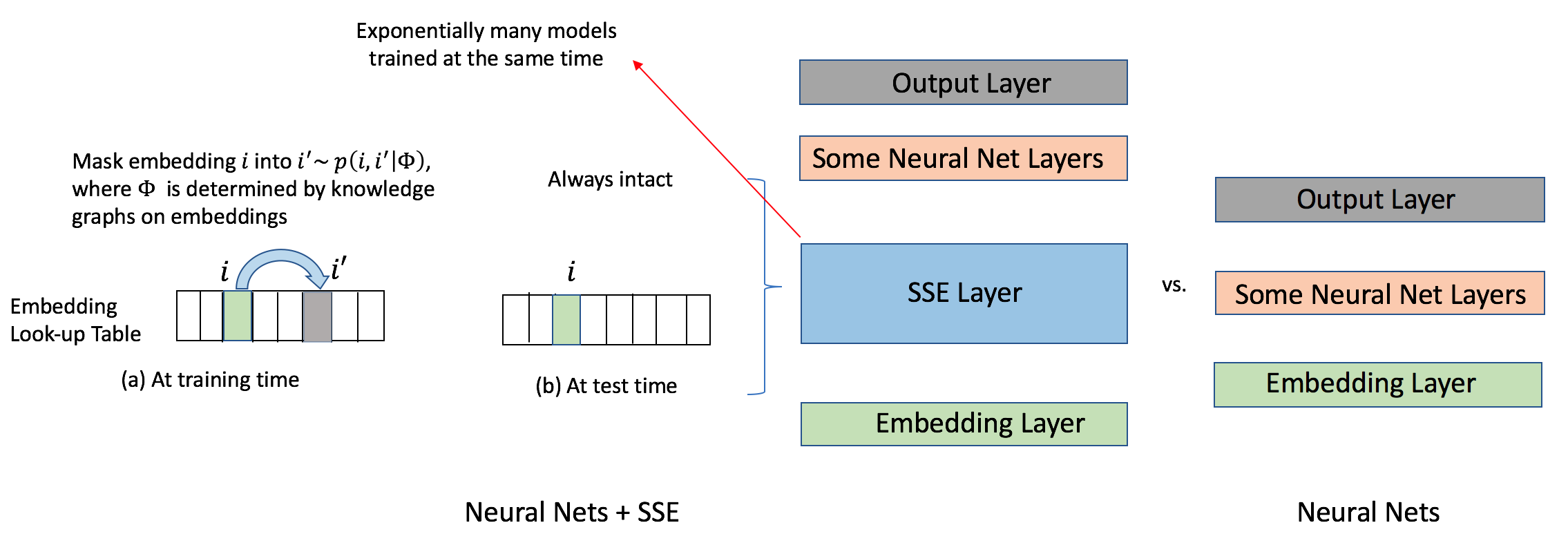}
  \caption{SSE-Graph described in Algorithm~\ref{alg:see-graph} and Figure~\ref{fig:sse-graph} can be viewed as adding exponentially many distinct reordering layers above the embedding layer. 
  A modified backpropagation procedure in Algorithm~\ref{alg:see-graph} is used to train exponentially many such neural networks at the same time.
  }
  \label{fig:train_test}
\end{figure}


Suppose that we have access to knowledge graphs \cite{miller1995wordnet, lehmann2015dbpedia} over embeddings, and we have a prior belief that two embeddings will share information and replacing one with the other should not incur a significant change in the loss distribution. For example, if two movies are both comedies and they are starred by the same actors, it is very likely that for the same user, replacing one comedy movie with the other comedy movie will result in little change in the loss distribution. 
In stochastic optimization, we can replace the loss gradient for one movie's embedding with the other similar movie's embedding, and this will not significantly bias the gradient if the prior belief is accurate.
On the other hand, if this exchange is stochastic, then it will act to smooth the gradient steps in the long run, thus regularizing the gradient updates.

\begin{figure}
  \centering
  \includegraphics[width=\linewidth]{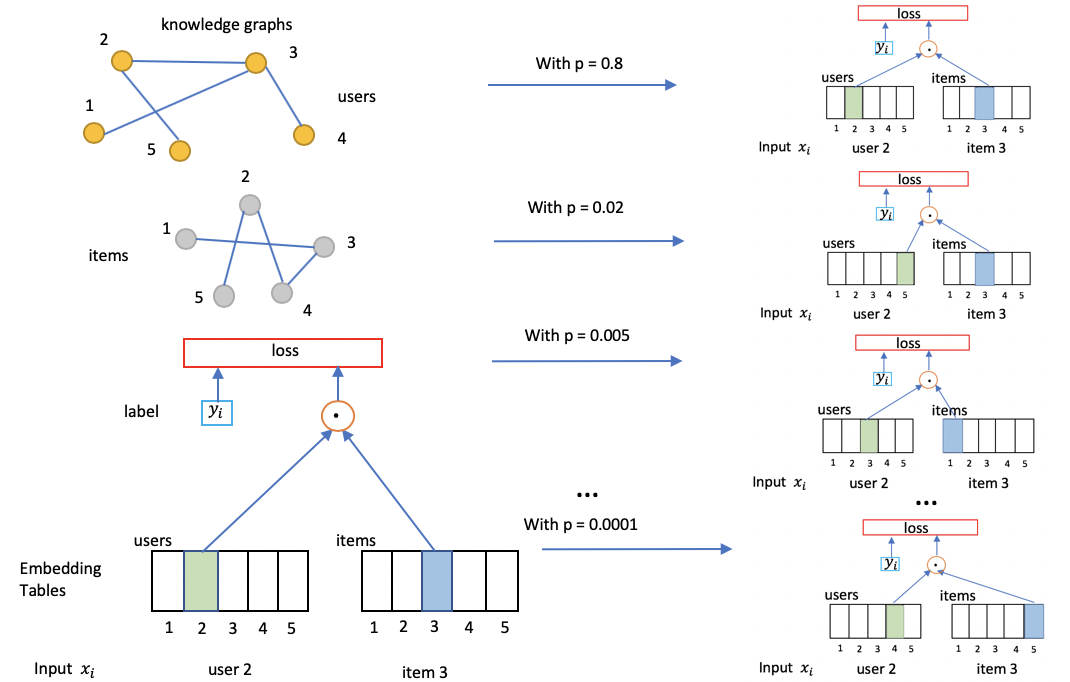}
  \caption{Illustration of how SSE-Graph algorithm in Figure~\ref{fig:train_test} works for a simple neural network.
  }
  \label{fig:sse-graph}
\end{figure}

\subsection{General SSE with Knowledge Graphs: SSE-Graph}


Instead of optimizing objective function $R_n(\Theta)$ in \eqref{eq:obj5}, SSE-Graph described in Algorithm~\ref{alg:see-graph}, Figure~\ref{fig:train_test}, and Figure~\ref{fig:sse-graph} is approximately optimizing the objective function below:
\begin{equation}\label{eq:single_obj}
    S_n(\Theta) = \sum_i \sum_{\textbf{k} \in \mathcal{I}} p(\textbf{j}^i, \textbf{k}|\Phi) \ell(\mathbf{E}[\mathbf{k}] | \Theta),
\end{equation}
where $p(\mathbf{j}, \mathbf{k} | \Phi)$ is the transition probability (with parameters $\Phi$) of exchanging the encoding vector $\mathbf{j} \in \mathcal I$ with a new encoding vector $\mathbf{k} \in \mathcal I$ in the Cartesian product index set of all embedding tables.
When there is a single embedding table ($M =1$) then there are no hard restrictions on the transition probabilities, $p(.,.)$, but when there are multiple tables ($M > 1$) then we will enforce that $p(.,.)$ takes a tensor product form (see \eqref{eq:multiple_obj}).
When we are assuming that there is only a single embedding table ($M=1$) we will not bold $j, E[j]$ and suppress their indices.

In the single embedding table case, $M=1$, there are many ways to define transition probability from $j$ to $k$. One simple and effective way is to use a random walk (with random restart and self-loops) on a knowledge graph $\mathcal{G}$, i.e.~when embedding $j$ is connected with $k$ but not with $l$, we can set the ratio of  $p(j, k|\Phi)$ and  $p(j, l|\Phi)$ to be a constant greater than $1$. In more formal notation, we have  \begin{equation}\label{eq:transition1}
   j \sim k, j \not\sim l \longrightarrow  p(j, k|\Phi)/p(j, l|\Phi) = \rho, 
\end{equation} where $\rho > 1$ and is a tuning parameter. It is motivated by the fact that embeddings connected with each other in knowledge graphs should bear more resemblance and thus be more likely replaced by each other. 
Also, we let 
    $p(j, j|\Phi) = 1 - p_0,$
where $p_0$ is called the {\it SSE probability} and embedding retainment probability is $1-p_0$. We treat both $p_0$ and $\rho$ as tuning hyper-parameters in experiments. With \eqref{eq:transition1} and $\sum_k p(j, k|\Phi) = 1$, we can derive transition probabilities between any two embeddings to fill out the transition probability table.

When there are multiple embedding tables, $M > 1$, then we will force that the transition from $\mathbf{j}$ to $\mathbf{k}$ can be thought of as independent transitions from $j_l$ to $k_l$ within embedding table $l$ (and index set $\mathcal{I}_l$).
Each table may have its own knowledge graph, resulting in its own transition probabilities $p_l(.,.)$.
The more general form of the SSE-graph objective is given below:

\begin{equation}\label{eq:multiple_obj}
    S_n(\Theta) =  \sum_i \sum_{k_1, \dots, k_M} p_1(j^i_1, k_1 |\Phi)\cdots p_M(j^i_M, k_M |\Phi) \ell(E_1[k_1], \dots, E_M[k_M] | \Theta),
\end{equation}
Intuitively, this SSE objective could reduce the variance of the estimator.

Optimizing \eqref{eq:multiple_obj} with SGD or its variants (Adagrad \cite{duchi2011adaptive}, Adam \cite{kingma2014adam}) is simple. We just need to randomly switch each original embedding tensor $\mathbf{E}[\mathbf{j}^i]$ with another embedding tensor $\mathbf{E}[\mathbf{k}]$ randomly sampled according to the transition probability (see Algorithm~\ref{alg:see-graph}). This is equivalent to have a randomized embedding look-up layer as shown in Figure~\ref{fig:train_test}. 

We can also accommodate sequences of embeddings, which commonly occur in natural language application, by considering
 $(j^i_{l,1}, k_{l,1}), \dots, (j^i_{l,n^i_l}, k_{l,n^i_l})$ instead of $(j^i_l, k_l)$ for $l$-th embedding table in \eqref{eq:multiple_obj}, where $1 \leq l \leq M$ and $n^i_l$ is the number of embeddings in table $l$ that are associated with $(x_i, y_i)$.
 When there is more than one embedding look-up table, we sometimes prefer to use different $p_0$ and $\rho$ for different look-up tables in \eqref{eq:transition1} and the SSE probability constraint. For example, in recommender systems, we would use $p_u, \rho_u$ for user embedding table and $p_i, \rho_i$ for item embedding table.

We find that SSE with knowledge graphs, i.e., SSE-Graph, can force similar embeddings to cluster when compared to the original neural network without SSE-Graph. In Figure~\ref{fig:pca}, one can easily see that more embeddings tend to cluster into 2 black holes after applying SSE-Graph when embeddings are projected into 3D spaces using PCA. Interestingly, a similar phenomenon occurs when assuming the knowledge graph is a complete graph, which we would introduce as SSE-SE below.

\begin{figure*}
\centering
\begin{tabular}{ccc}
\hspace{-8pt}\includegraphics[width=0.3\linewidth]{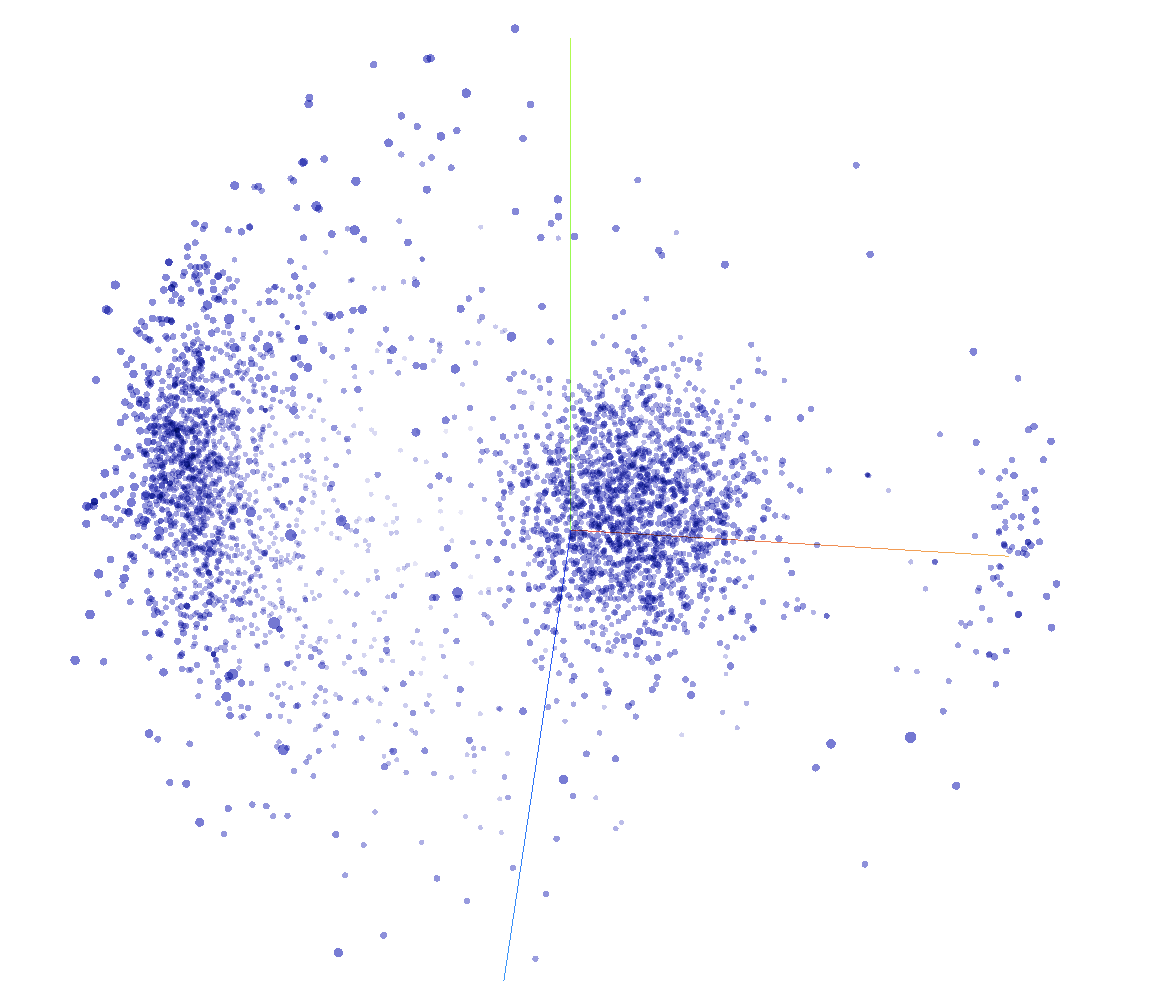} &
\hspace{-14pt}\includegraphics[width=0.3\linewidth]{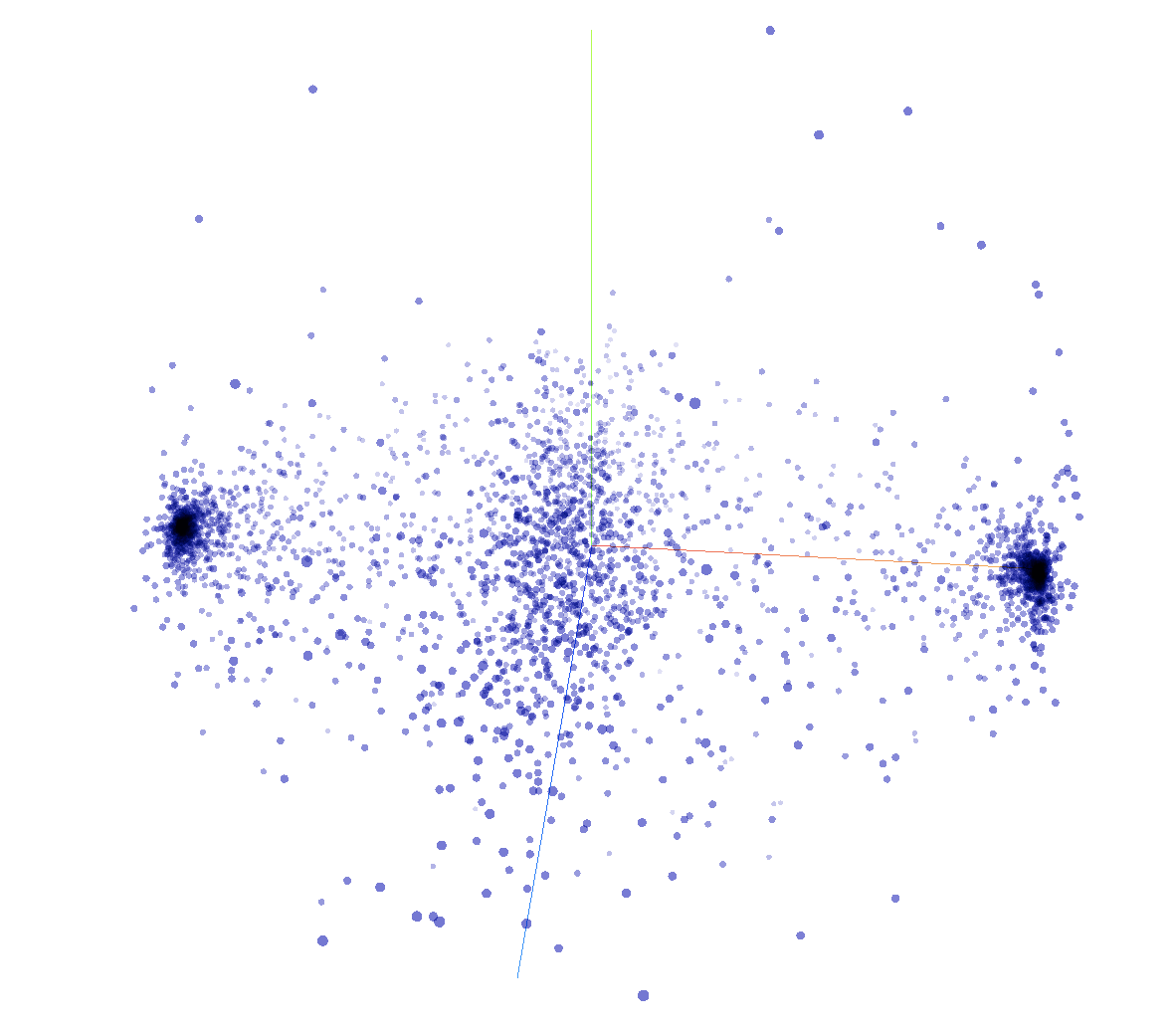} &
\hspace{-14pt}\includegraphics[width=0.3\linewidth]{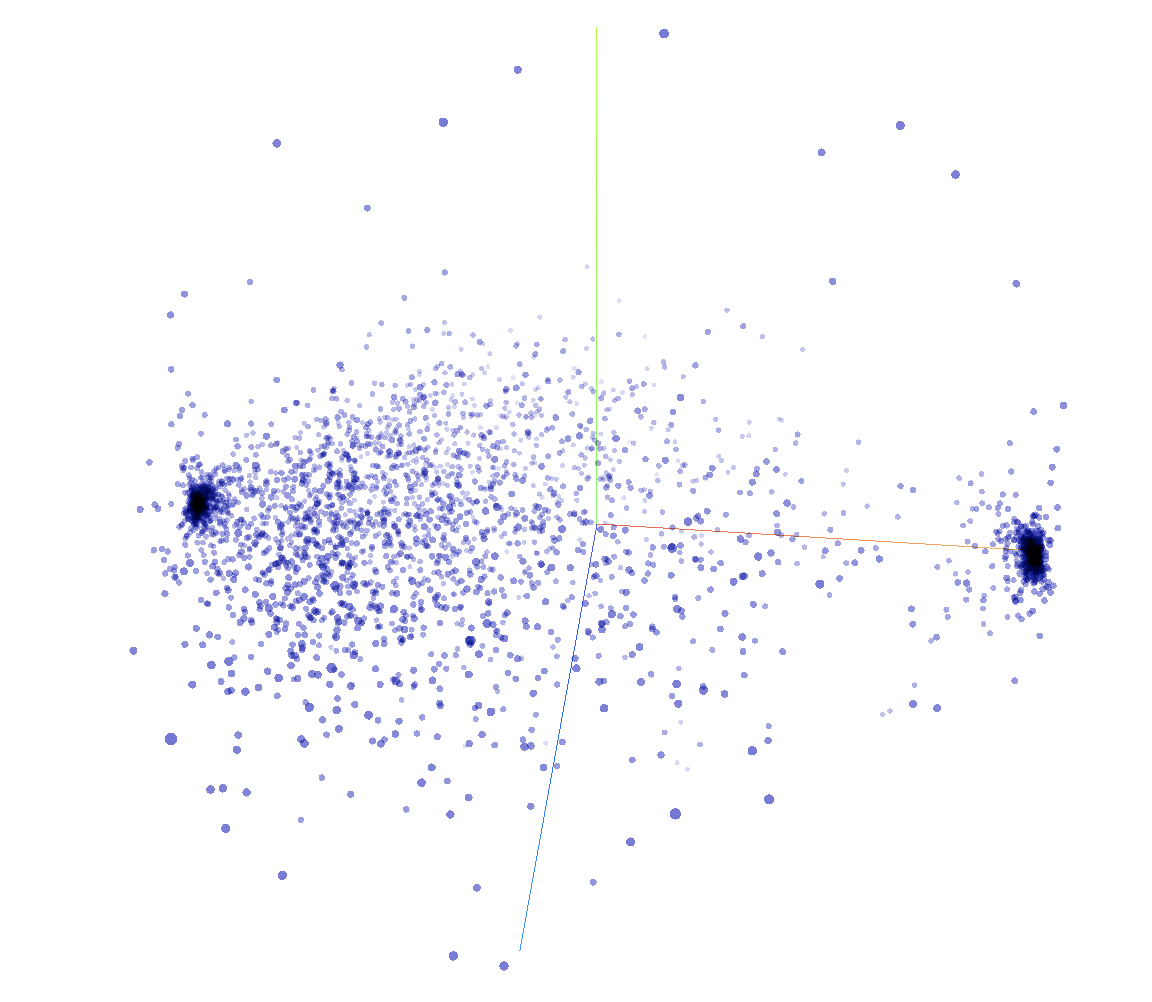}
\end{tabular}
\vspace{-10pt}\caption{Projecting 50-dimensional embeddings obtained by training a simple neural network without SSE (Left), and with SSE-Graph (Center) , SSE-SE (Right) into 3D space using PCA.}
\label{fig:pca}
\end{figure*}

\subsection{Simplified SSE with Complete Graph: SSE-SE}
One clear limitation of applying the SSE-Graph is that not every dataset comes with good-quality knowledge graphs on embeddings.
For those cases, we could assume there is a complete graph over all embeddings so there is a small transition probability between every pair of different embeddings:
\begin{equation}\label{eq:sse-se}
    p(j, k|\Phi) = \frac{p_0}{N - 1}, \quad \forall 1 \leq k \neq j \leq N,
\end{equation} where $N$ is the size of the embedding table.
The SGD procedure in Algorithm~\ref{alg:see-graph} can still be applied and we call this algorithm SSE-SE (Stochastic Shared Embeddings - Simple and Easy).
It is worth noting that SSE-Graph and SSE-SE are applied to embeddings associated with not only input $x_i$ but also those with output $y_i$. Unless there are considerably many more embeddings than data points and model is significantly overfitting, normally $p_0 = 0.01$ gives reasonably good results.

Interestingly, we found that the SSE-SE framework is related to several techniques used in practice. 
For example, BERT pre-training unintentionally applied a method similar to SSE-SE to input $x_i$ by replacing the masked word with a random word. 
This would implicitly introduce an SSE layer for input $x_i$ in Figure~\ref{fig:train_test}, because now embeddings associated with input $x_i$ be stochastically mapped according to \eqref{eq:sse-se}. 
The main difference between this and SSE-SE is that it merely augments the input once, while SSE introduces randomization at every iteration, and we can also accommodate label embeddings.
In experimental Section~\ref{sec: bert}, we will show that SSE-SE would improve original BERT pre-training procedure as well as fine-tuning procedure. 

\subsection{Theoretical Guarantees}

We explain why SSE can reduce the variance of estimators and thus leads to better generalization performance. 
For simplicity, we consider the SSE-graph objective \eqref{eq:single_obj} where there is no transition associated with the label $y_i$, and only the embeddings associated with the input $x_i$ undergo a transition.
When this is the case, we can think of the loss as a function of the $x_i$ embedding and the label, $\ell(\mathbf{E}[\mathbf{j}^i], y_i; \Theta)$.
We take this approach because it is more straightforward to compare our resulting theory to existing excess risk bounds.

The SSE objective in the case of only input transitions can be written as,
\begin{equation}
    S_n(\Theta) = \sum_i \sum_{\mathbf{k}} p(\mathbf{j}^i, \mathbf{k}) \cdot \ell(\mathbf{E}[\mathbf{k}],y_i|\Theta), 
\end{equation}
and there may be some constraint on $\Theta$.  
Let $\hat \Theta$ denote the minimizer of $S_n$ subject to this constraint. 
We will show in the subsequent theory that minimizing $S_n$ will get us close to a minimizer of $S(\Theta) = \mathbb E S_n(\Theta)$, and that under some conditions this will get us close to the Bayes risk.
We will use the standard definitions of empirical and true risk, $R_n(\Theta) = \sum_{i} \ell(x_i,y_i|\Theta)$ and $R(\Theta) = \mathbb E R_n(\Theta)$.

Our results depend on the following decomposition of the risk.  By optimality of $\hat \Theta$, 
\begin{equation}
    R(\hat{\Theta}) = S_n(\hat{\Theta}) + [R(\hat{\Theta}) - S(\hat{\Theta})] + [S(\hat{\Theta}) - S_n(\hat{\Theta})] \leq S_n(\Theta^\ast) + B(\hat{\Theta}) + \mathcal{E}(\hat{\Theta})
\end{equation}
where $B(\Theta) = | R(\Theta) - S(\Theta) |$,  and $E(\Theta) = | S(\Theta) - S_n(\Theta) |$. 
We can think of $B(\Theta)$ as representing the bias due to SSE, and $E(\Theta)$ as an SSE form of excess risk.
Then by another application of similar bounds,
\begin{equation}
R(\hat{\Theta}) \le R(\Theta^\ast) + B(\hat \Theta) + B(\Theta^\ast) + E(\hat \Theta) + E(\Theta^\ast).
\end{equation}
The high level idea behind the following results is that when the SSE protocol reflects the underlying distribution of the data, then the bias term $B(\Theta)$ is small, and if the SSE transitions are well mixing then the SSE excess risk $E(\Theta)$ will be of smaller order than the standard Rademacher complexity.
This will result in a small excess risk.


\begin{theorem}
\label{thm:rademacher}
Consider SSE-graph with only input transitions.
 
Let $L(\mathbf{E}[\mathbf{j}^i]) = \mathbb E_{Y | X = x^i} \ell(\mathbf{E}[\mathbf{j}^i],Y| \Theta)$ be the expected loss conditional on input $x^i$ and $e(\mathbf{E}[\mathbf{j}^i],y; \Theta) = \ell(\mathbf{E}[\mathbf{j}^i],y| \Theta) - L(\mathbf{E}[\mathbf{j}^i]| \Theta)$ be the residual loss.
Define the conditional and residual SSE empirical Rademacher complexities to be 
   \begin{align}
\label{eq:rade_1}
&\rho_{L,n} = \mathbb E_{\sigma} \sup_{\Theta} \left| \sum_{i} \sigma_{i} \sum_{\mathbf{k}} p(\mathbf{j}^i,\mathbf{k}) \cdot L(\mathbf{E}[\mathbf{k}] | \Theta) \right|, \\
\label{eq:rade_2}
&\rho_{e,n} = \mathbb E_{\sigma} \sup_{\Theta} \left| \sum_{i} \sigma_{i} \sum_{\mathbf{k}} p(\mathbf{j}^i,\mathbf{k}) \cdot e(\mathbf{E}[\mathbf{k}], y_i; \Theta) \right|,
\end{align}
respectively where $\sigma$ is a Rademacher $\pm 1$ random vectors in $\mathbb R^n$.
Then we can decompose the SSE empirical risk into 
\begin{equation}
\label{eq:excess_risk_bd}
\mathbb E \sup_{\Theta} |S_n(\Theta) - S(\Theta)| \le 2 \mathbb E [\rho_{L,n} + \rho_{e,n}].
\end{equation}
\end{theorem}

\begin{remark}
The transition probabilities in \eqref{eq:rade_1}, \eqref{eq:rade_2} act to smooth the empirical Rademacher complexity.  To see this, notice that we can write the inner term of \eqref{eq:rade_1} as $(P \sigma)^\top L$, where we have vectorized $\sigma_i, L(x_i;\Theta)$ and formed the transition matrix $P$.  Transition matrices are contractive and will induce dependencies between the Rademacher random variables, thereby stochastically reducing the supremum.  
In the case of no label noise, namely that $Y | X$ is a point mass, $e(x,y;\Theta) = 0$, and $\rho_{e,n} = 0$.  
The use of $L$ as opposed to the losses, $\ell$, will also make $\rho_{L,n}$ of smaller order than the standard empirical Rademacher complexity.
We demonstrate this with a partial simulation of $\rho_{L,n}$ on the Movielens1m dataset in Figure~\ref{fig:sim} of the Appendix.
\end{remark}

\begin{theorem}
\label{thm:main}
Let the SSE-bias be defined as
\[
\mathcal B = \sup_{\Theta} \left| \mathbb E \left[ \sum_i \sum_{\mathbf{k}} p(\mathbf{j}^i,  \mathbf{k}) \cdot \left( \ell(\mathbf{E}[\mathbf{k}],y_i|\Theta) - \ell(\mathbf{E}[\mathbf{j}^i],y_i|\Theta) \right) \right] \right|.
\]
Suppose that $0 \le \ell(.,.;\Theta) \le b$ for some $b > 0$,
then 
\[
\mathbb P \left\{ R(\hat \Theta) > R(\Theta^*) + 2 \mathcal B + 4 \mathbb E [\rho_{L,n} + \rho_{e,n}] + \sqrt n u \right\} \le e^{-\frac{u^2}{2 b^2}}.
\]
\end{theorem}

\begin{remark}
The price for `smoothing' the Rademacher complexity in Theorem \ref{thm:rademacher} is that SSE may introduce a bias.  This will be particularly prominent when the SSE transitions have little to do with the underlying distribution of $Y,X$.  On the other extreme, suppose that $p(\textbf{j}, \textbf{k})$ is non-zero over a neighborhood $\mathcal N_{\textbf{j}}$ of $\textbf{j}$, and that for data $x',y'$ with encoding $\textbf{k} \in \mathcal N_{\textbf{j}}$, $x', y'$ is identically distributed with $x_i,y_i$, then $\mathcal B = 0$.  In all likelihood, the SSE transition probabilities will not be supported over neighborhoods of iid random pairs, but with a well chosen SSE protocol the neighborhoods contain approximately iid pairs and $\mathcal B$ is small.
\end{remark}

\begin{table}
 \caption{Compare SSE-Graph and SSE-SE against ALS-MF with Graph Laplacian Regularization. The $p_u$ and $p_i$ are the SSE probabilities for user and item embedding tables respectively, as in \eqref{eq:sse-se}. Definitions of $\rho_u$ and $\rho_i$ can be found in \eqref{eq:transition1}. Movielens10m does not have user graphs. 
}
  \label{tb:sse-se_sse-graph2}
  \centering
 \resizebox{0.8\columnwidth}{!}{
  \begin{tabular}{cccccccccccc}
    \toprule
    &\multicolumn{5}{c}{Movielens1m} & \multicolumn{5}{c}{Movielens10m}  \\
    \cmidrule(r){2-6}   \cmidrule(r){7-11}
    Model    & RMSE & $ \rho_{u}$& $\rho_{i}$ & $p_{u}$ & $p_{i}$ & RMSE & $ \rho_{u}$& $\rho_{i}$ & $p_{u}$ & $p_{i}$\\
    \midrule
    SGD-MF            & 1.0984 & - &     - & - & - & 1.9490 & - & -   & -   &  -\\
  
    Graph Laplacian + ALS-MF   & 1.0464 & - &  - &  -    & -  & 1.9755&- &  - & -   &  -\\
    
SSE-Graph + SGD-MF & \textbf{1.0145}& 500   &  200  & 0.005 & 0.005 & \textbf{1.9019}& 1 & 500  & 0.01   &  0.01 \\
  SSE-SE + SGD-MF    & 1.0150 & 1 & 1 & 0.005 & 0.005 & 1.9085& 1 & 1& 0.01   &  0.01\\
    \bottomrule
  \end{tabular}
}
\end{table}


\begin{table}
  \caption{SSE-SE outperforms Dropout for Neural Networks with One Hidden Layer such as Matrix Factorization Algorithm regardless of dimensionality we use. $p_s$ is the SSE probability for both user and item embedding tables and $p_d$ is the dropout probability.}
  \label{tb:mf}
  \centering
 \resizebox{0.7\columnwidth}{!}{
  \begin{tabular}{cccccccccc}
    \toprule
    & \multicolumn{3}{c}{Douban} & \multicolumn{3}{c}{Movielens10m}  & \multicolumn{3}{c}{Netflix}          \\
    \cmidrule(r){2-4} \cmidrule(r){5-7}  \cmidrule(r){8-10} 
    Model     & RMSE & $p_{d}$ & $p_{s}$ & RMSE & $p_{d}$ & $p_{s}$ & RMSE & $p_{d}$ & $p_{s}$\\
    \midrule
    MF   & 0.7339 & - & - & 0.8851  & -  & - & 0.8941 & -  & - \\
    Dropout + MF   & 0.7296 & 0.1 & - & 0.8813  & 0.1  & - &  0.8897 & 0.1  & -   \\
    SSE-SE + MF   & 0.7201 & - & 0.008  & 0.8715 & -   & 0.008  &  0.8842 &  -   & 0.008 \\
    SSE-SE + Dropout + MF   & \textbf{0.7185} & 0.1 & 0.005  & \textbf{0.8678} & 0.1  & 0.005   & \textbf{0.8790} & 0.1  & 0.005 \\
    \bottomrule
  \end{tabular}
}
\end{table}

\begin{table}
  \caption{SSE-SE outperforms dropout for Neural Networks with One Hidden Layer such as Bayesian Personalized Ranking Algorithm regardless of dimensionality we use. We report the metric precision for top $k$ recommendations as $P@k$.}
  \label{tb:bpr}
  \centering
 \resizebox{0.75\columnwidth}{!}{
  \begin{tabular}{cccccccccc}
    \toprule
    & \multicolumn{3}{c}{Movielens1m}   & \multicolumn{3}{c}{Yahoo Music}   & \multicolumn{3}{c}{Foursquare}        \\
    \cmidrule(r){2-4} \cmidrule(r){5-7}  \cmidrule(r){8-10}  
    Model     & $P@1$ & $P@5$ & $P@10$ & $P@1$ & $P@5$ & $P@10$ & $P@1$ & $P@5$ & $P@10$ \\
    \midrule
    SQL-Rank (2018) & \textbf{0.7369} & 0.6717 & 0.6183  &  \textbf{0.4551}   &  \textbf{0.3614}   & \textbf{0.3069}  & 0.0583  & 0.0194   & \textbf{0.0170}  \\
    \midrule
    BPR   & 0.6977 & 0.6568 &  0.6257   & 0.3971 & 0.3295 &   0.2806 & 0.0437 & 0.0189 &   0.0143  \\
    Dropout + BPR   & 0.7031 & 0.6548 &  0.6273  & 0.4080 & 0.3315 &  0.2847 & 0.0437 & 0.0184 &  0.0146  \\
    SSE-SE + BPR   & 0.7254 & \textbf{0.6813} &  \textbf{0.6469}   & 0.4297 & 0.3498 & 0.3005 & \textbf{0.0609} & \textbf{0.0262} & 0.0155  \\
    \bottomrule
  \end{tabular}
}
\end{table}

\section{Experiments}\label{sec:experiments}
We have conducted extensive experiments on 6 tasks, including 3 recommendation tasks (explicit feedback, implicit feedback and sequential recommendation) and 3 NLP tasks (neural machine translation, BERT  pre-training, and BERT fine-tuning for sentiment classification) and found that our proposed SSE can effectively improve generalization performances on a wide variety of tasks. 
Note that the details about datasets and parameter settings can be found in the appendix.  

\subsection{Neural Networks with One Hidden Layer (Matrix Factorization and BPR)}\label{sec:simple_nn}


Matrix Factorization Algorithm (MF) \cite{mnih2008probabilistic} and Bayesian Personalized Ranking Algorithm (BPR) \cite{rendle2009bpr} can be viewed as neural networks with one hidden layer (latent features) and are quite popular in recommendation tasks.
MF uses the squared loss designed for explicit feedback data while BPR uses the pairwise ranking loss designed for implicit feedback data. 

First, we conduct experiments on two explicit feedback datasets: Movielens1m and Movielens10m.
For these datasets, we can construct graphs based on actors/actresses starring the movies. 
We compare SSE-graph and the popular Graph Laplacian Regularization (GLR) method~\cite{rao2015collaborative} in Table~\ref{tb:sse-se_sse-graph2}. The results  show that SSE-graph consistently outperforms GLR. 
This indicates that our SSE-Graph has greater potentials over graph Laplacian regularization as we do not explicitly penalize the distances across embeddings, but rather we implicitly penalize the effects of similar embeddings on the loss.
Furthermore, we show that even without existing knowledge graphs of embeddings, our SSE-SE  performs only slightly worse than SSE-Graph but still much better than GLR and MF. 

In general, SSE-SE is a good alternative when graph information is not available.
We then show that our proposed SSE-SE can be used together with standard regularization techniques such as dropout and weight decay to improve recommendation results regardless of the loss functions and dimensionality of embeddings.  This is evident in Table~\ref{tb:mf} and Table~\ref{tb:bpr}. With the help of SSE-SE, BPR can perform better than the state-of-art listwise approach SQL-Rank \cite{wu2018sql} in most cases. We include the optimal SSE parameters in the table for references and leave out other experiment details to the appendix. In the rest of the paper, we would mostly focus on SSE-SE as we do not have high-quality graphs of embeddings on most datasets.

\begin{table}
  \caption{SSE-SE has two tuning parameters: probability $p_x$ to replace embeddings associated with input $x_i$ and probability $p_y$ to replace embeddings associated with output $y_i$. We use the dropout probability of $0.1$, weight decay of $1e^{-5}$, and learning rate of $1e^{-3}$ for all experiments.}
  \label{tb:sasrec}
  \centering
 \resizebox{0.75\columnwidth}{!}{
  \begin{tabular}{ccccccc}
    \toprule
    & \multicolumn{2}{c}{Movielens1m } & Dimension& $\#$ of Blocks  & \multicolumn{2}{c}{SSE-SE Parameters}             \\
    \cmidrule(r){2-3} \cmidrule(r){4-5}  \cmidrule(r){6-7}
    Model     & NDCG$@10$ & Hit Ratio$@10$ & $d$ & $b$ & $p_x$  & $p_y$ \\
    \midrule
    SASRec   & 0.5941 & 0.8182 & 100 & 2  & -  & -\\
    SASRec   & 0.5996 & 0.8272 & 100 & 6  & -  & -\\
    \midrule
    SSE-SE + SASRec   & 0.6092 & 0.8250 & 100  & 2 & 0.1  & 0  \\
    SSE-SE + SASRec   & 0.6085 & 0.8293 & 100  & 2 & 0  & 0.1  \\
    SSE-SE + SASRec   & 0.6200 & 0.8315 & 100  & 2 & 0.1  & 0.1  \\
    \midrule
    SSE-SE + SASRec   & \bfseries{0.6265} & \bfseries{0.8364} & 100  & 6 & 0.1  & 0.1  \\
    \bottomrule
  \end{tabular}
}
\end{table}

\subsection{Transformer Encoder Model for Sequential Recommendation}\label{sec:sasrec}


SASRec \cite{kang2018self} is the state-of-the-arts algorithm for sequential recommendation task. It applies the transformer model \cite{vaswani2017attention}, where a sequence of items purchased by a user can be viewed as a sentence in transformer, and next item prediction is equivalent to next word prediction in the language model.
In Table~\ref{tb:sasrec}, we perform SSE-SE on input embeddings ($p_x=0.1$, $p_y=0$), output embeddings ($p_x=0.1$, $p_y=0$) and both embeddings ($p_x=p_y=0.1$), and observe that all of them significantly improve over state-of-the-art SASRec ($p_x=p_y=0$). 
The regularization effects of SSE-SE is even more obvious when we increase the number of self-attention blocks from 2 to 6, as this will lead to a more sophisticated model with many more parameters. This leads to the model overfitting terribly even with dropout and weight decay. We can see in Table~\ref{tb:sasrec} that when both methods use dropout and weight decay, SSE-SE + SASRec is doing much better than SASRec without SSE-SE.

\begin{table}
  \caption{Our proposed SSE-SE helps the Transformer achieve better BLEU scores on English-to-German in 10 out of 11 newstest data between 2008 and 2018.}
  \label{tb:nmt}
  \centering
  \resizebox{0.8\columnwidth}{!}{
  \begin{tabular}{cccccccccccc}
    \toprule
    & \multicolumn{11}{c}{Test BLEU}                   \\
    \cmidrule(r){2-12} 
    Model     & 2008 & 2009 & 2010 & 2011 & 2012   & 2013  & 2014 & 2015 & 2016 & 2017 & 2018 \\
    \midrule
    Transformer            & 21.0 &      20.7  & 22.7 
                        & 20.6 & 20.6 & \bfseries{25.3}   & 26.2 & 28.4 & 32.1 & 27.2 &  38.8  \\
    SSE-SE + Transformer    & \bfseries{21.4} & \bfseries{21.1} & \bfseries{23.0} 
                        & \bfseries{21.0} & \bfseries{20.8}  & 25.2  & \bfseries{27.2}  & \bfseries{29.2} & \bfseries{33.1} & \bfseries{27.9} &  \bfseries{39.9} \\
    \bottomrule
  \end{tabular}
  }
\end{table}

\subsection{Neural Machine Translation}\label{sec:nmt}
We use the transformer model \cite{vaswani2017attention} as the backbone for our experiments. 
The baseline model is the standard 6-layer transformer architecture and we apply SSE-SE to both encoder, and decoder by replacing corresponding vocabularies' embeddings in the source and target sentences. We trained on the standard WMT 2014 English to German dataset which consists of roughly 4.5 million parallel sentence pairs and tested on WMT 2008 to 2018 news-test sets. 
We use the OpenNMT implementation in our experiments. We use the same dropout rate of 0.1 and label smoothing value of 0.1 for the baseline model and our SSE-enhanced model. 
The only difference between the two models is whether or not we use our proposed SSE-SE with $p_0 = 0.01$ in \eqref{eq:sse-se} for both encoder and decoder embedding layers. We evaluate both models' performances on the test datasets using BLEU scores \cite{post-2018-call}.


We summarize our results in Table~\ref{tb:nmt} and find that SSE-SE helps improving accuracy and BLEU scores on both dev and test sets in 10 out of 11 years from 2008 to 2018. In particular, on the last 5 years' test sets from 2014 to 2018, the transformer model with SSE-SE improves BLEU scores by 0.92 on average when compared to the baseline model without SSE-SE.

\subsection{BERT for Sentiment Classification}\label{sec: bert}


BERT's model architecture \cite{devlin2018bert} is a multi-layer bidirectional Transformer encoder based on the Transformer model in neural machine translation.
Despite SSE-SE can be used for both pre-training and fine-tuning stages of BERT, we want to mainly focus on pre-training as fine-tuning bears more similarity to the previous section. We use SSE probability of 0.015 for embeddings (one-hot encodings) associated with labels and SSE probability of 0.015 for embeddings (word-piece embeddings) associated with inputs. One thing worth noting is that even in the original BERT model's pre-training stage, SSE-SE is already implicitly used for token embeddings. In original BERT model, the authors masked 15\% of words for a maximum of 80 words in sequences of maximum length of 512 and 10\% of the time replaced the [mask] token with a random token. That is roughly equivalent to SSE probability of 0.015 for replacing input word-piece embeddings. 


We continue to pre-train Google pre-trained BERT model on our crawled IMDB movie reviews with and without SSE-SE and compare downstream tasks performances. In Table~\ref{tb:bert-imdb}, we find that SSE-SE pre-trained BERT base model helps us achieve the state-of-the-art results for the IMDB sentiment classification task, which is better than the previous best in \cite{howard2018universal}. We report test set accuracy of 0.9542 after fine-tuning for one epoch only. For the similar SST-2 sentiment classification task in Table~\ref{tb:bert-sst2}, we also find that SSE-SE can improve BERT pre-trains better. Our SSE-SE pre-trained model achieves 94.3\% accuracy on SST-2 test set after 3 epochs of fine-tuning while the standard pre-trained BERT model only reports 93.8 after fine-tuning. Furthermore, we show that SSE-SE with SSE probability 0.01 can also improve dev and test accuracy in the fine-tuning stage. If we are using SSE-SE for both pre-training and fine-tuning stage of the BERT base model, we can achieve 94.5\% accuracy on the SST-2 test set, approaching the 94.9\% accuracy by the BERT large model. We are optimistic that our SSE-SE can be applied to BERT large model as well in the future. 



\begin{table}
  \caption{Our proposed SSE-SE applied in the pre-training stage on our crawled IMDB data improves the generalization ability of pre-trained IMDB model and helps the BERT-Base model outperform current SOTA results on the IMDB Sentiment Task after fine-tuning.}
  \label{tb:bert-imdb}
  \centering
 \resizebox{0.65\columnwidth}{!}{
  \begin{tabular}{cccc}
    \toprule
    & \multicolumn{3}{c}{IMDB Test Set}                \\
    \cmidrule(r){2-4}  
    Model     & AUC & Accuracy & F1 Score \\
    \midrule
    ULMFiT \cite{howard2018universal} & - & 0.9540 & - \\
    \midrule
    Google Pre-trained Model + Fine-tuning            & 0.9415 &  0.9415    & 0.9419  \\
    Pre-training + Fine-tuning   & 0.9518 & 0.9518 &  0.9523 \\
    (SSE-SE + Pre-training) + Fine-tuning   & \bfseries{0.9542}  & \bfseries{0.9542}    & \bfseries{0.9545}    \\
    \bottomrule
  \end{tabular}
 }
\end{table}

\begin{table}
  \caption{SSE-SE pre-trained BERT-Base models on IMDB datasets turn out working better on the new unseen SST-2 Task as well. }
  \label{tb:bert-sst2}
  \centering
\resizebox{0.8\columnwidth}{!}{
  \begin{tabular}{ccccc}
    \toprule
    & \multicolumn{3}{c}{SST-2 Dev Set} & \multicolumn{1}{c}{SST-2 Test Set}                 \\
    \cmidrule(r){2-4}  \cmidrule(r){5-5} 
    Model     & AUC & Accuracy & F1 Score & Accuracy (\%) \\
    \midrule 
    Google Pre-trained + Fine-tuning             & 0.9230 &      0.9232  & 0.9253  &   93.6  \\
    Pre-training + Fine-tuning   & 0.9265 & 0.9266 & 0.9281 &   93.8 \\
    (SSE-SE + Pre-training) + Fine-tuning   & 0.9276 &  0.9278  & 0.9295  & 94.3  \\
    (SSE-SE + Pre-training) + (SSE-SE + Fine-tuning) & \bfseries{0.9323} &  \bfseries{0.9323}  & \bfseries{0.9336}  & \bfseries{94.5} \\

    \bottomrule
  \end{tabular}
}
\end{table}

\begin{figure*}
\centering
\begin{tabular}{cc}
\hspace{-8pt}
\includegraphics[width=0.45\linewidth]{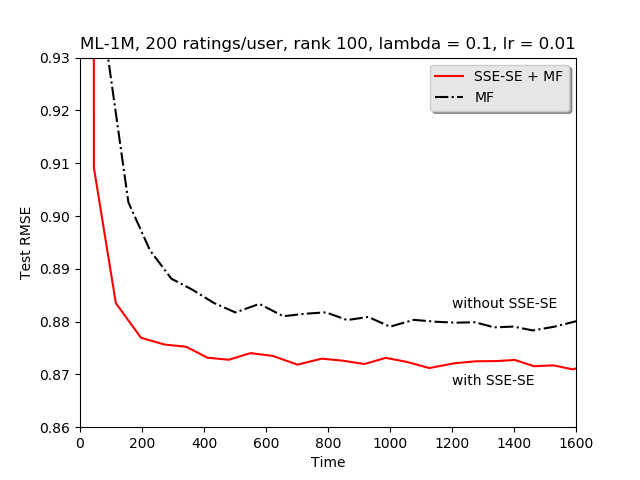} &
\includegraphics[width=0.45\linewidth]{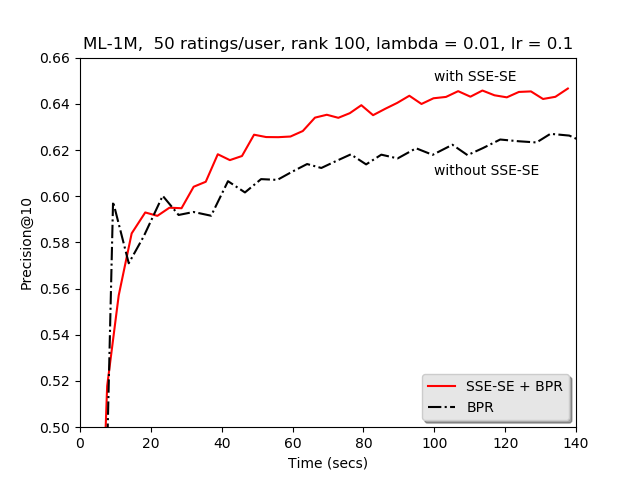} 
\end{tabular}
\vspace{-10pt}\caption{Compare Training Speed of Simple Neural Networks with One Hidden Layer, i.e. MF and BPR, with and without SSE-SE. }
\label{fig:sse-speed}
\end{figure*}

\subsection{Speed and convergence comparisons. }
In Figure~\ref{fig:sse-speed}, it is clear to see that our one-hidden-layer neural networks with SSE-SE are achieving much better generalization results than their respective standalone versions. One can also easily spot that SSE-version algorithms converge at much faster speeds with the same learning rate.

\section{Conclusion} We have proposed Stochastic Shared Embeddings, which is a data-driven approach to regularization, that stands in contrast to brute force regularization such as Laplacian and ridge regularization.  Our theory is a first step towards explaining the regularization effect of SSE, particularly, by `smoothing' the Rademacher complexity.  The extensive experimentation demonstrates that SSE can be fruitfully integrated into existing deep learning applications.

\chapter{SSE-PT: Sequential Recommendation Via Personalized Transformer}

\section{Introduction}

The sequential recommendation problem has been an important open research question, yet using temporal information to improve recommendation performance has proven to be challenging. 
SASRec, proposed by \cite{kang2018self} for sequential recommendation problems, has achieved state-of-the-art results and enjoyed more than 10x speed-up when compared to earlier CNN/RNN-based methods. 
However, the model used in SASRec is the standard Transformer which is inherently an un-personalized model. 
In practice, it is important to include a personalized Transformer in SASRec especially for recommender systems, but 
\cite{kang2018self} found that adding additional personalized embeddings did not improve the performance of their Transformer model, and postulate that the failure of adding personalization is due to the fact that they already use the user history and the user embeddings only contribute to overfitting.
In this work, we propose a novel method, Personalized Transformer (SSE-PT), that successfully introduces personalization into self-attentive neural network architectures.

Introducing user embeddings into the standard transformer model is intrinsically difficult with existing regularization techniques, as unavoidably a large number of user parameters are introduced, which is often at the same scale of the number of training data. But we show that personalization can greatly improve ranking performance with a recent regularization technique called Stochastic Shared Embeddings (SSE) \cite{wu2019stochastic}. 
The personalized Transformer (SSE-PT) model with SSE regularization works well for all 5 real-world datasets we consider without overfitting, outperforming previous state-of-the-art algorithm SASRec by almost 5\% in terms of NDCG@10. Furthermore, after examining some random users' engagement history, we find our model is not only more interpretable but also able to focus on recent engagement patterns for each user. Moreover, our SSE-PT model with a slight modification, which we call SSE-PT++, can handle extremely long sequences and outperform SASRec in ranking results with comparable training speed, striking a balance between performance and speed requirements.

\section{Related Work}

\subsection{Session-based and Sequential Recommendation}
Both session-based and sequential (i.e., next-basket) recommendation algorithms take advantage of additional temporal information to make better personalized recommendations. The main difference between session-based recommendations \cite{hidasi2015session} and sequential recommendations \cite{kang2018self} is that the former assumes that the user ids are not recorded and therefore the length of engagement sequences are relatively short. Therefore, session-based recommendations normally do not consider user factors. On the other hand, sequential recommendation treats each sequence as a user's engagement history \cite{kang2018self}.
Both settings, do not explicitly require time-stamps: only the relative temporal orderings are assumed known (in contrast to, for example, timeSVD++ \cite{koren2009collaborative} using time-stamps).
Initially, sequence data in temporal order are usually modelled with Markov models, in which a future observation is conditioned on the last few observed items \cite{rendle2010factorizing}. 
In \cite{rendle2010factorizing}, a personalized Markov model with user latent factors is proposed for more personalized results. 

In recent years, deep learning techniques, borrowed from natural language processing (NLP) literature, are getting widely used in tackling sequential data. Like word sentences in NLP, item sequences in recommendations can be similarly modelled by recurrent neural networks (RNN) \cite{hidasi2015session, hidasi2018recurrent} and convolutional neural network (CNN) \cite{tang2018personalized} models. Recently, attention models are increasingly used in both NLP \cite{vaswani2017attention, devlin2018bert} and recommender systems \cite{liu2018stamp, kang2018self}. SASRec \cite{kang2018self} is a recent method with state-of-the-art performance among the many deep learning models. 
Motivated by the Transformer model in neural machine translation \cite{vaswani2017attention}, SASRec utilizes a similar architecture to the encoder part of the Transformer model. Our proposed model, SSE-PT, is a personalized extension of the transformer model.


\subsection{Regularization Techniques}
In deep learning, models with many more parameters than data points can easily overfit to the training data. This may prevent us from adding user embeddings as additional parameters into complicated models like the Transformer model \cite{kang2018self}, which can easily have 20 layers 
with millions of parameters for a medium-sized dataset like Movielens10M \cite{harper2016movielens}. 
$\ell_2$ regularization~\cite{hoerl1970ridge} is the most widely used approach and has been used in many matrix factorization models in recommender systems; $\ell_1$ regularization~\cite{tibshirani1996regression} is used when a sparse model is preferred. For deep neural networks, it has been shown that $\ell_p$  regularizations are often too weak, while dropout~\cite{hinton2012improving,srivastava2014dropout} is more effective in practice.  There are many other regularization techniques, including parameter sharing \cite{goodfellow2016deep}, max-norm regularization \cite{srebro2005maximum}, gradient clipping \cite{pascanu2013difficulty}, etc.
Very recently, a new regularization technique called Stochastic Shared Embeddings (SSE) \cite{wu2019stochastic} is proposed as a new means of regularizing embedding layers.  We find that the base version SSE-SE is essential to the success of our Personalized Transformer (SSE-PT) model.


\section{Methodology}
\subsection{Sequential Recommendation}
Given $n$ users and each user engaging with a subset of $m$ items in a temporal order, the goal of sequential recommendation is to 
learn a good  personalized ranking of top $K$ items out of total $m$ items for any given user at any given time point. 
We assume  data in the format of $n$ item sequences: 
\begin{equation}\label{eq:input}
  s_i=(j_{i1}, j_{i2}, \dots, j_{iT}) \text{ for } 1 \leq i \leq n.  
\end{equation}
 Sequences $s_i$ of length $T$ contain indices of the last $T$ items that user $i$ has interacted with in the temporal order (from old to new). 
For different users, the sequence lengths can vary, but we can  pad the shorter sequences so all of them have length $T$. We cannot simply randomly split data points into train/validation/test sets because they come in temporal orders. Instead, we need to make sure our training data is before validation data which is before test data temporally.
We use last items in sequences as test sets, second-to-last items as validation sets and the rest as training sets.  We use ranking metrics such as NDCG@$K$ and Recall$@K$ for evaluations, which are defined in the Appendix.

\subsection{Personalized Transformer Architecture}
Our model, which we call SSE-PT, is motivated by the Transformer model in \cite{vaswani2017attention} and \cite{kang2018self}. It also utilizes a new regularization technique called stochastic shared embeddings \cite{wu2019stochastic}.
In the following sections, we are going to examine each important component of our Personalized Transformer (SSE-PT) model, especially the embedding layer, and the novel application of stochastic shared embeddings (SSE) regularization technique.


\paragraph{Embedding Layer}
We define a learnable user embedding look-up table $U \in R^{n \times d_u}$ and item embedding look-up table $V \in R^{m \times d_i}$, where  $d_u$, $d_i$ are the number of hidden units for user and item respectively. We also specify learnable positional encoding table $P \in R^{T \times d}$, where $d = d_u + d_i$.
So each input sequence 
$s_i \in R^T$ will be represented by the following embedding: 
\begin{equation}\label{eq:emb}
    E = \begin{bmatrix}
    [v_{j_{i1}}\text{; } u_i] + p_1   \\
    [v_{j_{i2}}\text{; } u_i] + p_2 \\
    \vdots \\
    [v_{j_{iT}}\text{; } u_i] + p_T
    \end{bmatrix} \in R^{T \times d},
\end{equation} 
where $[v_{j_{it}}; u_i]$ represents concatenating item embedding $v_{j_{it}} \in R^{d_i}$ and user embedding $u_i \in R^{d_u}$ into embedding $E_t \in R^d$ for time $t$. Note that the main difference between our model and \cite{kang2018self} is that we introduce the user embeddings $u_i$, making our model personalized.

\begin{figure}[ht]
\vskip -0.1in
\begin{center}
\centerline{\includegraphics[width=0.6\columnwidth]{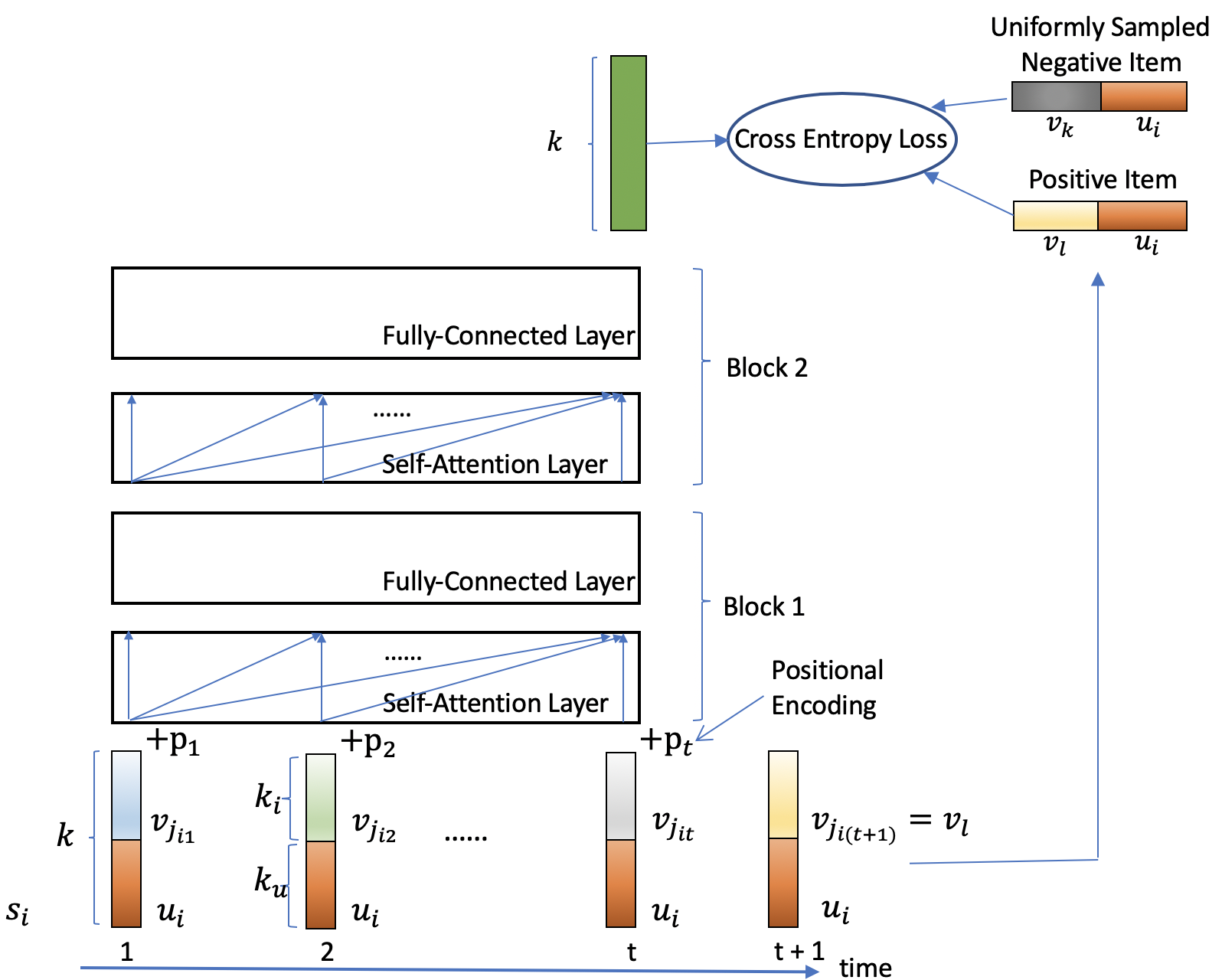}}  
\end{center}
\vskip -0.2in
\caption{Illustration of our proposed SSE-PT model}
\label{fig:SSE-PT}
\vspace{-10pt}\end{figure}

\paragraph{Transformer Encoder}
On top of the embedding layer, we have $B$ blocks of self-attention layers and fully connected layers, where each layer extracts features for each time step based on the previous layer's outputs.
Since this part is identical to the Transformer encoder used in the original papers~\cite{vaswani2017attention, kang2018self}, we will skip the details. 

\paragraph{Prediction Layer}
At time $t$, the predicted probability of user $i$ engaged item $l$ is: 
\vspace{-5pt}\begin{equation}
    p_{itl} = \sigma(r_{itl}),
\end{equation} where
$\sigma$ is the sigmoid function and $r_{itl}$ is the predicted score of item $l$ by user $l$ at time point $t$, defined as:
\begin{equation}\label{eq:pred}
        r_{itl} = F_{t-1}^{B} \cdot [v_l\text{; } u_i],
\end{equation}
where $F_{t-1}^{B}$ is the output hidden units associated with the transformer encoder at the last timestamp. 
Although we can use another set of user and item embedding look-up tables for the $u_i$ and $v_l$, we find it better to use the same set of embedding look-up tables $U, V$ as in the embedding layer. But regularization for those embeddings can be different. To distinguish the $u_i$ and $v_l$ in \eqref{eq:pred} from $u_i, v_j$ in \eqref{eq:emb}, we call embeddings in \eqref{eq:pred} output embeddings and those in \eqref{eq:emb} input embeddings.


The binary cross entropy loss between predicted probability for the positive item $l = j_{i (t+1)}$ and one uniformly sampled negative item $k \in \Omega$ is given as
$-[\log (p_{itl}) + \log(1 - p_{itk})]$. 
Summing over $s_i$ and $t$, we obtain the objective function that we want to minimize is:
\begin{equation}
    \sum\nolimits_{i}\sum\nolimits_{t=1}^{T-1} \sum_{k\in \Omega}-\big[\log (p_{itl}) +  \log(1 - p_{itk})\big].
\end{equation}
At the inference time, top-$K$ recommendations for user $i$ at time $t$ can be made by sorting scores  $r_{itl}$ for all items $\ell$ and recommending the first $K$ items in the sorted list.


\paragraph{Novel Application of Stochastic Shared Embeddings}

The most important regularization technique to SSE-PT model is 
the Stochastic Shared Embeddings (SSE) \cite{wu2019stochastic}. The main idea of SSE is to stochastically replace embeddings with another embedding with some pre-defined probability during SGD, which has the effect of regularizing the embedding layers.
Without SSE, all the existing well-known regularization techniques like layer normalization, dropout and weight decay fail and cannot prevent the model from over-fitting badly after introducing user embeddings. \cite{wu2019stochastic} develops two versions of SSE, SSE-Graph and SSE-SE. In the simplest uniform case, SSE-SE replaces one embedding with another embedding uniformly with probability $p$, which is called SSE probability in \cite{wu2019stochastic}.
Since we don't have knowledge graphs for user or items, we simply apply the SSE-SE 
to our SSE-PT model. We find SSE-SE
makes possible training this personalized model with $O(n d_u)$ additional parameters.

 There are 3 different places in our model that SSE-SE can be applied. We can apply SSE-SE to input/output user embeddings, input item embeddings, and output item embeddings with probabilities $p_u$, $p_i$ and $p_y$ respectively. Note that input user embedding and output user embedding are always replaced at the same time with SSE probability $p_u$.
Empirically, we find that SSE-SE to user embeddings and output item embeddings always helps, but SSE-SE to input item embeddings is only useful when the average sequence length is large, e.g., more than 100 in Movielens1M and Movielens10M datasets.

\paragraph{Other Regularization Techniques}
Besides the {\it SSE} \cite{wu2019stochastic}, we also utilized other widely used regularization techniques, including {\it layer normalization} \cite{ba2016layer}, {\it batch normalization} \cite{ioffe2015batch}, {\it residual connections} \cite{he2016deep}, {\it weight decay} \cite{krogh1992simple}, and {\it dropout} \cite{srivastava2014dropout}. Since they are used in the same way in the previous paper \cite{kang2018self}, we omit the details to the Appendix.

\subsection{Handling Long Sequences: SSE-PT++}
 To handle extremely long sequences, a slight modification can be made on the base SSE-PT model in terms of how input sequences $s_i$'s are fed into the SSE-PT neural network. We call the enhanced model SSE-PT++ to distinguish it from the previously discussed SSE-PT model, which cannot handle sequences longer than $T$.   
 
 The motivation of SSE-PT++ over SSE-PT comes from: sometimes we want to make use of extremely long sequences, $s_i=(j_{i1}, j_{i2}, \dots, j_{it}) \text{ for } 1 \leq i \leq n$, where $t > T$, but our SSE-PT model can only handle sequences of maximum length of $T$. 
 The simplest way is to sample starting index $1 \leq v \leq t$ uniformly and use $s_i=(j_{iv}, j_{i(v+1)}, \dots, j_{iz})$, where $z = \min (t, v + T - 1)$. Although sampling the starting index uniformly from $[1, t]$ can accommodate long sequences of length $t > T$, this does not work well in practice. Uniform sampling does not take into account the importance of recent items in a long sequence. To solve this dilemma, we introduce an additional hyper-parameter $p_s$ which we call {\it sampling probability}. It implies that with probability $p_s$, we sample the starting index $v$ uniformly from $[1, t - T]$ and use sequence $s_i=(j_{iv}, j_{i(v+1)}, \dots, j_{i(v+T-1)})$ as input. With probability $1 - p_s$, we simply use the recent $T$ items  $(j_{i(t-T+1)}, \dots, j_{it})$ as input. If the sequence $s_i$ is already shorter than $T$, then we always use the recent input sequence for user $i$.
 
 Our proposed SSE-PT++ model can work almost as well as SSE-PT with a much smaller $T$. One can see in Table~\ref{tb:dl-ml2} with $T = 100$, SSE-PT++ can perform almost as well as SSE-PT. The time complexity of the SSE-PT model is of order $O(T^2 d + T d^2)$. Therefore, reducing $T$ by one half would lead to a theoretically 4x speed-up in terms of the training and inference speeds. As to the model's space complexity, both SSE-PT and SSE-PT++ are of order $O(nd_u + md_i + Td + d^2)$. 


\section{Experiments}
In this section, we compare our proposed algorithms, Personalized Transformer (SSE-PT) and SSE-PT++, with other state-of-the-art algorithms on real-world datasets. We implement our codes in Tensorflow and conduct all our experiments on a server with 40-core Intel Xeon E5-2630 v4 @ 2.20GHz CPU, 256G RAM and Nvidia GTX 1080 GPUs.

\paragraph{Datasets}
We use 5 datasets. The first 4 have exactly the same train/dev/test splits as in \cite{kang2018self}. The datasets are: {\it Beauty} and {\it Games} categories from Amazon product review datasets\footnote{\url{http://jmcauley.ucsd.edu/data/amazon/}}; {\it Steam} dataset introduced in \cite{kang2018self}, which contains reviews crawled from a large video game distribution platform; {\it Movielens1M} dataset \cite{harper2016movielens}, a widely used benchmark datasets containing one million user movie ratings; {\it Movielens10M} dataset with ten million user ratings cleaned by us.
Detailed dataset statistics are given in Table~\ref{tb:datasets}. One can easily see that the first 3 datasets have short sequences (average length < 12) while the last 2 datasets have very long sequences (> 10x longer).

\paragraph{Evaluation Metrics}
The evaluation metrics we use are standard ranking metrics, namely NDCG and Recall for top recommendations (See Appendix).
We follow the same evaluation setting as the previous paper \cite{kang2018self}: predicting ratings at time point $t + 1$ given the previous $t$ ratings.   
For a large dataset with numerous users and items, the evaluation procedure would be slow because \eqref{eq:ndcg} would require computing the ranking of all items based on their predicted scores for every single user. As a means of speed-up evaluations, we sample a fixed number $C$ (e.g., 100) of negative candidates while always keeping the positive item that we know the user will engage next. This way, both $R_{ij}$ and $\Pi_i$ will be narrowed down to a small set of item candidates, and prediction scores will only be computed for those items through a single forward pass of the neural network.

Ideally, we want both NDCG and Recall to be as close to $1$ as possible, because NDCG@$K = 1$ means the positive item is always put on the top-$1$ position of the top-$K$ ranking list, and Recall@$K = 1$ means the positive item is always contained by the top-$K$ recommendations the model makes.

\begin{table*}[ht]
\centering
\caption{Comparing various state-of-the-art temporal collaborative ranking algorithms on various datasets. The (A) to (E) are non-deep-learning methods, the (F) to (K) are deep-learning methods and the (L) to (O) are our variants. We did not report SSE-PT++ results for beauty, games and steam, as the input sequence lengths are very short (see Table~\ref{tb:datasets}), so there is no need for SSE-PT++.}
\label{tb:dl-combined}
\begin{center}
\begin{small}
\begin{sc}
\resizebox{\textwidth}{!}{
\begin{tabular}{l:c:c:c:c}
\toprule

Dataset & BEAUTY & GAMES & STEAM & ML-1M \\  
Metric & Recall@10 \; NDCG@10 & Recall@10 \; NDCG@10 & Recall@10 \; NDCG@10 & Recall@10 \; NDCG@10\\
\midrule
(A) POPRec & 
\;\;0.4003 \;\;\;\;\;\;\;\; 0.2277 & 
\;\;0.4724 \;\;\;\;\;\;\;\; 0.2779 & 
\;\;0.7172 \;\;\;\;\;\;\;\; 0.4535 & 
\;\;0.4329 \;\;\;\;\;\;\;\; 0.2377 \\
(B) BPR & 
\;\;0.3775 \;\;\;\;\;\;\;\; 0.2183 & 
\;\;0.4853 \;\;\;\;\;\;\;\; 0.2875 & 
\;\;0.7061 \;\;\;\;\;\;\;\; 0.4436 & 
\;\;0.5781 \;\;\;\;\;\;\;\; 0.3287 \\
(C) FMC &
\;\;0.3771 \;\;\;\;\;\;\;\; 0.2477 & 
\;\;0.6358 \;\;\;\;\;\;\;\; 0.4456 & 
\;\;0.7731 \;\;\;\;\;\;\;\; 0.5193 & 
\;\;0.6983 \;\;\;\;\;\;\;\; 0.4676 \\
(D) FPMC &
\;\;0.4310 \;\;\;\;\;\;\;\; 0.2891 & 
\;\;0.6802 \;\;\;\;\;\;\;\; 0.4680 & 
\;\;0.7710 \;\;\;\;\;\;\;\; 0.5011 & 
\;\;0.7599 \;\;\;\;\;\;\;\; 0.5176 \\
(E) TRANSREC &
\;\;0.4607 \;\;\;\;\;\;\;\; 0.3020 & 
\;\;0.6838 \;\;\;\;\;\;\;\; 0.4557 & 
\;\;0.7624 \;\;\;\;\;\;\;\; 0.4852 & 
\;\;0.6413 \;\;\;\;\;\;\;\; 0.3969 \\
\midrule
(F) GRU4REC &
\;\;0.2125 \;\;\;\;\;\;\;\; 0.1203 & 
\;\;0.2938 \;\;\;\;\;\;\;\; 0.1837 & 
\;\;0.4190 \;\;\;\;\;\;\;\; 0.2691 & 
\;\;0.5581 \;\;\;\;\;\;\;\; 0.3381 \\
(G) STAMP &
\;\;0.4607 \;\;\;\;\;\;\;\; 0.3020 & 
\;\;0.6838 \;\;\;\;\;\;\;\; 0.4557 & 
\;\;0.7624 \;\;\;\;\;\;\;\; 0.4852 & 
\;\;0.6413 \;\;\;\;\;\;\;\; 0.3969 \\
(H) GRU4REC+ &
\;\;0.3949 \;\;\;\;\;\;\;\; 0.2556 & 
\;\;0.6599 \;\;\;\;\;\;\;\; 0.4759 & 
\;\;0.8018 \;\;\;\;\;\;\;\; 0.5595 & 
\;\;0.7501 \;\;\;\;\;\;\;\; 0.5513 \\
(I) CASER &
\;\;0.4264 \;\;\;\;\;\;\;\; 0.2547 &  
\;\;0.5282 \;\;\;\;\;\;\;\; 0.3214 & 
\;\;0.7874 \;\;\;\;\;\;\;\; 0.5381 & 
\;\;0.7886 \;\;\;\;\;\;\;\; 0.5538 \\
(J) SASREC &
\;\;0.4837 \;\;\;\;\;\;\;\; 0.3220 & 
\;\;0.7434 \;\;\;\;\;\;\;\; 0.5401 & 
\;\;0.8732 \;\;\;\;\;\;\;\; 0.6293 & 
\;\;0.8233 \;\;\;\;\;\;\;\; 0.5936 \\
(K) HGN &
\;\;0.4469 \;\;\;\;\;\;\;\; 0.2994 & 
\;\;0.7164 \;\;\;\;\;\;\;\; 0.5209 & 
\;\;0.7426 \;\;\;\;\;\;\;\; 0.4871 & 
\;\;0.7584 \;\;\;\;\;\;\;\; 0.5241 \\
\midrule
(L) SSE-SASREC &
\;\;0.4878 \;\;\;\;\;\;\;\; 0.3342 & 
\;\;0.7517 \;\;\;\;\;\;\;\; 0.5535 & 
\;\;0.8697 \;\;\;\;\;\;\;\; 0.6333 & 
\;\;0.8230 \;\;\;\;\;\;\;\; 0.5995 \\
(M) PT &
\;\;0.3954 \;\;\;\;\;\;\;\; 0.2449 & 
\;\;0.6427 \;\;\;\;\;\;\;\; 0.4434 & 
\;\;0.7535 \;\;\;\;\;\;\;\; 0.4853 & 
\;\;0.7658 \;\;\;\;\;\;\;\; 0.5241 \\
(N) SSE-PT &
\;\;\textbf{0.5028} \;\;\;\;\;\;\;\; \textbf{0.3370} & 
\;\;\textbf{0.7757} \;\;\;\;\;\;\;\; \textbf{0.5660} & 
\;\;\textbf{0.8772} \;\;\;\;\;\;\;\; \textbf{0.6378} & 
\;\;0.8341 \;\;\;\;\;\;\;\; 0.6281 \\
(O) SSE-PT++ &
 \; \; -- \;\;\;\;\;\;\;\;\;\;\;\;\;\;\;\;\; -- & 
 \; \; -- \;\;\;\;\;\;\;\;\;\;\;\;\;\;\;\;\; -- & 
 \; \; -- \;\;\;\;\;\;\;\;\;\;\;\;\;\;\;\;\; -- & 
\;\;\textbf{0.8389} \;\;\;\;\;\;\;\; \textbf{0.6292} \\

\bottomrule
\end{tabular}
}

\end{sc}
\end{small}
\end{center}
\vskip 0.1in
\end{table*}

\begin{table}[ht]
\centering
\caption{Comparing SASRec, SSE-PT and SSE-PT++ on Movielens1M Dataset while varying the maximum length allowed and dimension of embeddings.}
\label{tb:dl-ml2}

\begin{center}
\begin{small}
\begin{sc}
\resizebox{0.8\textwidth}{!}{
\begin{tabular}{lcccccc}
\toprule
 Methods & NDCG@$10$ & Recall@$10$ & Max Len & user dim & item dim \\
\midrule
SASREC   &  0.5769  & 0.8045   & 100 & N/A & 100 \\
SASREC   &  0.5936  & 0.8233 &200  & N/A & 50 \\
SASREC   &  0.5919 & 0.8202  &200 & N/A & 100 \\
   \midrule
SSE-PT & 0.6142 & 0.8212 & 100 &  50 & 100  \\
SSE-PT & 0.6191 & 0.8358 &200 & 50 & 50  \\
SSE-PT & 0.6281 & 0.8341 &200 &  50 & 100  \\
  \midrule
SSE-PT++ & 0.6186 & 0.8318 & 100 &  50 & 100  \\
SSE-PT++ & 0.6208 & 0.8358 & 200 &  50 & 50  \\
SSE-PT++ & \bfseries{0.6292} & \bfseries{0.8389} &200 &  50 & 100  \\


\bottomrule
\end{tabular}
}
\end{sc}
\end{small}
\end{center}
\vskip -0.15in
\end{table}

\paragraph{Baselines}
We include 5 non-deep-learning and 6 deep-learning algorithms in our comparisons.

\paragraph{Non-deep-learning Baselines} The simplest baseline is {\it PopRec}, basically ranking items according to their popularity. More advanced methods such as matrix factorization based baselines include Bayesian personalized ranking for implicit feedback \cite{rendle2009bpr}, namely {\it BPR}; Factorized Markov Chains and Personalized Factorized Markov Chains models \cite{rendle2010factorizing} also known as {\it FMC} and {\it PFMC}; and translation based method \cite{he2017translation} called {\it TransRec}.

\paragraph{Deep-learning Baselines}
Recent years have seen many advances in deep learning for sequential recommendations. {\it GRU4Rec} is the first RNN-based method proposed for this problem \cite{hidasi2015session}; {\it GRU4Rec$^+$} \cite{hidasi2018recurrent} later is proposed to address some shortcomings of the initial version. {\it Caser} is the corresponding CNN-based method \cite{tang2018personalized}. {\it STAMP} \cite{liu2018stamp} utilizes the attention mechanism without using RNN or CNN as building blocks. Very recently, {\it SASRec} utilizes state-of-art Transformer encoder \cite{vaswani2017attention} with self-attention mechanisms. Hierarchical gating networks, also known as {\it HGN} \cite{ma2019hierarchical} are also proposed to solve this problem.

\begin{table}[ht]
\centering
\caption{Comparing Different Regularizations for SSE-PT on Movielen1M Dataset. NO REG stands for no regularization.
PS stands for parameter sharing across all users while PS(AGE) means PS is used within each age group. 
SASRec is added to last row after all SSE-PT results as a baseline.}
\label{tb:reg}
\begin{center}
\begin{small}
\begin{sc}
\resizebox{0.7\textwidth}{!}{
\begin{tabular}{lcccccr}
\toprule
Regularization & NDCG@$5$ & $\%$ GAIN & Recall@$5$ & $\%$ GAIN \\
\midrule
NO REG (BASELINE)  & 0.4855 & - & 0.6500 &  -   \\
PS            & 0.5065 & 4.3 & 0.6656 & 2.4  \\
PS (JOB)       & 0.4938 & 1.7 & 0.6570 & 1.1  \\
PS (GENDER)    & 0.5110 & 5.3 & 0.6672 & 2.6  \\
PS (AGE)       & 0.5133 & 5.7 & 0.6743 & 3.7  \\
$l_2$            & 0.5149 & 6.0 & 0.6786 & 4.4  \\
DROPOUT       & 0.5165 & 6.4 & 0.6823 & 5.0  \\
$l_2$  + DROPOUT & 0.5293 & 9.0  & 0.6921 & 6.5  \\
 SSE-SE            & 0.5393 & 11.1 & 0.6977 & 7.3  \\
$l_2$ + SSE-SE + DROPOUT &  \bfseries{0.5870} & \textbf{20.9} & \textbf{0.7442} & \textbf{14.5}  \\
\midrule
SASRec ($l_2$ + DROPOUT)& 0.5601 & & 0.7164 &  \\

\bottomrule
\end{tabular}
}
\end{sc}
\end{small}
\end{center}
\end{table}

\paragraph{Experiment Setup}
We use the same datasets as in \cite{kang2018self} and follow the same procedure in the paper: use last items for each user as test data, second-to-last as validation data and the rest as training data.  We implemented our method in Tensorflow and solve it with Adam Optimizer \cite{kingma2014adam} with a learning rate of $0.001$, momentum exponential decay rates $\beta_1 = 0.9, \beta_2 = 0.98$ and a batch size of $128$. In Table~\ref{tb:dl-combined}, since we use the same data, the performance of previous methods except STAMP have been reported in \cite{kang2018self}. We tune the dropout rate, and SSE probabilities $p_u, p_i, p_y$ for input user/item embeddings and output embeddings on validation sets 
and report the best NDCG and Recall for top-$K$ recommendations on test sets. 
For a fair comparison,
we restrict all algorithms to use up to 50 hidden units for item embeddings.
For the SSE-PT and SASRec models, we use the same number of transformer encoder blocks (i.e. $B = 2$) and set the maximum length $T = 200$ for Movielens 1M and 10M dataset and $T = 50$ for other datasets.
We use top-$K$ with $K = 10$ and the number of negatives $C = 100$ in the evaluation procedure. In practice, using a different $K$ and $C$ does not affect our conclusions.

\paragraph{Comparisons} One can easily see from Table~\ref{tb:dl-combined} that our proposed SSE-PT has the best performance over all previous methods on all four datasets. On most datasets, our SSE-PT improves NDCG by more than 4\% when compared with SASRec \cite{kang2018self} and more than 20\% when compared to non-deep-learning methods.
SSE-SE, together with dropout and weight decay, is the best choice for regularization, which is evident from Table~\ref{tb:reg}. SSE-SE is a more effective way to regularize our neural networks than any existent techniques including parameter sharing, dropout, weight decay. In practice, these SSE probabilities, just like dropout rate, can be treated as tuning parameters and easily tuned. Movielens10M results are left to Table~\ref{tb:ml10m} in the Appendix.

\begin{figure*}[ht]
\begin{center}
\centerline{\includegraphics[width=1\columnwidth]{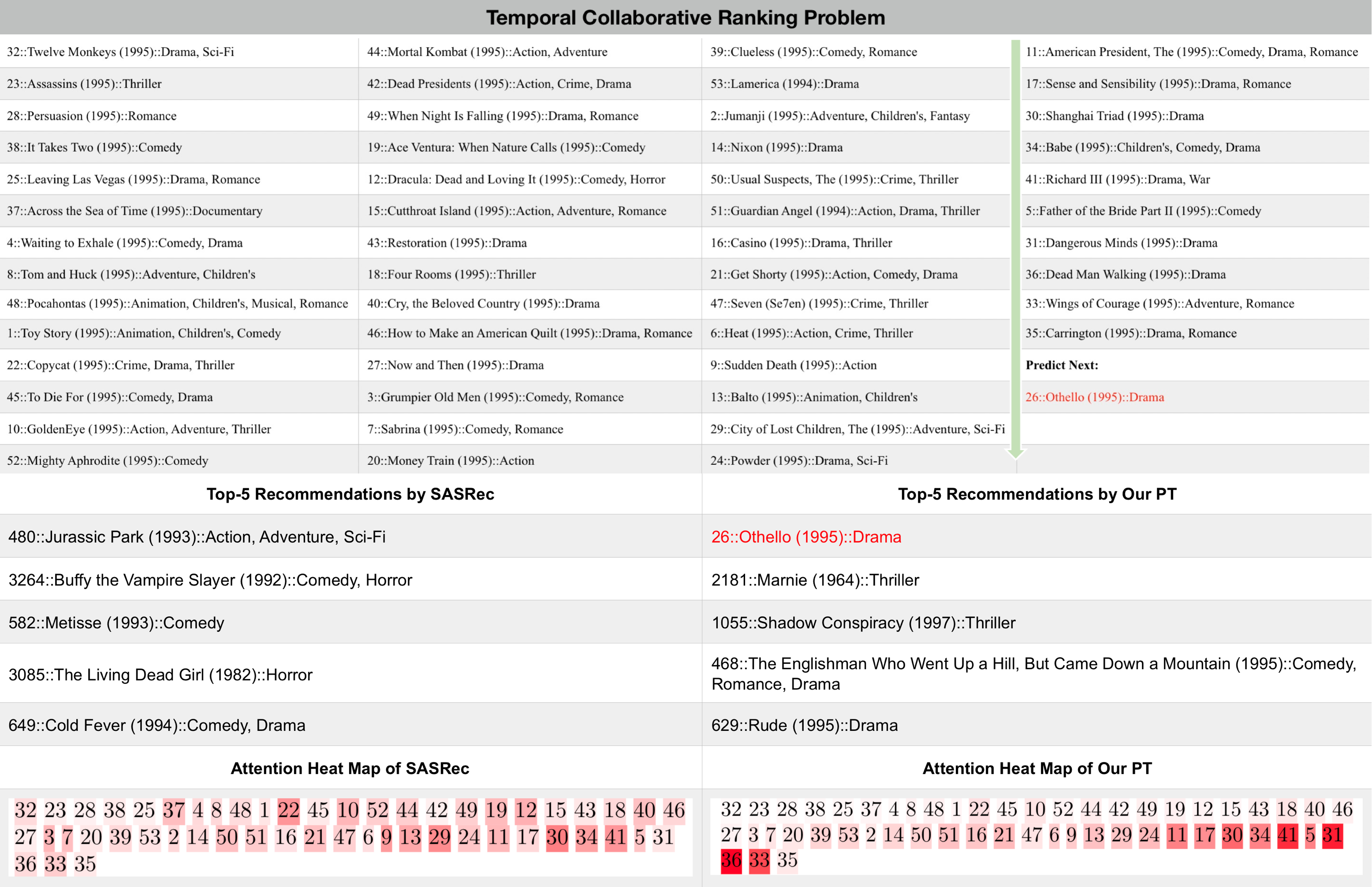}}
\end{center}
\caption{Illustration of how SASRec (Left) and SSE-PT (Right) differs on utilizing the Engagement History of A Random User in Movielens1M Dataset.}
\label{fig:example}
\end{figure*}

\subsection{Attention Mechanism Visualization}
 Apart from evaluating our SSE-PT against SASRec using well-defined ranking metrics on real-world datasets, we also visualize the differences between both methods in terms of their attention mechanisms.
In Figure~\ref{fig:example}, a random user's engagement history in Movielens1M dataset is given in temporal order (column-wise). We hide the last item whose index is 26 in test set and hope that a temporal collaborative ranking model can figure out item-26 is the one this user will watch next using only previous engagement history. One can see for a typical user; they tend to look at a different style of movies at different times. Earlier on, they watched a variety of movies, including Sci-Fi, animation, thriller, romance, horror, action, comedy and adventure. But later on, in the last two columns of Figure~\ref{fig:example}, drama and thriller are the two types they like to watch most, especially the drama type. In fact, they watched 9 drama movies out of recent 10 movies. For humans, it is natural to reason that the hidden movie should probably also be drama type. So what about the machine's reasoning?

For our SSE-PT, the hidden item indexed 26 is put in the first place among its top-5 recommendations. Intelligently, the SSE-PT recommends 3 drama movies, 2 thriller movies and mixing them up in positions.  Interestingly, the top recommendation is `Othello', which like the recently watched `Richard III', is an adaptation of a Shakespeare play, and this dependence is reflected in the attention weight.  On the contrast, SASRec cannot provide top-5 recommendations that are personalized enough. It recommends a variety of action, Sci-Fi, comedy, horror, and drama movies but none of them match item-26. Although this user has watched all these types of movies in the past, they do not watch these anymore as one can easily tell from his recent history. Unfortunately, SASRec cannot capture this and does not provide personalized recommendations for this user by focusing more on drama and thriller movies. 
It is easy to see that in contrast, our SSE-PT model shares with human reasoning that more emphasis should be placed on recent movies. 




\begin{figure*}[ht]
  \begin{center}
  \includegraphics[width=0.8\textwidth]{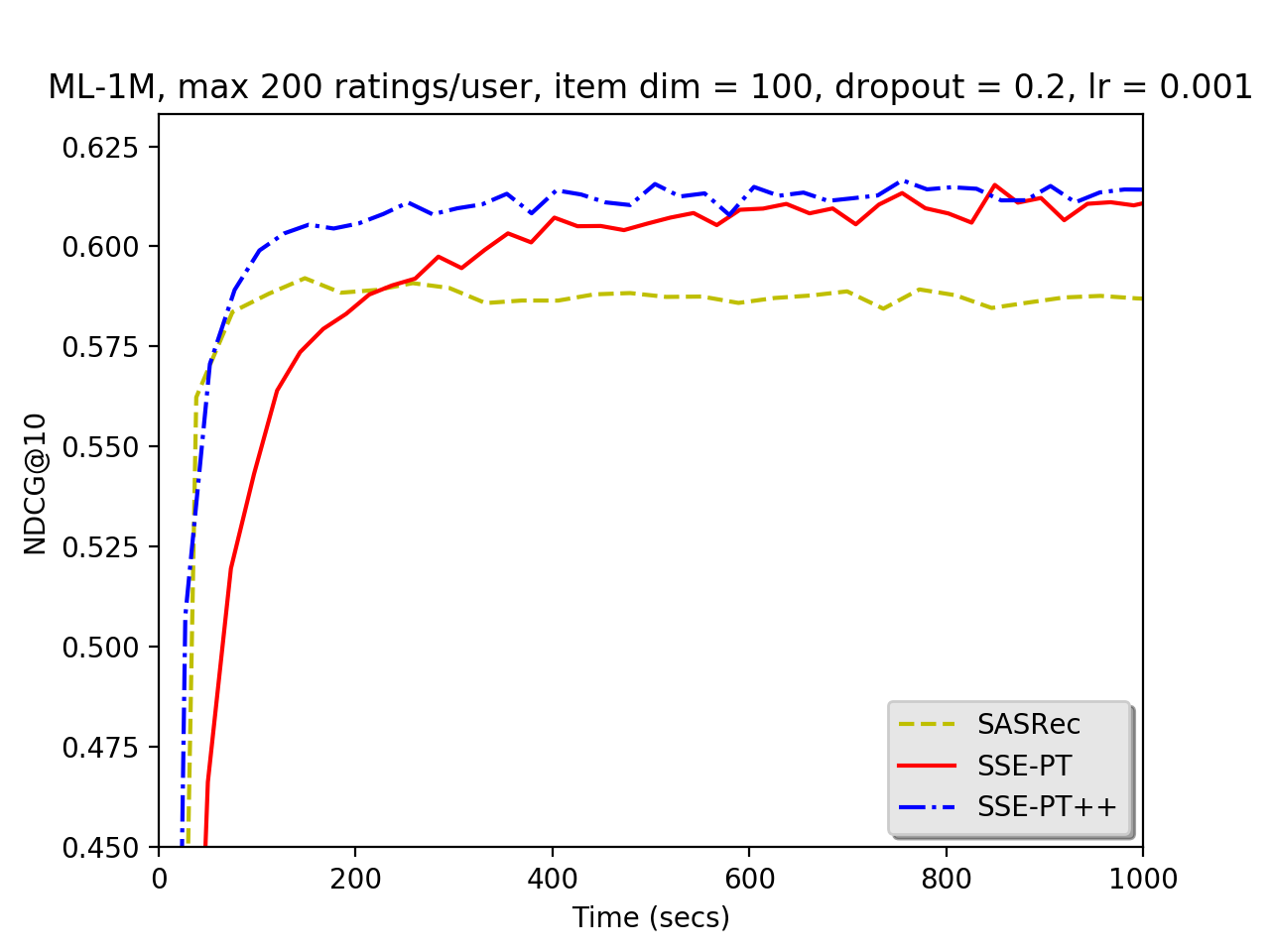}
  \end{center}
  \caption{Illustration of the speed of SSE-PT}
  \label{fig:speed6}
\end{figure*}

\subsection{Training Speed}
In \cite{kang2018self}, it has been shown that SASRec is about 11 times faster than Caser and 17 times faster than GRU4Rec$^+$ and achieves much better NDCG@10 results so we did not include Caser and GRU4Rec$^+$ in our comparisons. In Figure~\ref{fig:speed6}, we only compare the training speeds and ranking performances among SASRec, SSE-PT and SSE-PT++ for Movielens1M dataset.
Given that we added additional user embeddings into our SSE-PT model, it is expected that it will take slightly longer to train our model than un-personalized SASRec.
We find empirically that training speed of the SSE-PT and SSE-PT++ model are comparable to that of SASRec, with SSE-PT++ being the fastest and the best performing model. 
It is clear that our SSE-PT and SSE-PT++ achieve much better ranking performances than our baseline SASRec using the same training time.

\subsection{Ablation Study}

\paragraph{SSE probability} Given the importance of SSE regularization for our SSE-PT model, we carefully examined the SSE probability for input user embedding in Table~\ref{tb:ssep} in Appendix. We find that the appropriate hyper-parameter SSE probability is not very sensitive: anywhere between 0.4 and 1.0 gives good results, better than parameter sharing and not using SSE-SE. This is also evident based on comparison results in Table~\ref{tb:reg}.

\paragraph{Sampling Probability} Recall that the sampling probability is unique to our SSE-PT++ model. We show in Table~\ref{tb:sampling} in Appendix using an appropriate sampling probability like $0.2\to 0.3$ would allow it to outperform SSE-PT when the same maximum length is used.

\paragraph{Number of Attention Blocks} We find for our SSE-PT model, a larger number of attention blocks is preferred. One can easily see in Table~\ref{tb:block} in Appendix, the optimal ranking performances are achieved at $B = 4$ or $5$ for Movielens1M dataset and at $B = 6$ for Movielens10M dataset. 

\paragraph{Personalization and Number of Negatives Sampled}
Based on the results in Table~\ref{tb:negative} in Appendix, we are positive that the personalized model always outperforms the un-personalized one when we use the same regularization techniques. This holds true regardless of how many negatives sampled or what ranking metrics are used during evaluation.

\section{Conclusion}
In this chapter, we propose a novel neural network architecture called Personalized Transformer for the temporal collaborative ranking problem. It enjoys the benefits of being a personalized model, therefore achieving better ranking results for individual users than the current state-of-the-art. By examining the attention mechanisms during inference, the model is also more interpretable and tends to pay more attention to recent items in long sequences than un-personalized deep learning models.

\chapter{Conclusions}
\section{Summary of Contributions}
The goal of this thesis is to cover some recent advances in collaborative filtering and ranking research and demonstrate that there are various types of orthogonality in which one can contribute to the field.
During my PhD study, I was fortunate to explore many directions within collaborative filtering and ranking research. For the first 2 years of my research, I conducted foundational collaborative ranking research on improving the optimization procedure of the pairwise ranking loss in chapter 3 and designing new listwise loss objective functions in chapter 4 for collaborative ranking. For my last 2 years of PhD, I was exploring directions on incorporating additional side information such as graphs and temporal orderings. In chapter 5, we came up with a new graph encoding method in chapter 2 to enhance existing graph-based collaborative filtering, allowing them to encode deep graph information and therefore achieve better recommendation performances. We made the temporal collaborative ranking model personalized in chapter 7 by incorporating user embeddings. In the process, motivated by the need to prevent over-fitting caused by the additional parameters, we introduced a general regularization technique for embedding layers in deep learning in chapter 6, which was shown to be useful for many other models with lots of embedding parameters both within and outside recommendations. 

\section{Future Work}
Despite that the dissertation is lengthy and has contained many important research directions, it is by no means complete and I feel there are still many interesting and important directions worth pursuing but not done within this dissertation. I imagine at the end of day, assuming the computing power catches up, it is very likely we can discard most of the feature engineering and directly learn from raw data formats such as texts, images and videos, just like what happened in computer vision and natural language processing fields with the advances of deep learning techniques. This has not yet happened in neither academics or industry. To make this happen, I believe that there are a few shortcomings in current research practices:
\begin{itemize}
    \item There do not exist unified metrics to evaluate across papers. Some papers use root mean square errors (RMSE) as accuracy metric while others use ranking metrics. Then even within the ranking metrics, people tend to use different metrics: from AUC score to NDCG, precision and recall, not to mention that different top $k$ can be used for NDCG, precision and recall. 
    \item There do not exist unified datasets to test on across papers. There are many datasets outside there, from Netflix to Movielens, and Douban to Flixster, etc.. Different papers tend to use different training-test splits: some do random splits while others do splits based on temporal orderings. Moreover, the recommender systems community does not really have a shared and well-maintained leader-board on the winning solutions. The results of various proposed algorithms are very likely to perform differently on different datasets.
    \item Most datasets do not come with good features nor complete raw input data. To some extent, it is understandable because of the strict privacy and copyright concerns. But this has put academic researches at a disadvantage and often industry applied research will not take as much risk to pursue long-term projects. On the other hand, academics are able to take more risks to pursue longer-term research topics but unfortunately are limited by the constraints of good datasets and powerful computing resources.
\end{itemize}
While addressing these shortcomings, I think it would be very exciting to see different levels of interactions between the recommendation field and other AI fields, including natural language understanding and computer vision. News/Book recommendation, image recommendation, video recommendation and audio recommendation are some promising examples that may see breakthroughs of new models, just as what happened during Netflix competition about 10 years ago \cite{bennett2007netflix}. But to do that, we need a well-defined problem, dataset and metric, and lots of people participating both from academics and industry by combining strengths from both parties. 

\chapter{Appendix}
\section{Appendix to Chapter 2}

\begin{algorithm}[H]
\caption{A Standard Bloom Filter}
\label{alg:bloom-filter}
  \begin{compactitem}[leftmargin=*]
  \item[] {\bf class} {\tt BloomFilter:}
    \begin{compactitem}[leftmargin=*]
      \item[] {\bf def} $\mathtt{constructor}(\texttt{self}, c, \cbr{\mathtt{h_t}(\cdot): t = 1,\ldots,k})$:
        \begin{compactitem}[leftmargin=*]
          \item[] $\mathtt{self.b}[i] = 0\quad \forall i=1,\ldots,c$
          \item[] $\mathtt{self.h_{t} = h_{t}}\quad \forall i = 1,\ldots,k$
        \end{compactitem}
      \item[]
      \item[] {\bf def} $\mathtt{add(self, x)}$:
        \begin{compactitem}[leftmargin=*]
          \item[] $\mathtt{self.b}[\mathtt{self.h_t(x)}] = 1\quad\forall t = 1,\ldots,k$
        \end{compactitem}
      \item[]
      \item[] {\bf def} $\mathtt{union(self, bf)}$:
        \begin{compactitem}[leftmargin=*]
          \item[] $\mathtt{self.b}[i] \leftarrow \mathtt{self.b}[i] \mid \mathtt{bf}.\mathtt{b}[i]\quad \forall i=1,\ldots,c$
        \end{compactitem}
      \item[]
      \item[] {\bf def} $\mathtt{size(self)}$:
        \begin{compactitem}[leftmargin=*]
        \item[] {\bf return} $\ceil{-\frac{c}{k} \log\rbr{1 - \frac{\mathtt{nnz(self.b)}}{c}}}$
        \end{compactitem}
    \end{compactitem}
  \end{compactitem}
\end{algorithm}

\begin{figure*}
\begin{tabular}{cc}
\hspace{-8pt}
\includegraphics[width=0.5\linewidth]{figs/speed_douban_v1.png} &
\includegraphics[width=0.5\linewidth]{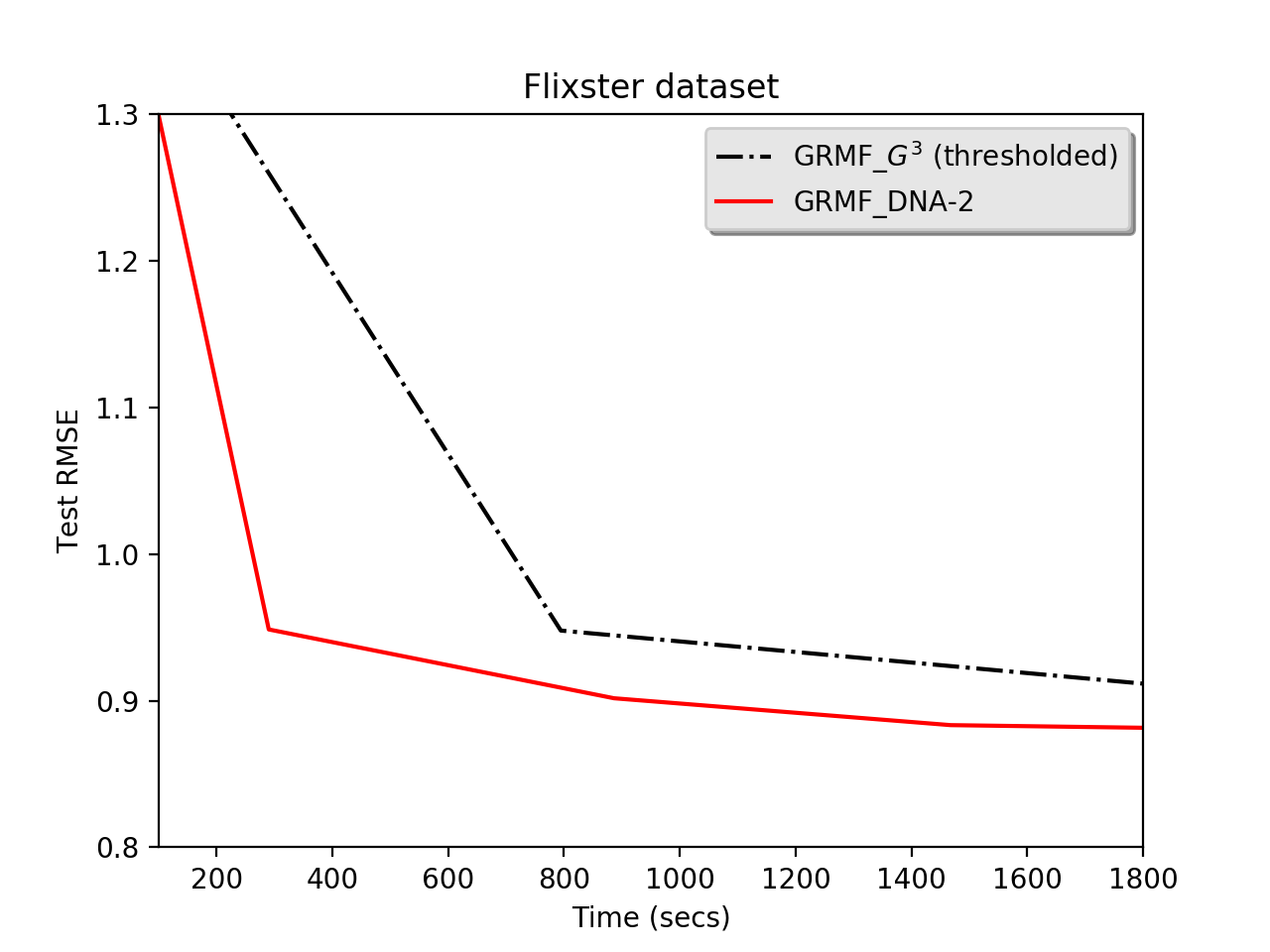} 
\end{tabular}
\caption{Compare Training Speed of GRMF, with and without Graph DNA. }
\label{fig:speed2}
\end{figure*}

\subsection{Simulation Study}
In the simulation we carried out, we set the number of users $n = 10,000$ and the number of items $m = 2,000$. We uniformly sample $5\%$ for training and $2\%$ for testing out of the total $nm$ ratings. We choose $T = 3$ so the graph contains at most $6$-hop information among $n$ users. We use rank $r = 50$ for both user and item embeddings. We set influence weight $w = 0.6$, i.e. in each propagation step, $60\%$ of one user's preference is decided by its friends (i.e. neighbors in the friendship graph). We set $p = 0.001$, which is the probability for each of the possible edges being chosen in Erd$\tilde{o}$s-R$\acute{e}$nyi graph $G$. A small edge probability $p$, influence weight $w < 1.0$, and a not too-large $T$ is needed, because we don't want that all users become more or less the same after $T$ propagation steps.

\subsection{Metrics}\label{sec:metric}

We omit the definitions of RMSE, Precision@$k$, NDCG@$k$, MAP as those can be easily found online.
HLU: Half-Life Utility \cite{breese1998empirical, shani2008mining} is defined as:
        \begin{equation}
            \text{HLU} = \frac{1}{n}\sum_{i = 1}^n \text{HLU}_i,
        \end{equation} where $n$ is the number of users and $\text{HLU}_i$ is given by:
        \begin{equation}
            \text{HLU}_i = \sum_{l = 1}^{k} \frac{ max(R_{i\Pi_{il}} - d, 0)}{2^{(j - 1) / (\alpha - 1)}},
        \end{equation}
        where $R_{i\Pi_{il}}$ follows previous definition, $d$ is the neural vote (usually the rating average), and $\alpha$ is the viewing halflife. The halflife is the number of the item on the list such that there is a $50$-$50$ chance the user will review that item~\cite{breese1998empirical}.

\begin{algorithm}
\caption{Simulation of Synthetic Data}
\label{alg:sim}
\begin{algorithmic}[1]
\Require $n$ users, $m$ items, rank $r$, influence weight $w$, $T$ propagation steps
\Ensure $R_{\text{tr}} \in \dR^{n \times m}$, $R_{\text{te}} \in\dR^{n \times m}$, $G \in\dR^{n \times n}$
\State Randomly initialize $U \in \dR^{n \times r}, V \in \dR^{m \times r}$ from standard normal distribution
\State Generate a random undirected Erd$\tilde{o}$s-R$\acute{e}$nyi graph $G$ with each edge being chosen with probability $p$
\For{$t = 1, ..., T$}
  \For{$i = 1, ..., n$}
      \State $\tilde{U}_i = w \cdot \sum_{j: (i, j) \in G}U_j + (1 - w) \cdot U_i$
  \EndFor
  \State Set $U = \tilde{U}$
\EndFor
\State Generate rating matrix $R = U V^T$
\State Random sample observed user/item indices in training and test data: $\Omega_{\text{tr}}, \Omega_{\text{te}}$
\State Obtain $R_{\text{tr}} = \Omega_{\text{tr}} \circ R, R_{\text{te}} = \Omega_{\text{te}} \circ R$
\State \textbf{return} rating matrices $R_{\text{tr}}, R_{\text{te}}$, user graph $G$
\end{algorithmic}
\end{algorithm}

\begin{table*}
  \caption{Compare Bloom filters of different depths and sizes an on Synthesis Dataset. Note that the number of bits of Bloom filter is decided by Bloom filter's maximum capacity and tolerable error rate (i.e. false positive error, we use $0.2$ as default).}
  \label{tab:sim_bits}
  \resizebox{0.8\textwidth}{!}{
  \begin{tabular}{lcccccccc}
    \toprule
    methods & max capacity & $c$ bits & nnz ratio & RMSE ($ \times 10^{-3}$) & \% Relative Graph Gain  \\
    \midrule
    GRMF\_$G^2$ & -  & - & -   & 2.6543  & 59.5903\\
    GRMF\_DNA-1 & 20 & 135 & 0.217  & 2.4303  & 163.8734 \\
    GRMF\_DNA-1 & 50 & 336 & 0.093  &  2.4795 &  140.9683 \\
    GRMF\_DNA-2 & 20 & 135 & 0.880 & 2.4921  & 135.1024 \\
    GRMF\_DNA-2 & 50 & 336  & 0.608  & 2.4937  & 134.3575 \\
    GRMF\_DNA-2 & 100 & 672  & 0.381  & 2.4510  & 154.2365 \\
    GRMF\_DNA-2 & 200 & 1,341 & 0.215  & 2.4541  & 152.7933  \\
    GRMF\_DNA-3 & 200 & 1,341 & 0.874  & 2.4667  & 146.9274 \\
    GRMF\_DNA-3 & 600 & 4,020 & 0.525  & 2.4572  & 151.3500 \\
    GRMF\_DNA-3 & 1,000 & 6,702 & 0.364  & 2.4392  & 159.7299 \\
    GRMF\_DNA-3 & \bfseries{1,500} & \bfseries{10,050}& \bfseries{0.262}  & \bfseries{2.4247}  & \bfseries{166.4804} \\
    GRMF\_DNA-4 & 2,000 & 13,401& 0.743 & 2.5532 &  106.6573 \\
    GRMF\_DNA-4 & 4,000 & 26,799& 0.499  & 2.4466 & 156.2849 \\
  \bottomrule
\end{tabular}
}
\end{table*}

\begin{table*}
    \caption{Compare nnz of different methods on Douban and Flixster datasets. GRMF\_$G^4$ and GRMF\_DNA-2 are using the same $4$-hop information in the graph but in different ways. Note that we do not exclude potential overlapping among columns.}
    \vskip -0.1in
    \label{tab:nnz}
     \resizebox{1.0\textwidth}{!}{
    \begin{tabular}{lllccccccr}
    \toprule
    Dataset & methods & $R_{\text{tr}}$ & $G$ & $G^2$ & $G^3$ & $G^4$ & $B$ & total nnz \\
    \midrule
    \multirow{8}{*}{Douban}
    &MF   & 9,803,098 &  -  & -   & -     & - & -    & 9,803,098  \\
    &GRMF\_$G$   & 9,803,098 & 1,711,780 & -   & -  & -   & -     & 11,514,878  \\
    &GRMF\_$G^2$  & 9,803,098 & 1,711,780 & 106,767,776   & -  & -   & -     & 118,282,654  \\
    &GRMF\_$G^3$ & 9,803,098 & 1,711,780 & 106,767,776   & 2,313,572,544     & -   & -  & 2,431,855,198  \\
    &GRMF\_$G^4$ & \bfseries{9,803,098} & \bfseries{1,711,780} & \bfseries{106,767,776}   & \bfseries{2,313,572,544}     & \bfseries{8,720,553,105}    & - &  \bfseries{11,152,408,303} \\
    &GRMF\_DNA-1  & 9,803,098 &  0  & -   & -   & -  &    8,834,740  &  18,637,838 \\
    &GRMF\_DNA-2  & 9,803,098 &  1,711,780  & -   & -   & -  &   142,897,900   &  154,412,778 \\
    &GRMF\_DNA-3  & 9,803,098 &  1,711,780  & -   & -  & -   &   928,159,604   &  939,674,482 \\
    \midrule
    \multirow{8}{*}{Flixster}
    &MF   & 3,619,304 &  -  & -   & -     & -  & -   &  3,619,304 \\
    &GRMF\_$G$   & 3,619,304 & 2,538,746 & -   & -     & -  & -   &  6,158,050 \\
    &GRMF\_$G^2$  & 3,619,304 &2,538,746 & 130,303,379   & -     & -  & -   & 136,461,429  \\
    &GRMF\_$G^3$ & 3,619,304 & 2,538,746 &130,303,379   & 2,793,542,551     & -  & -   & 3,060,307,359   \\
    &GRMF\_$G^4$ & \bfseries{3,619,304} & \bfseries{2,538,746} & \bfseries{130,303,379}   & \bfseries{2,793,542,551}      & \bfseries{12,691,844,513}   & - &  \bfseries{15,752,151,872} \\
    &GRMF\_DNA-1  & 3,619,304 &  0  & -   & -   & -  &  12,664,952    &  16,284,256 \\
    &GRMF\_DNA-2  & 3,619,304 &  2,538,746  & -   & -   & -  &   181,892,883   &  188,050,933 \\
    &GRMF\_DNA-3   & 3,619,304 &  2,538,746  & -   & -  & -   &   1,185,535,529   & 1,191,693,579  \\
    \bottomrule
\end{tabular}
}
\end{table*}

\subsection{Graph Regularized Weighted Matrix Factorization for Implicit feedback}

We use the rank $r = 10$, negatives' weight $\rho = 0.01$ and measure the prediction performance with metrics MAP, HLU, Precision@$k$ and NDCG$@k$ (see definitions of metrics in Appendix~\ref{sec:metric}).

We follow the similar procedure to what is done before in GRMF and co-factor: we run all combinations of tuning parameters of $\lambda_l \in \{0.01, 0.1, 1, 10, 100\}$ and $\lambda_g \in \{0.01, 0.1, 1, 10, 100\}$ for each method on validation data for fixed number $40$ epochs and choose the best combination as the parameters to use on test data. We then report the best prediction results during first $40$ epochs on test data with the chosen parameter combination.

\subsection{Reproducibility}
\label{sec:reproduce}
To reproduce results reported in the paper, one need to download data (douban and flixster) and third-party C++ Matrix Factorization library from the link \url{https://www.csie.ntu.edu.tw/~cjlin/papers/ocmf-side/}. One can simply follow README there to compile the codes in Matlab and run one-class matrix factorization library in different modes (both explicit feedback and implicit feedback works). The advantage of using this library is that the codes support multi-threading and runs quite fast with very efficient memory space allocations. It also supports with graph or other side information. All three methods' baseline can be simply run with the tuning parameters we reported in the Table~\ref{tab:repro-2}, ~\ref{tab:repro-3}, ~\ref{tab:repro-4} in Appendix.

To reproduce results of our DNA methods, one need to generate Bloom filter matrix $B$ following Algorithm~\ref{alg:graph-bloom}. 
We will provide our python codes implementing Algorithm~\ref{alg:graph-bloom} and Matlab codes converting into the formats the library requires.

For baselines and our DNA methods, We perform a parameter sweep for $\lambda_l \in \{0.01, 0.1, 1, 10, 100\}$, $\lambda_g \in \{0.01, 0.1, 1, 10, 100\}$, $\alpha \in \{0.0001, 0.001, 0.01, 0.1, 0.3, 0.7, 1\}$, for $\beta \in \{0.005, 0.01, 0.03, 0.05, 0.1 \}$ when needed. We run all combinations of tuning parameters for each method on validation set for $40$ epochs and choose the best combination as the parameters to use on test data. We then report the best test RMSE in first $40$ epochs on test data with the chosen parameter combination. 
We provide all the chosen combinations of tuning parameters that achieves reported optimal results in results tables in the Table~\ref{tab:repro-2}, ~\ref{tab:repro-3}, ~\ref{tab:repro-4} in Appendix. 
One just need to exactly follow our procedures in Section~\ref{sec:app} to construct new $\dot{G}, \dot{U}$ to replace the $G, U$ in baseline methods before feeding into Matlab.

As to simulation study, we will also provide python codes to repeat our Algorithm~\ref{alg:sim} to generate synthesis dataset. One can easily simulate the data before converting into Matlab data format and running the codes as before. The optimal parameters can be found in Table~\ref{tab:repro-1}. For all the methods, we select the best parameters $\lambda_l$ and $\lambda_g$ from $\{0.01, 0.1, 1, 10, 100\}$. For method GRMF\_$G^2$, we  tune an additional parameter $\alpha\in \{0.0001, 0.001, 0.01, 0.1, 0.3, 0.7, 1\}$.
For the thrid-order method 
GRMF\_$G^3$, we tune $\beta\in \{0.005, 0.01, 0.03, 0.05, 0.1 \}$ in addition to $\lambda_l, \lambda_G, \alpha$.
Due to the speed constraint, we are not able to tune a broader range of choices for $\alpha$ and $\beta$ as it is too time-consuming to do so especially for douban and flixster datasets. For example, it takes takes about 3 weeks using 16-cores CPU to tune both $\alpha, \beta$ on flixster dataset. We run each method with every possible parameter combination for fixed $80$ epochs on the same training data, tune the best parameter combination based on a small predefined validation data and report the best RMSE results on test data with the best tuning parameters during the first $80$ epochs. Note that only on the small synthesis dataset, we calculate full $G^3$ and report the results. On real datasets, there is no way to calculate full $G^4$ to utilize the complete $4$-hop information, because one can easily spot in Table~\ref{tab:nnz} the number of non-zero elements (nnz) is growing exponentially when the hop increases by $1$, which makes it impossible for one to utilize complete $3$-hop and $4$-hop information.

In Table~\ref{tab:repro-2}, one can compare magnitude of optimal $\alpha$ and $\beta$ to have a good idea of whether $G$ or $G^2$ is more useful. $G$ represents shallow graph information and $G^2$ represents deep graph information. If one already run GRMF\_$G^2$, one can then use this as a preliminary test to decide whether to go deep with
DNA-3 ($d = 3$) to capture deep graph information or simply go ahead with DNA-1 ($d = 1$) to fully utilize shallow information. 
For douban dataset, we have $\alpha = 0.05 >  0.0005 = \beta$, which implies shallow information is important and we should fully utilize it. It explains why DNA-1 is performing well both in terms of performance and speed on douban dataset. It is worth noting that 
GRMF\_DNA-1's Bloom filter matrix $B$ contains much more nnz than that of $G$ in Table~\ref{tab:nnz} though $20\%$ less than that of $G^2$. On the other hand, for flixster dataset, we have $\alpha = 0.01 <  0.1 = \beta$, which implies in this dataset deeper information is more important and we should go deeper. That explains why here 
GRMF\_DNA-3 ($6$-hop) achieves about $10$ times more gain than using $1$-hop 
GRMF\_$G$.  

\begin{table*}
  \caption{Compare Matrix Factorization for Explicit Feedback on Synthesis Dataset. The synthesis dataset has $10,000$ users and $2,000$ items with user friendship graph of size $10,000 \times 10,000$. Note that the graph only contains at most $6$-hop valid information. GRMF\_$G^6$ means GRMF with $G + \alpha \cdot G^2 + \beta \cdot G^3 + \gamma \cdot G^4 + \epsilon \cdot G^5 + \omega \cdot G^6.$ GRMF\_DNA-$d$ means depth $d$ is used.}
  \label{tab:repro-1}
  \resizebox{0.8\textwidth}{!}{
  \begin{tabular}{lccccccccc}
    \toprule
    methods & test RMSE ($ \times 10^{-3}$) & $\lambda_l$ & $\lambda_g $ & $\alpha$ & $\beta$ & $\gamma$ & $\epsilon$ &$\omega$ &\% gain over baseline \\
    \midrule
    MF                  & 2.9971 & 0.01  & -    & -     & - & - & -  & - & -  \\
    GRMF\_$G$               & 2.7823 & 0.01  & 0.01 & -  & -  & - & -  & - & 7.16693\\
    GRMF\_$G^2$  & 2.6543 & 0.01  & 0.01 & 0.3   & - & - & - &-  & 11.43772\\
    GRMF\_$G^3$ & 2.5687 & 0.01 & 0.01 & 0.01 & 0.05 & - &- &  -  &14.29382\\
    GRMF\_$G^4$ & 2.5562 & 0.01 & 0.01 & 0.01 & 0.05 & 0.1& - &-  & 14.71088\\
    GRMF\_$G^5$  & 2.4853 & 0.01 & 0.01 & 0.01 & 0.05 & 0.1 & 0.1&-  & 17.07651\\
    GRMF\_$G^6$  & 2.4852 & 0.01 & 0.01 & 0.01 & 0.05 & 0.1 & 0.1& 0.01 &17.07984\\
    GRMF\_DNA-1   & 2.4303 & 0.01 & 0.01  & - & - & - & - & -  &18.91161 \\
    GRMF\_DNA-2  & 2.4510 & 0.01 & 0.01  & - & - & - & - & -  &18.22095 \\
    GRMF\_DNA-3  & \bfseries{2.4247} & \bfseries{0.01} & \bfseries{0.01}  & - & - & - &- & -  &\bfseries{19.09846}\\
    GRMF\_DNA-4  & 2.4466 & 0.01 & 0.01  & - & - & - & - & -  &18.36776\\

  \bottomrule
\end{tabular}
}
\end{table*}

\begin{table*}
  \caption{Compare Matrix Factorization methods for Explicit Feedback on Douban and Flixster data. We use rank $r = 10$.} 
  \label{tab:repro-2}
  \resizebox{0.8\textwidth}{!}{
  \begin{tabular}{lllcccccc}
    \toprule
    Dataset & methods & test RMSE ($ \times 10^{-1}$) & $\lambda_l$ & $\lambda_g $ & $\alpha$ & $\beta$ & \% gain over baseline \\
    \midrule
    \multirow{8}{*}{Douban}
    &MF              & 7.3107 & 1   & -   & -     & -  & -     \\
    &GRMF\_$G$                 & 7.2398 & 0.1 & 100 & -     & -  & 0.9698\\
    &GRMF\_$G^2$ & 7.2381 & 0.1 & 100 & 0.001 & -  & 0.9930\\
    &GRMF\_$G^3$ (full)  & 7.2432  & 0.1 & 100 & 0.05 & 0.0005 & 0.9350\\
    &GRMF\_$G^3$ (thresholded) & 7.2382 & 0.1 & 100 & 0.05 & 0.0005 & 0.9917\\
    &GRMF\_DNA-1    & 7.2191 & 0.1 & 100 & - & - & 1.2689 \\
    &GRMF\_DNA-2   & 7.2359 & 1 & 10 & - & - & 1.0232 \\
    &GRMF\_DNA-3  & \bfseries{7.2095} & \bfseries{0.01} & \bfseries{100} & - & - & \bfseries{1.3843}\\
    \midrule
    \multirow{8}{*}{Flixster}
    &MF                & 8.8111 & 0.1 & 1   & -    & -     & -      \\
    &GRMF\_$G$               & 8.8049 & 0.01 & 1    & -     & -  & 0.0704\\
    &GRMF\_$G^2$ & 8.7849 & 0.01 & 1  & 0.05 & -  & 0.2974\\
    &GRMF\_$G^3$  (full) & 8.7932 & 0.1 & 1 & 0.01 & 0.1 & 0.2032\\
    &GRMF\_$G^3$  (thresholded) & 8.7920 & 0.01 & 1 & 0.01 & 0.1 & 0.2168\\
    &GRMF\_DNA-1    & 8.8013  & 0.01 & 1   & - & - & 0.1112  \\
    &GRMF\_DNA-2   & 8.8007 & 0.1 & 1  & - & - & 0.1180 \\
    &GRMF\_DNA-3 & \bfseries{8.7453} & \bfseries{0.1} & \bfseries{100} & - & - & \bfseries{0.7468}\\
  \bottomrule
\end{tabular}
}
\end{table*}

\begin{table*}
  \caption{Compare Co-factor Methods for Explicit Feedback on Douban and Flixster Datasets. We use rank $r = 10$ for both methods.}
 \label{tab:repro-3}
  \resizebox{0.6\textwidth}{!}{
  \begin{tabular}{llcccc}
    \toprule
    Dataset & methods & test RMSE ($ \times 10^{-1}$) & $\lambda_l$  & \% gain over baseline  \\
    \midrule
    \multirow{2}{*}{Douban}
    & co-factor\_$G$                     & 7.2743 & 1   & -     \\
    & co-factor\_DNA-$3$ & \bfseries{7.2674} & \bfseries{1}   & \bfseries{0.5923}\\
    \midrule
    \multirow{2}{*}{Flixster}
    & co-factor\_$G$                  & 8.7957 & 0.01   & -     \\
    & co-factor\_DNA-$3$ & \bfseries{8.7354} & \bfseries{0.01}   & \bfseries{0.8591}\\
  \bottomrule
\end{tabular}
}
\end{table*}

\begin{table*}
  \caption{Compare Weighted Matrix Factorization with Graph for Implicit Feedback on Douban and Flixster Datasets. We use rank $r = 10$ for both methods and all metric results are in $\%$.}
  \vskip -0.1in
  \label{tab:repro-4}
  \resizebox{0.7\textwidth}{!}{
  \begin{tabular}{llccccccccc}
    \toprule
    Dataset & Methods & MAP & HLU & P@$1$ & P@$5$ & NDCG@$1$ & NDCG@$5$ & $\lambda_l$ & $\lambda_g $ \\
    \midrule
    \multirow{2}{*}{Douban}
    &WMF\_$G$               & 8.340 & 13.033 & 14.944 & 10.371 & 14.944 & 12.564 & 0.01 & 10    \\
    &WMF\_DNA-3 & \bfseries{8.400} & \bfseries{13.110} & \bfseries{14.991} & \bfseries{10.397} & \bfseries{14.991} & \bfseries{12.619} & \bfseries{1}    & \bfseries{1}     \\
     \midrule
    \multirow{2}{*}{Flixster}
    &WMF\_$G$              & 10.889 & 14.909 & 12.303 & 7.9927 & 12.303 & 12.734 & 10 & 0.1    \\
     &WMF\_DNA-3 & \bfseries{11.612} & \bfseries{15.687} & \bfseries{12.644} & \bfseries{8.1583} & \bfseries{12.644} & \bfseries{13.399} & \bfseries{1}    & \bfseries{1}     \\

  \bottomrule
\end{tabular}
}
\end{table*}

\subsection{Code}\label{sec:code}
Part of our code is already made available on Github: \url{https://github.com/wuliwei9278/Graph-DNA}.

\section{Appendix to Chapter 4}
We include pseudo-codes for Algorithm~\ref{alg:gradv}, \ref{alg:update} and Figures~\ref{time2}, ~\ref{sqp}, ~\ref{explicit_topk} in the appendix to chapter 4.

\begin{algorithm}[tb]
  \caption{Compute gradient for $V$ when $U$ fixed}
  \label{alg:gradv}
\begin{algorithmic}
  \State {\bfseries Input:} $\Pi$, $U$, $V$, $\lambda, \rho$
  \State {\bfseries Output:} $g$ \Comment{$g\in \R^{r \times m}$ is the gradient for $f(V)$}
    \State $g = \lambda \cdot V$
  \For{$i = 1$ {\bfseries to} $n$}
     \State Precompute $h_t = u_i^T v_{\Pi_{it}}$ for $1 \leq t \leq \bar{m}$ \Comment{For implicit feedback, it should be $(1 + \rho) \cdot \tilde{m}$ instead of $\tilde{m}$, since $\rho \cdot \tilde{m}$ $0$'s are appended to the back}
     \State Initialize $total = 0$, $tt = 0$ 
     \For{$t = \bar{m}$ {\bfseries to} $1$}
        \State $total \mathrel{+}= \exp(h_t)$
        \State $tt \mathrel{+}= 1/total$
     \EndFor
     \State Initialize $c[t] = 0$ for $1 \leq t \leq \bar{m}$ 
     \For{$t = \bar{m}$ {\bfseries to} $1$}
        \State $c[t] \mathrel{+}= h_t \cdot (1 - h_t) $
        \State $c[t] \mathrel{+}= \exp(h_t) \cdot h_t \cdot (1 - h_t) \cdot tt$
        \State $total \mathrel{+}= \exp(h_t)$
        \State $tt \mathrel{-}= 1 / total$
     \EndFor
     \For{$t = 1$ {\bfseries to} $\bar{m}$}
        \State $g[:, \Pi_{it}] \mathrel{+}= c[t] \cdot u_i$
    \EndFor
  \EndFor 
  \State {\bfseries Return} $g$
\end{algorithmic}
\end{algorithm}

\begin{algorithm}[tb]
  \caption{Gradient update for $V$ (Same procedure for updating $U$)}
  \label{alg:update}
\begin{algorithmic}
    \State {\bfseries Input:} $V, ss$, $rate$ \Comment{$rate$ refers to the decaying rate of the step size $ss$}
    \State {\bfseries Output:} $V$
    \State Compute gradient $g$ for $V$ \Comment{see alg~\ref{alg:gradv}}
    \State $V \mathrel{-}= ss \cdot g$
    \State $ss \mathrel{*}= rate$
    \State {\bfseries Return} $V$
\end{algorithmic}
\end{algorithm}

\begin{figure}[ht]
\begin{center}
\centerline{\includegraphics[width=\columnwidth]{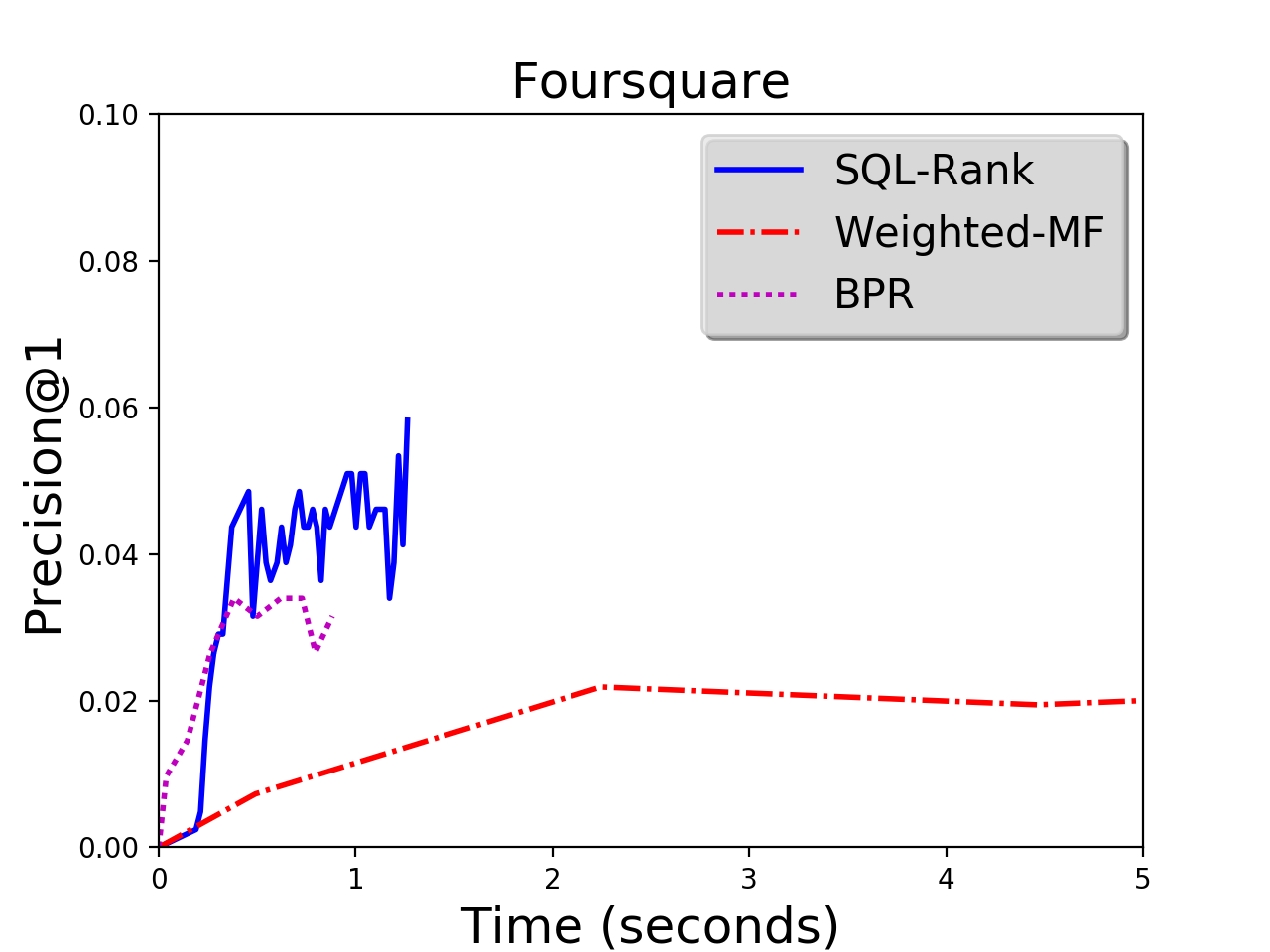}}
\caption{Comparing implicit feedback methods.}
\label{time2}
\end{center}
\end{figure}

\begin{figure}[ht]
\vskip 0.2in
\begin{center}
\centerline{\includegraphics[width=\columnwidth]{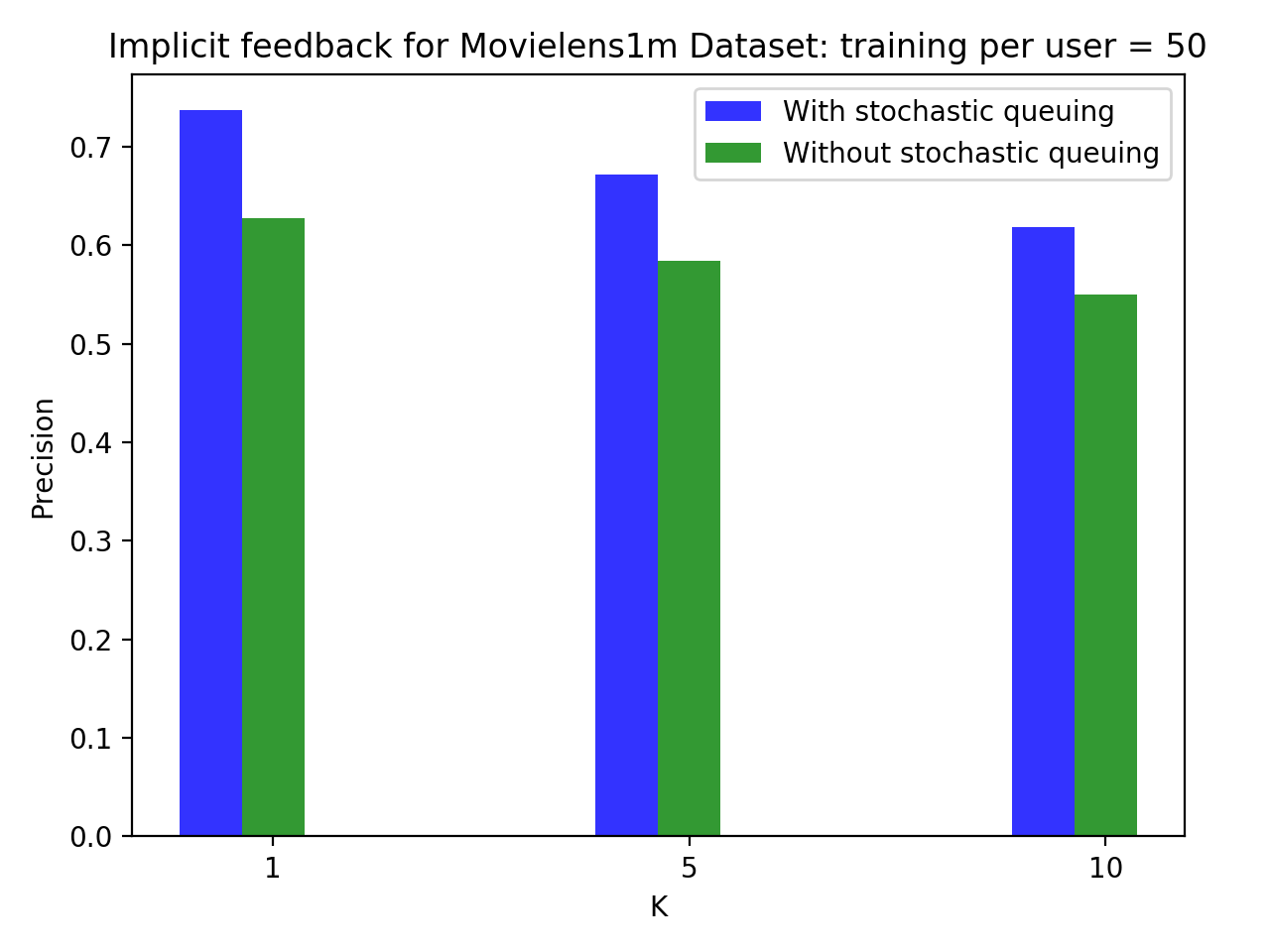}}
\caption{Effectiveness of Stochastic Queuing Process.}
\label{sqp}
\end{center}
\vskip -0.2in
\end{figure}

\begin{figure}[ht]
\vskip 0.2in
\begin{center}
\centerline{\includegraphics[width=\columnwidth]{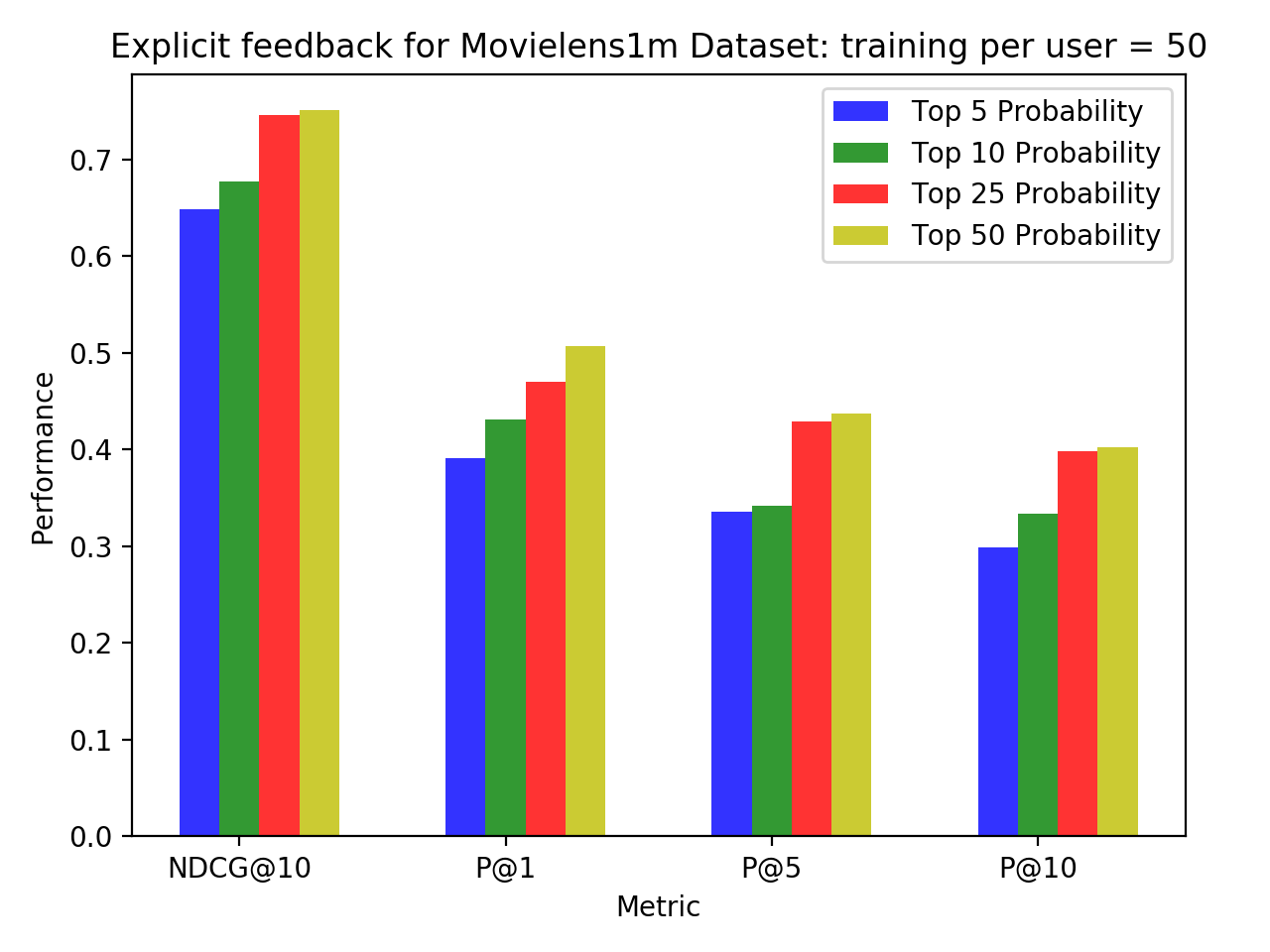}}
\caption{Effectiveness of using full lists.}
\label{explicit_topk}
\end{center}
\vskip -0.2in
\end{figure}

\section{Appendix to Chapter 5}
For experiments in Section~\ref{sec:simple_nn}, we use Julia and C++ to implement SGD. For experiments in Section~\ref{sec:sasrec}, and Section~\ref{sec: bert}, we use Tensorflow and SGD/Adam Optimizer. For experiments in Section~\ref{sec:nmt}, we use Pytorch and Adam with noam decay scheme and warm-up. We find that none of these choices affect the strong empirical results supporting the effectiveness of our proposed methods, especially the SSE-SE. In any deep learning frameworks, we can introduce stochasticity to the original embedding look-up behaviors and easily implement SSE-Layer in Figure~\ref{fig:train_test} as a custom operator. 

\subsection{Neural Networks with One Hidden Layer}
\label{app:mf}
To run SSE-Graph, we need to construct good-quality knowledge graphs on embeddings. We managed to match movies in Movielens1m and Movielens10m datasets to IMDB websites, therefore we can extract plentiful information for each movie, such as the cast of the movies, user reviews and so on. For simplicity reason, we construct the knowledge graph on item-side embeddings using the cast of movies. Two items are connected by an edge when they share one or more actors/actresses. For user side, we do not have good quality graphs: we are only able to create a graph on users in Movielens1m dataset based on their age groups but we do not have any side information on users in Movielens10m dataset. When running experiments, we do a parameter sweep for weight decay parameter and then fix it before tuning the parameters for SSE-Graph and SSE-SE. We utilize different $\rho$ and $p$ for user and item embedding tables respectively. The optimal parameters are stated in Table~\ref{tb:sse-se_sse-graph2} and Table~\ref{tb:mf}. We use the learning rate of 0.01 in all SGD experiments.

In the first leg of experiments, we examine users with fewer than 60 ratings in Movielens1m and Movielens10m datasets. In this scenario, the graph should carry higher importance. One can easily see from Table~\ref{tb:sse-se_sse-graph2} that without using graph information, our proposed SSE-SE is the best performing matrix factorization algorithms among all methods, including popular ALS-MF and SGD-MF in terms of RMSE. With Graph information, our proposed SSE-Graph is performing significantly better than the Graph Laplacian Regularized Matrix Factorization method. This indicates that our SSE-Graph has great potentials over Graph Laplacian Regularization as we do not explicitly penalize the distances across embeddings but rather we implicitly penalize the effects of similar embeddings on the loss.

In the second leg of experiments, we remove the constraints on the maximum number of ratings per user. We want to show that SSE-SE can be a good alternative when graph information is not available. We follow the same procedures in \cite{wu2017large, wu2018sql}. In Table~\ref{tb:mf}, we can see that SSE-SE can be used with dropout to achieve the smallest RMSE across Douban, Movielens10m, and Netflix datasets. In Table~\ref{tb:bpr}, one can see that SSE-SE is more effective than dropout in this case and can perform better than STOA listwise approach SQL-Rank \cite{wu2018sql} on 2 datasets out of 3. 

In Table~\ref{tb:mf}, SSE-SE has two tuning parameters: probability $p_u$ to replace embeddings associated with user-side embeddings and probability $p_i$ to replace embeddings associated with item side embeddings because there are two embedding tables. But here for simplicity, we use one tuning parameter $p_{s} = p_u = p_i$. We use dropout probability of $p_{d}$, dimension of user/item embeddings $d$, weight decay of $\lambda$ and learning rate of $0.01$ for all experiments, with the exception that the learning rate is reduced to $0.005$ when both SSE-SE and Dropout are applied. For Douban dataset, we use $d = 10, \lambda = 0.08$. For Movielens10m and Netflix dataset, we use $d = 50, \lambda = 0.1$.

\subsection{Neural Machine Translation}
We use the transformer model \cite{vaswani2017attention} as the backbone for our experiments. The control group is the standard transformer encoder-decoder architecture with self-attention. In the experiment group, we apply SSE-SE towards both encoder and decoder by replacing corresponding vocabularies' embeddings in the source and target sentences. We trained on the standard WMT 2014 English to German dataset which consists of roughly 4.5 million parallel sentence pairs and tested on WMT 2008 to 2018 news-test sets. Sentences were encoded into 32,000 tokens using a byte-pair encoding. We use the SentencePiece, OpenNMT and SacreBLEU implementations in our experiments. We trained the 6-layer transformer base model on a single machine with 4 NVIDIA V100 GPUs for 20,000 steps. We use the same dropout rate of 0.1 and label smoothing value of 0.1 for the baseline model and our SSE-enhanced model. Both models have dimensionality of embeddings as $d = 512$. When decoding, we use beam search with the beam size of 4 and length penalty of 0.6 and replace unknown words using attention. For both models, we average last 5 checkpoints (we save checkpoints every 10,000 steps) and evaluate the model's performances on the test datasets using BLEU scores. The only difference between the two models is whether or not we use our proposed SSE-SE with $p = 0.01$ in Equation~\ref{eq:sse-se} for both encoder and decoder embedding layers.

\subsection{BERT}
In the first leg of experiments, we crawled one million user reviews data from IMDB and pre-trained the BERT-Base model (12 blocks) for $500,000$ steps using sequences of maximum length 512 and batch size of 8, learning rates of $2e^{-5}$ for both models using one NVIDIA V100 GPU. Then we pre-trained on a mixture of our crawled reviews and reviews in IMDB sentiment classification tasks (250K reviews in train and 250K reviews in test) for another $200,000$ steps before training for another $100,000$ steps for the reviews in IMDB sentiment classification task only. In total, both models are pre-trained on the same datasets for $800,000$ steps with the only difference being our model utilizes SSE-SE. In the second leg of experiments, we fine-tuned the two models obtained in the first-leg experiments on two sentiment classification tasks: IMDB sentiment classification task and SST-2 sentiment classification task. The goal of pre-training on IMDB dataset but fine-tuning for SST-2 task is to explore whether SSE-SE can play a role in transfer learning.

The results are summarized in Table~\ref{tb:bert-imdb} for IMDB sentiment task. In experiments, we use maximum sequence length of 512, learning rate of $2e^{-5}$, dropout probability of $0.1$ and we run fine-tuning for 1 epoch for the two pre-trained models we obtained before. For the Google pre-trained BERT-base model, we find that we need to run a minimum of 2 epochs. This shows that pre-training can speed up the fine-tuning. We find that Google pre-trained model performs worst in accuracy because it was only pre-trained on Wikipedia and books corpus while ours have seen many additional user reviews. We also find that SSE-SE pre-trained model can achieve accuracy of 0.9542 after fine-tuning for one epoch only. On the contrast, the accuracy is only 0.9518 without SSE-SE for embeddings associated with output $y_i$. 

For the SST-2 task, we use maximum sequence length of 128, learning rate of $2e^{-5}$, dropout probability of 0.1 and we run fine-tuning for 3 epochs for 
all 3 models in Table~\ref{tb:bert-sst2}. We report AUC, accuracy and F1 score for dev data. For test results, we submitted our predictions to Glue website for the official evaluation. We find that even in transfer learning, our SSE-SE pre-trained model still enjoys advantages over Google pre-trained model and our pre-trained model without SSE-SE. Our SSE-SE pre-trained model achieves 94.3\% accuracy on SST-2 test set versus 93.6 and 93.8 respectively. If we are using SSE-SE for both pre-training and fine-tuning, we can achieve 94.5\% accuracy on the SST-2 test set, which approaches the 94.9 score reported by the BERT-Large model. SSE probability of 0.01 is used for fine-tuning.


\subsection{Proofs}

\begin{figure}
  \centering
  \includegraphics[width=0.835\linewidth]{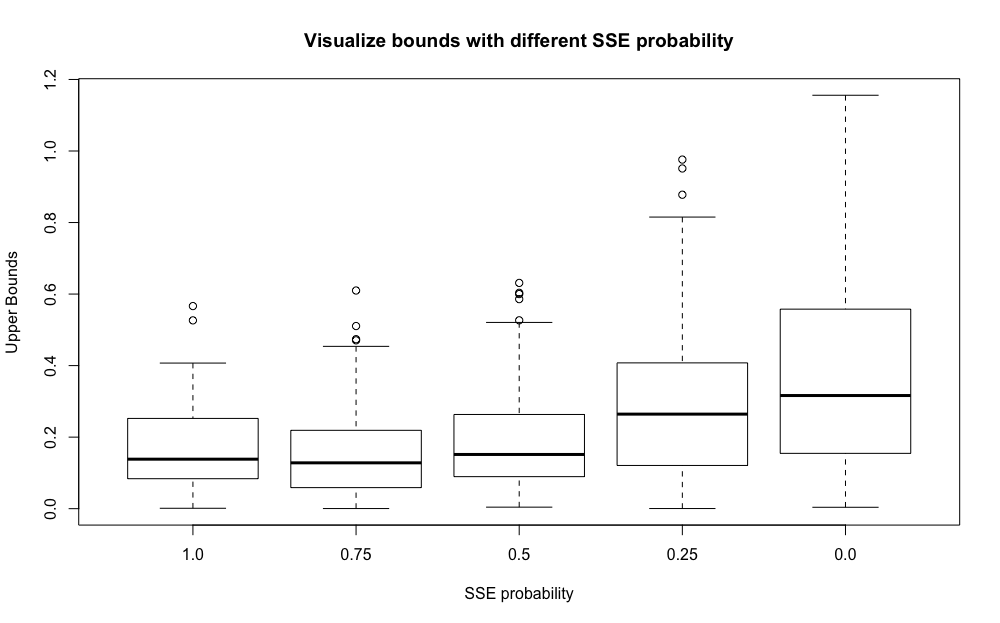}
  \caption{Simulation of a bound on $\rho_{L,n}$ for the movielens1M dataset.  Throughout the simulation, $L$ is replaced with $\ell$ (which will bound $\rho_{L,n}$ by Jensen's inequality).  The SSE probability parameter dictates the probability of transitioning.  When this is $0$ (box plot on the right), the distribution is that of the samples from the standard Rademacher complexity (without the sup and expectation).  As we increase the transition probability, the values for $\rho_{L,n}$ get smaller.
  }\vspace{-12pt}
  \label{fig:sim}
\end{figure}

Throughout this section, we will suppress the probability parameters, $p(.,.|\Phi) = p(.,.)$.

\begin{proof}[Proof of Theorem \ref{thm:rademacher}]
Consider the following variability term,
\begin{equation}
   \sup_\Theta | S(\Theta) - S_n(\Theta) |. 
\end{equation}
Let us break the variability term into two components
\[
\mathbb{E}_{X,Y} \sup_\Theta \left| S_n(\Theta) - \mathbb E_{Y|X} [S_n(\Theta)] \right| + \mathbb{E}_{X,Y} \sup_\Theta \left| \mathbb E_{Y|X} [S_n(\Theta)] - S(\Theta) \right|,
\]
where $X,Y$ represent the random input and label.
To control the first term, we introduce a ghost dataset $(x_{i}, y_{i}')$, where $y_{i}'$ are independently and identically distributed according to $y_i | x_i$.
Define
\begin{equation}
    S'_n(\Theta) = \sum_{i} \sum_{\mathbf{k}} p(\mathbf{j}^i, \mathbf{k}) \ell(\mathbf{E}[\mathbf{k}],y_i'|\Theta)
\end{equation} be the empirical SSE risk with respect to this ghost dataset.
 
 We will rewrite $\mathbb E_{Y|X} [S_n(\Theta)]$ in terms of the ghost dataset and apply Jensen's inequality and law of iterated conditional expectation:
\begin{align}
    &\mathbb{E} \sup_\Theta \left| \mathbb E_{Y|X} [S_n(\Theta)] -  S_n(\Theta) \right| \\
    &= \mathbb{E} \sup_\Theta \left| \mathbb{E}_{Y' | X} \left[S'_n(\Theta) - S_n(\Theta) \right] \right| \\
    &\leq \mathbb{E}\mathbb{E}_{Y' | X} \left[ \sup_\Theta \left| S'_n(\Theta) - S_n(\Theta) \right| \right] \\
    &= \mathbb{E} \sup_\Theta  \left| S'_n(\Theta) - S_n(\Theta) \right|.
\end{align}

Notice that 
\begin{align*}
S'_n(\Theta) - S_n(\Theta) = \sum_{i} \sum_{\mathbf{k}} p(\mathbf{j}^i, \mathbf{k}) \left( \ell(\mathbf{E}[\mathbf{k}],y_i'|\Theta) - \ell(\mathbf{E}[\mathbf{k}],y_i|\Theta)\right)\\
= \sum_{i} \sum_{\mathbf{k}} p(\mathbf{j}^i, \mathbf{k}) \left( e(\mathbf{E}[\mathbf{k}],y_i'|\Theta) - e(\mathbf{E}[\mathbf{k}],y_i|\Theta)\right).
\end{align*}
Because $y_i, y_i' | X$ are independent the term $(\sum_{\mathbf{k}} p(\mathbf{j}^i, \mathbf{k}) \left( e(\mathbf{E}[\mathbf{k}],y_i'|\Theta) - e(\mathbf{E}[\mathbf{k}],y_i|\Theta)\right))_i$ is a vector of symmetric independent random variables.
Thus its distribution is not effected by multiplication by arbitrary Rademacher vectors $\sigma_i \in \{-1,+1\}$.
\[
\mathbb{E} \sup_\Theta  \left| S'_n(\Theta) - S_n(\Theta) \right| = \mathbb{E} \sup_\Theta \left| \sum_{i} \sigma_i \sum_{\mathbf{k}} p(\mathbf{j}^i, \mathbf{k}) \left( e(\mathbf{E}[\mathbf{k}],y_i'|\Theta) - e(\mathbf{E}[\mathbf{k}],y_i|\Theta)\right) \right|.
\]
But this is bounded by 
\[
2 \mathbb{E} \mathbb E_\sigma \sup_\Theta \left| \sum_{i} \sigma_i \sum_{\mathbf{k}} p(\mathbf{j}^i, \mathbf{k}) e(\mathbf{E}[\mathbf{k}],y_i|\Theta) \right|.
\]

For the second term,
\[
\mathbb{E} \sup_\Theta \left| \mathbb E_{Y|X} [S_n(\Theta)] - S(\Theta) \right|
\]
we will introduce a second ghost dataset $x_i',y_i'$ drawn iid to $x_i,y_i$.
Because we are augmenting the input then this results in a new ghost encoding $\bo j'^{i}$.
Let
\begin{equation}
    S'_n(\Theta) = \sum_{i} \sum_{\mathbf{k}} p(\mathbf{j}'^{i}, \mathbf{k}) \ell(\mathbf{E}[\mathbf{k}],y_i'|\Theta)
\end{equation} be the empirical risk with respect to this ghost dataset.
Then we have that 
\[
S(\Theta) = \mathbb E_{X'} \mathbb E_{Y' | X'} S_n'(\Theta) 
\]
Thus,
\begin{align}
    &\mathbb{E} \sup_\Theta \left| \mathbb E_{Y|X} [S_n(\Theta)] -  S(\Theta) \right| \\
    &= \mathbb{E} \sup_\Theta \left| \mathbb{E}_{X'} \left[E_{Y|X} [S_n(\Theta)] - E_{Y'|X'} [S'_n(\Theta)] \right] \right| \\
    &\leq \mathbb{E}\mathbb{E}_{X'} \left[ \sup_\Theta \left| E_{Y|X} [S_n(\Theta)] - E_{Y'|X'} [S'_n(\Theta)] \right| \right] \\
    &= \mathbb{E} \sup_\Theta \left| E_{Y|X} [S_n(\Theta)] - E_{Y'|X'} [S'_n(\Theta)] \right|.
\end{align}
Notice that we may write,
\[
E_{Y|X} [S_n(\Theta)] - E_{Y'|X'} [S'_n(\Theta)] = \sum_{i} \sum_{\mathbf{k}} \left( p(\mathbf{j}^{i}, \mathbf{k}) - p(\mathbf{j}'^{i},\mathbf{k}) \right) L(\mathbf{E}[\mathbf{k}]|\Theta)
\]
Again we may introduce a second set of Rademacher random variables $\sigma'_i$, which results in 
\[
\mathbb{E} \sup_\Theta \left| E_{Y|X} [S_n(\Theta)] - E_{Y'|X'} [S'_n(\Theta)] \right| \le 2 \mathbb E \mathbb E_{\sigma'} \sup_\Theta \left| \sum_{i} \sigma'_i \sum_{\mathbf{k}} p(\mathbf{j}^{i}, \mathbf{k}) L(\mathbf{E}[\mathbf{k}]|\Theta) \right|.
\]
And this is bounded by 
\[
2 \mathbb E \mathbb E_{\sigma'} \sup_\Theta \left| \sum_{i} \sigma'_i \sum_{\mathbf{k}} p(\mathbf{j}^{i}, \mathbf{k}) L(\mathbf{E}[\mathbf{k}]|\Theta) \right| \le 2 \mathbb E \sup_\Theta \left| \sum_{i} \sigma'_i \sum_{\mathbf{k}} p(\mathbf{j}^{i}, \mathbf{k}) \ell(\mathbf{E}[\mathbf{k}],y_i|\Theta) \right|
\]
by Jensen's inequality again.
\end{proof}

\begin{proof}[Proof of Theorem \ref{thm:main}]
It is clear that $2 \mathcal B \ge B(\hat \Theta) + B(\Theta^*)$.
It remains to show our concentration inequality.
Consider changing a single sample, $(x_i,y_i)$ to $(x_i',y_i')$, thus resulting in the SSE empirical risk, $S_{n,i}(\Theta)$.
Thus,
\begin{align*}
&S_n(\Theta) - S_{n,i}(\Theta) = \sum_{\mathbf{k}} p(\mathbf{j}^i, \mathbf{k}) \cdot \ell(\mathbf{E}[\mathbf{k}],y_i|\Theta) - \sum_{\mathbf{k}} p(\mathbf{j}'^i, \mathbf{k}) \cdot \ell(\mathbf{E}[\mathbf{k}],y_i'|\Theta)\\
&= \sum_{\mathbf{k}} p(\mathbf{j}^i, \mathbf{k}) \cdot \left( \ell(\mathbf{E}[\mathbf{k}],y_i|\Theta) - \ell(\mathbf{E}[\mathbf{k}],y_i'|\Theta)\right) + \sum_{\mathbf{k}} \left( p(\mathbf{j}'^i, \mathbf{k}) - p(\mathbf{j}^i, \mathbf{k}) \right) \cdot \ell(\mathbf{E}[\mathbf{k}],y_i'|\Theta)\\
&\le b \left( \sum_{\mathbf{k}} p(\mathbf{j}^i, \mathbf{k}) + \sum_{\mathbf{k}} p(\mathbf{j}'^i, \mathbf{k}) \right) \le 2b.
\end{align*}
Then the result follows from McDiarmid's inequality.

\end{proof}
\section{Appendix to Chapter 6}

\begin{itemize}
    \item NDCG$@K$: defined as:
        \begin{equation}\label{eq:ndcg}
            \text{NDCG}@K = \frac{1}{n} \sum_{i = 1}^{n} \frac{\text{DCG}@K(i, \Pi_i)}{\text{DCG}@K(i, \Pi_i^*)}, 
        \end{equation} where $i$ represents $i$-th user and
        \begin{equation} 
        \text{DCG}@K(i, \Pi_i)= \sum_{l = 1}^{K} \frac{2^{R_{i\Pi_{il}}} - 1}{\log_2(l + 1)}. 
        \end{equation}
        In the DCG definition, $\Pi_{il}$ represents the index of the $l$-th ranked item for user $i$ in test data based on the learned score matrix $X$. $R$ is the rating matrix and $R_{ij}$ is the rating given to item $j$ by user $i$. $\Pi_i^*$ is the ordering provided by the ground truth rating.

    \item Recall@$K$: defined as a fraction of positive items retrieved by the top $K$ recommendations the model makes: 
        \begin{equation}\label{eq:recall}
            \text{Recall}@K = \frac{\sum_{i=1}^{n} \mathbbm{1} \{  \exists 1 \leq l \leq K : R_{i\Pi_{il}} = 1 \}} {n},
        \end{equation}
        here we already assume there is only a single positive item that user will engage next and the indicator function $\mathbbm{1} \{  \exists 1 \leq l \leq k : R_{i\Pi_{il}} = 1 \}$ is defined to indicate whether the positive item falls into the top $K$ position in our obtained ranked list using scores predicted in \eqref{eq:pred}. 
\end{itemize}

 \paragraph{Layer Normalization}
 Layer normalization \cite{ba2016layer} normalizes neurons within a layer. Previous studies \cite{ba2016layer} show it is more effective than batch normalization for training recurrent neural networks (RNNs).
One alternative is the batch normalization \cite{ioffe2015batch} but we find it does not work as well as the layer normalization in practice even for a reasonable large batch size of 128. Therefore, our SSE-PT model adopts layer normalization.  

\paragraph{Residual Connections}
Residual connections are firstly proposed in ResNet for image classification problems \cite{he2016deep}. 
Recent research finds that residual connections can help training very deep neural networks even if they are not convolutional neural networks \cite{vaswani2017attention}.  
Using residual connections allows us to train very deep neural networks here. For example, the best performing model for Movielens10M dataset in Table~\ref{tb:block} is the SSE-PT with 6 attention blocks, in which $1 + 6 * 3 + 1 = 20$ layers are trained end-to-end.

\paragraph{Weight Decay}
Weight decay \cite{krogh1992simple}, also known as $l_2$ regularization \cite{hoerl1970ridge}, is applied to all embeddings, including both user and item embeddings.

\paragraph{Dropout}
Dropout \cite{srivastava2014dropout} is applied to the embedding layer $E$, self-attention layer and pointwise feed-forward layer by stochastically dropping some percentage of hidden units to prevent co-adaption of neurons. Dropout has been shown to be an effective way of regularizing deep learning models. 

In summary, layer normalization and dropout are used in all layers except prediction layer. Residual connections are used in both self-attention layer and pointwise feed-forward layer. SSE-SE is used in embedding layer and prediction layer.

\begin{table}[ht]
\centering
\caption{Description of Datasets Used in Evaluations.}
\label{tb:datasets}
\begin{center}
\begin{small}
\begin{sc}
\resizebox{0.8\textwidth}{!}{
\begin{tabular}{ccccc}
\toprule
 dataset & \#users & \#items & avg sequence len & max sequence len   \\
\midrule
\multirow{1}{*}  Beauty   &  52,024  & 57,289   & 7.6 & 291 \\
games   &  31,013 &  23,715   & 7.3 & 858 \\
steam & 334,730 &  13,047 &   11.0 & 1,229  \\
ml-1m  &  6,040  &  3,416 & 163.5  & 2,275  \\
ml-10m  & 69,878  & 65,133 &  141.1 & 7,357  \\
\bottomrule
\end{tabular}
}
\end{sc}
\end{small}
\end{center}
\vskip -0.1in
\end{table}

\begin{itemize}
    \item PopRec: ranking items according to their popularity. 
    \item BPR: Bayesian personalized ranking for implicit feedback setting \cite{rendle2009bpr}. It is a low-rank matrix factorization model with a pairwise loss function.
    But it does not utilize the temporal information. Therefore, it serves as a strong baseline for non-temporal methods.
     \item FMC: Factorized Markov Chains: a first-order Markov Chain method, in which 
     predictions are made only based on previously engaged item.  
     \item PFMC: a personalized Markov chain model \cite{rendle2010factorizing} that combines matrix factorization and first-order Markov Chain to take advantage of both users' latent long-term preferences as well as short-term item transitions.
    \item TransRec: a first-order sequential recommendation method \cite{he2017translation} in which items are embedded into a transition space and users are modelled as translation vectors operating on item sequences.
\end{itemize}
SQL-Rank \cite{wu2018sql} and item-based recommendations \cite{sarwar2001item}  are omitted
because the former is similar to BPR \cite{rendle2009bpr} except using the listwise loss function instead of the pairwise loss function and the latter has been shown inferior to TransRec \cite{he2017translation}.

\begin{table*}[t]
\vskip -0.1in
  \caption{Comparing our SSE-PT, SSE-PT++ with SASRec on Movielen1M dataset. We use number of negatives $C = 100$, dropout probability of $0.2$ and learning rate of $1e^{-3}$ for all experiments while varying others. $p_u, p_i, p_u$ are SSE probabilities for user embedding, input item embedding and output item embedding respectively.}
  \vskip -0.01in
  \label{tb:ml1m}
  \centering
 \resizebox{1\columnwidth}{!}{
  \begin{tabular}{lccccccccc}
    \toprule
    & \multicolumn{2}{c}{Movielens1m }  & \multicolumn{2}{c}{Dimensions} & Number of Blocks & Sampling Probability & \multicolumn{3}{c}{SSE-SE Parameters}             \\
    \cmidrule(r){2-3} \cmidrule(r){4-5} \cmidrule(r){6-6} \cmidrule(r){7-7}  \cmidrule(r){8-10}
    Model     & NDCG$@10$ & Recall$@10$ & $d_u$ & $d_i$ & $b$ & $p_s$ & $p_u$ & $p_i$  & $p_y$ \\
    \midrule
    SASRec            & 0.5961 & 0.8195& - & 50 & 2 & - & -  & - & - \\
    SASRec   & 0.5941 & 0.8182 & - & 100 & 2  & - & - & -  & -\\
    SASRec   & \bfseries{0.5996} & \bfseries{0.8272} & - & 100 & 6 & - & - & -  & -\\
     \midrule
    SSE-PT  & 0.6101 & 0.8343 & 50 &  50 & 2 & -& 0.92 & 0.1  & 0 \\
    SSE-PT  & 0.6164 & 0.8336 & 50 &  50 & 2 & -& 0.92 & 0  & 0.1 \\
    SSE-PT  & 0.5832 & 0.8091 & 50 &  50 & 2 & -& 0 & 0.1  & 0.1 \\
    SSE-PT  & \bfseries{0.6174}  & \bfseries{0.8351} & 50 &  50 & 2& - & 0.92 & 0.1  & 0.1 \\
    \midrule
     SSE-PT  & 0.5949  & 0.8205 & 75 &  25 & 2 & -& 0.92 & 0.1  & 0.1 \\
     
   
   SSE-PT  & \bfseries{0.6214}  & \bfseries{0.8359} & 25 &  75 & 2& - & 0.92 & 0.1  & 0.1 \\
    \midrule
    SSE-PT  & 0.6281 & 0.8341 & 50 & 100 & 2 & -& 0.92 & 0.1  & 0.1 \\
    SSE-PT++ & \bfseries{0.6292} & \bfseries{0.8389} &  50 & 100 & 2& 0.3 & 0.92 & 0.1  & 0.1 \\
    \bottomrule
  \end{tabular}
}
\end{table*}

\begin{table*}[t]
  \caption{Comparing our SSE-PT with SASRec on Movielens10M dataset. Unlike Table~\ref{tb:ml1m}, we use the number of negatives $C = 500$ instead of $100$ as $C = 100$ is too easy for this dataset and it gets too difficult to tell the differences between different methods: Hit Ratio@10 approaches 1.}
  \label{tb:ml10m}
  \centering
 \resizebox{\columnwidth}{!}{
  \begin{tabular}{ccccccccc}
    \toprule
    & \multicolumn{2}{c}{Movielens1m }  & \multicolumn{2}{c}{Dimensions} & Number of Blocks  & \multicolumn{3}{c}{SSE-SE Parameters}             \\
    \cmidrule(r){2-3} \cmidrule(r){4-5} \cmidrule(r){6-6}  \cmidrule(r){7-9}
    Model     & NDCG$@10$ & Hit Ratio$@10$ & $d_u$ & $d_i$ & $b$ & $p_u$ & $p_i$  & $p_y$ \\
    \midrule
    SASRec   & 0.7268 & 0.9429 & - & 50 & 2  & - & -  & -\\
    SASRec   & 0.7413 & 0.9474 & - & 100 & 2  & - & -  & -\\
    \midrule
   SSE-PT  & 0.7199 & 0.9331 & 50 &  100 & 2 & PS & 0.01  & 0.01 \\
   SSE-PT  & 0.7169 & 0.9296 & 50 &  100 & 2 & 0.0 & 0.01  & 0.01 \\
    SSE-PT  & 0.7398 & 0.9418 & 50 &  100 & 2 & 0.2 & 0.01  & 0.01 \\
    SSE-PT  & 0.7500 & 0.9500 & 50 &  100 & 2 & 0.4 & 0.01  & 0.01 \\
    SSE-PT  & 0.7484 & 0.9480 & 50 &  100 & 2 & 0.6 & 0.01  & 0.01 \\
    SSE-PT  & \bfseries{0.7529} & 0.9485 & 50 &  100 & 2 & 0.8 & 0.01  & 0.01 \\
    SSE-PT  & 0.7503 &  \bfseries{0.9505} & 50 &  100 & 2 & 1.0 & 0.01  & 0.01 \\
    \bottomrule
  \end{tabular}
}
\end{table*}
\subsubsection{Deep-learning baselines}
\begin{itemize}
    \item GRU4Rec: the first RNN-based method proposed for the session-based recommendation problem \cite{hidasi2015session}. It utilizes the GRU structures \cite{chung2014empirical} initially proposed for speech modelling.
    \item GRU4Rec$^+$: follow-up work of GRU4Rec by the same authors: the model has a very similar architecture to GRU4Rec but has a more complicated loss function \cite{hidasi2018recurrent}.
    \item Caser: a CNN-based method \cite{tang2018personalized} which embeds a sequence of recent items in both time and latent spaces forming an `image' before learning local features through horizontal and vertical convolutional filters. In \cite{tang2018personalized}, user embeddings are included in the prediction layer only. On the contrast, in our Personalized Transformer, user embeddings are also introduced in the lowest embedding layer so they can play an important role in self-attention mechanisms as well as in prediction stages.
    \item STAMP: a session-based recommendation algorithm \cite{liu2018stamp} using attention mechanism. \cite{liu2018stamp} only uses fully connected layers with one attention block that is not self-attentive. 
    \item SASRec: a self-attentive sequential recommendation method \cite{kang2018self} motivated by Transformer in NLP \cite{vaswani2017attention}. Unlike our method SSE-PT, SASRec does not incorporate user embedding and therefore is not a personalized method. SASRec paper \cite{kang2018self} also does not utilize SSE \cite{wu2019stochastic} for further regularization: only dropout and weight decay are used.
    \item HGN: hierarchical gating networks method to solve the sequential recommendation problem \cite{ma2019hierarchical}, which incorporates the user embeddings and gating networks for better personalization than the SASRec model.
\end{itemize}

\begin{table}[ht]
\centering
\caption{Comparing Different SSE probability for user embeddings for SSE-PT on Movielens1M Dataset. Embedding hidden units of 50 for users and 100 for items, attention blocks of 2, SSE probability of 0.01 for item embeddings, dropout probability of 0.2 and max length of 200 are used.}
\label{tb:ssep}
\begin{center}
\begin{small}
\begin{sc}
\resizebox{0.6\textwidth}{!}{
\begin{tabular}{cccc}
\toprule
 User-Side SSE-SE Probability & NDCG@$10$ & Recall@$10$ \\
\midrule
Parameter Sharing & 0.6188 & 0.8294 \\ 
\midrule
 1.0 & 0.6258  &  0.8346  \\
 0.9 & \bfseries{0.6275}  &  0.8321  \\
 0.8 & 0.6244 & 0.8359 \\
 0.6 & 0.6256 & 0.8341 \\
 0.4 & 0.6237 & \bfseries{0.8369} \\
 0.2 & 0.6163 & 0.8281 \\
 0.0 & 0.5908 & 0.8048 \\
 
\bottomrule
\end{tabular}
}
\end{sc}
\end{small}
\end{center}
\end{table}

\begin{table}[ht]
\centering
\caption{Comparing Different Sampling Probability, $p_s$, of SSE-PT++ on Movielens1M Dataset. Hyper-parameters the same as Table~\ref{tb:ssep}, except that the max length $T$ allowed is set 100  instead of 200 to show effects of sampling sequences.}
\label{tb:sampling}
\begin{center}
\begin{small}
\begin{sc}
\resizebox{0.6 \textwidth}{!}{
\begin{tabular}{cccc}
\toprule
 Sampling Probability & NDCG@$10$ & Recall@$10$ \\
\midrule
SASRec ($T=100$) & 0.5769 &  0.8045\\ 
SSE-PT ($T=100$) & 0.6142 & 0.8212  \\
\midrule
 1.0 & 0.5697   &  0.7977  \\
 0.8 &0.5735  & 0.7801  \\
 0.6 & 0.6062  & 0.8242 \\
 0.4 &  0.6113 &  0.8273 \\
 0.3 &  0.6186 & \bfseries{0.8318} \\
 0.2 & \bfseries{0.6193} & 0.8233 \\
 0.0  & 0.6142 & 0.8212  \\
 
\bottomrule
\end{tabular}
}
\end{sc}
\end{small}
\end{center}
\end{table}


\begin{table}[ht]
\centering
\caption{Comparing Different Number of Blocks for SSE-PT while Keeping The Rest Fixed on Movielens1M and Movielens10M Datasets.}
\label{tb:block}
\begin{center}
\begin{small}
\begin{sc}
\resizebox{0.6\textwidth}{!}{
\begin{tabular}{ccccc}
\toprule
Datasets & \# of blocks & NDCG@$10$ & Recall@$10$ \\
\midrule
\multirow{7}{*}{Movielens1M}
& SASREC (6 blocks) & 0.5984 & 0.8207 \\ 
\cmidrule(r){2-4} 
& 1 & 0.6162  &  0.8301  \\
& 2 & 0.6280  &  0.8365  \\
& 3 & 0.6293 & 0.8376 \\
& 4 & 0.6270 & \bfseries{0.8401} \\
& 5 & \bfseries{0.6308} & 0.8361 \\
& 6 & 0.6270 & 0.8397 \\
\midrule
\multirow{7}{*}{Movielens10M}
& SASRec (6 blocks) & 0.7531 & 0.9490 \\
\cmidrule(r){2-4} 
& 1 &  0.7454 & 0.9478    \\
& 2 &  0.7512 &  0.9522  \\
& 3 &  0.7543 & 0.9491  \\
& 4 &  0.7608 & 0.9485  \\
& 5 & 0.7619 & 0.9524  \\
& 6 & \bfseries{0.7683} & \bfseries{0.9537}  \\
\bottomrule
\end{tabular}
}
\end{sc}
\end{small}
\end{center}
\end{table}

\begin{table}[t]
\centering
\caption{Varying number of negatives $C$ in evaluation on Movielens1M dataset. 
Other hyper-parameters are fixed for a fair comparison. }
\label{tb:negative}
\begin{center}
\begin{small}
\begin{sc}
\resizebox{0.6\textwidth}{!}{
\begin{tabular}{lcccc}
\toprule
 METRIC & NDCG@$10$ & Recall@$10$ & $C$ \\
\midrule
Un-Personalized   &  0.3787 &  0.6119   & 500 \\
 Personalized  & \bfseries{0.3846} & \bfseries{0.6171}  & 500 \\
 \midrule
 Un-Personalized   &  0.2791 &  0.4781    & 1000 \\
 Personalized  & \bfseries{0.2860} & \bfseries{0.4929} & 1000 \\
 \midrule
 Un-Personalized   &  0.1939 &  0.3515    & 2000 \\
 Personalized  & \bfseries{0.1993} & \bfseries{0.3667} & 2000 \\
\bottomrule
\end{tabular}
}
\end{sc}
\end{small}
\end{center}
\end{table}

\bibliography{ucdavisthesis}

\UMIabstract[In this dissertation, we cover some recent advances in collaborative filtering and ranking. 
In chapter 1, we give a brief introduction of the history and the current landscape of collaborative filtering and ranking; chapter 2 we first talk about pointwise collaborative filtering problem with graph information, and how our proposed new method can encode very deep graph information which helps four existing graph collaborative filtering algorithms; chapter 3 is on the pairwise approach for collaborative ranking and how we speed up the algorithm to near-linear time complexity; chapter 4 is on the new listwise approach for collaborative ranking and how the listwise approach is a better choice of loss for both explicit and implicit feedback over pointwise and pairwise loss; chapter 5 is about the new regularization technique Stochastic Shared Embeddings (SSE) we proposed for embedding layers and how it is both  theoretically sound and empirically effectively for 6 different tasks across recommendation and natural language processing; chapter 6 is how we introduce personalization for the state-of-the-art sequential recommendation model with the help of SSE, which plays an important role in preventing our personalized model from overfitting to the training data; chapter 7, we summarize what we have achieved so far and predict what the future directions can be; chapter 8 is the appendix to all the chapters.]

\end{document}